\documentclass{article}

\usepackage[utf8]{inputenc}
\usepackage[T1]{fontenc}

\usepackage{microtype}
\usepackage{graphicx}
\usepackage{subcaption}
\usepackage{algorithm}
\usepackage{algorithmic}
\usepackage{booktabs} 

\usepackage{hyperref}
\usepackage{url}
\usepackage{xcolor}
\usepackage{enumitem}
\usepackage{titletoc}
\usepackage{listings}
\usepackage[preprint]{neurips_2026}

\usepackage{amsmath}
\usepackage{amssymb}
\usepackage{mathtools}
\usepackage{amsthm}

\usepackage[capitalize,noabbrev]{cleveref}

\usepackage{booktabs}
\usepackage{siunitx}
\usepackage{multirow}
\usepackage{wrapfig}
\usepackage{dblfloatfix}
\usepackage{placeins}
\newcommand{\NA}{\textemdash}
\newcommand{\stagepair}[2]{#1$\rightarrow$#2}
\newcommand{\stagebest}[2]{#1$\rightarrow$\textbf{#2}}
\newcommand{\MeanAvg}{\ensuremath{\overline{\mathrm{Avg}}}}
\theoremstyle{plain}
\newtheorem{theorem}{Theorem}[section]
\newtheorem{proposition}[theorem]{Proposition}

\newtheorem{corollary}[theorem]{Corollary}
\theoremstyle{definition}

\theoremstyle{remark}

\usepackage[disable,textsize=tiny]{todonotes}

\title{Rotation-Preserving Supervised Fine-Tuning}

\author{%
\begin{tabular}{c}
Hangzhan Jin$^{1,2,*}$
\quad
Tianwei Ni$^{1,3}$
\quad
Lu Li$^{1,3}$\\
Pierre-Luc Bacon$^{1,3,5}$
\quad
Mohammad Hamdaqa$^{2}$
\quad
Doina Precup$^{1,4,5,6}$\\[0.35em]
\normalfont\footnotesize
$^{1}$Mila - Quebec AI Institute
\quad
$^{2}$Polytechnique Montr\'eal
\quad
$^{3}$Universit\'e de Montr\'eal\\[-0.1em]
\normalfont\footnotesize
$^{4}$McGill University
\quad
$^{5}$CIFAR AI Chair
\quad
$^{6}$Google DeepMind
\end{tabular}
}

\begin{document}

\maketitle

\begingroup
\renewcommand{\thefootnote}{}
\footnotetext{
$^{*}$Corresponding author: \texttt{hangzhan.jin@mila.quebec}.
}
\endgroup
\vspace{-0.5cm}
\begin{abstract}

Supervised fine-tuning (SFT) improves in-domain performance but can degrade out-of-domain (OOD) generalization. Prior work suggests that this degradation is related to changes in dominant singular subspaces of pretrained weight matrices. However, directly identifying loss-sensitive directions with Hessian or Fisher information is computationally expensive at LLM scale. In this work, we propose preserving projected rotations in pretrained singular subspaces as an efficient proxy for Fisher-sensitive directions, which we call Rotation-Preserving Supervised Fine-Tuning (RPSFT). RPSFT penalizes changes in the projected top-$k$ singular-vector block of each pretrained weight matrix, limiting unnecessary rotation while preserving task adaptation. Across model families and sizes trained on math reasoning data, RPSFT improves the in-domain/OOD trade-off over standard SFT and strong SFT baselines, better preserves pretrained representations, and provides stronger initializations for downstream RL fine-tuning. 
Code is available at \href{https://github.com/jinhangzhan/RPSFT.git}{https://github.com/jinhangzhan/RPSFT}.

\end{abstract}

\begin{figure}[htbp]
\centering
\includegraphics[width=\linewidth]{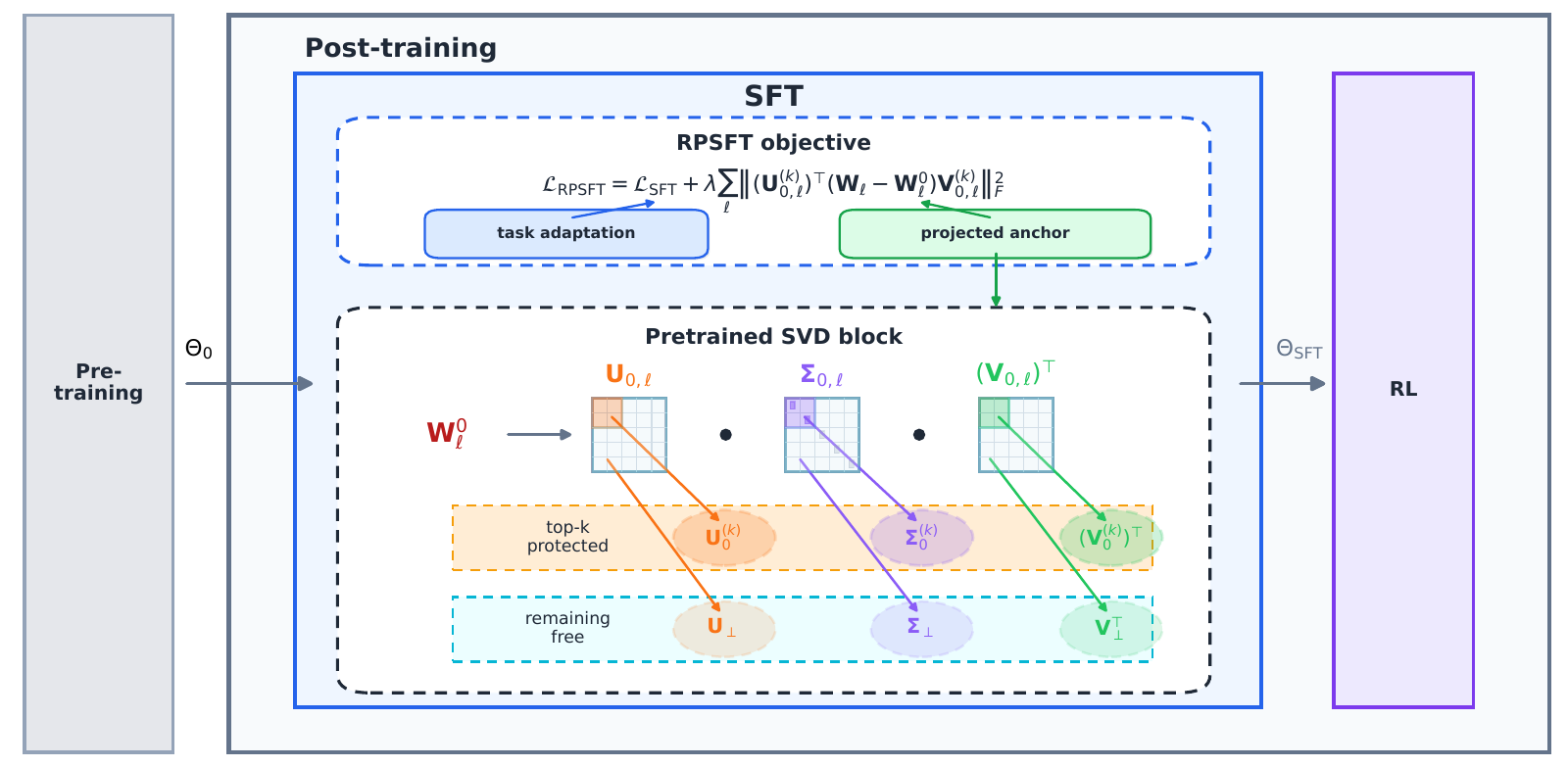}
\vspace{-1em}
\caption{Method overview. RPSFT modifies SFT by adding a projected-block anchor in the pretrained SVD basis while retaining full-parameter task adaptation. The top-$k$ pretrained singular block is protected, complementary directions remain free, and the resulting $\theta_{\text{SFT}}$ initializes RLFT.}
\vspace{-1em}
\label{fig:rpsft-method-overview}
\end{figure}

\section{Introduction}
\looseness=-1
Large language models are commonly post-trained with supervised fine-tuning (SFT) followed by reinforcement learning fine-tuning (RLFT). Although SFT improves tasks represented in the post-training data, it can degrade out-of-domain (OOD) capabilities, a form of forgetting caused by excessive specialization~\citep{zhu2025proximalsupervisedfinetuning, wu2025generalizationsftreinforcementlearning}. Prior work links this degradation to geometric drift, especially rotations of dominant singular directions in pretrained weight matrices~\citep{jin2025rlfinetuninghealsood, zhu2025pathtakenrlvrprovably}. Since these directions are associated with high-variance and high-curvature structure~\citep{haink2023hessianeigenvectorsprincipalcomponent}, preserving them may protect general-purpose capabilities during SFT.

This motivates our question: \textit{Can supervised fine-tuning preserve OOD generalization by limiting unnecessary geometric drift while retaining task adaptation?} To address this question, we propose \emph{Rotation-Preserving Supervised Fine-Tuning} (RPSFT), a simple regularization method motivated by the observed overlap between dominant singular subspaces and Fisher-sensitive directions tied to OOD forgetting. RPSFT penalizes changes in the pretrained top-$k$ singular-vector block of each selected weight matrix. Unlike freezing or hard gradient projection, it anchors only the dominant pretrained block while leaving complementary directions free to adapt, and integrates into standard SFT pipelines without additional data or task boundaries.

Concretely, for each selected matrix $\mathbf{W}$, RPSFT precomputes the pretrained top-$k$ left and right singular-vector bases $\mathbf{U}^{(k)}_{0}$ and $\mathbf{V}^{(k)}_{0}$, and adds the  Frobenius norm $\lambda\|(\mathbf{U}^{(k)}_{0})^\top(\mathbf{W}-\mathbf{W}^0)\mathbf{V}^{(k)}_{0}\|_F^2$ to the SFT loss. This formulation preserves the expressivity of supervised fine-tuning while stabilizing general-purpose capability learned during pretraining. Figure~\ref{fig:rpsft-method-overview} summarizes the overall RPSFT post-training workflow and Algorithm~\ref{alg:fpft} summarizes the algorithm. This method integrates directly into standard SFT pipelines and requires no additional data or task boundaries.

We evaluate RPSFT with full-parameter fine-tuning across Llama \citep{grattafiori2024llama} and Qwen \citep{qwen2.5} checkpoints trained on OpenR1-Math~\citep{openr1}. We treat math benchmarks as in-domain and general reasoning, safety, and knowledge benchmarks as OOD. Across model families and sizes, RPSFT improves the ID/OOD trade-off over SFT and strong SFT baselines, summarized in Figure~\ref{fig:result-overview}. To understand why it works, we analyze representation drift after supervised fine-tuning. Comparing the hidden states of the base model and the checkpoint after SFT, we show that RPSFT better preserves the pretrained representation geometry. We then evaluate downstream reinforcement learning with DAPO~\citep{yu2025dapoopensourcellmreinforcement}, a variant of GRPO~\citep{shao2024deepseekmathpushinglimitsmathematical}, across all three model sizes and find that RPSFT provides strong downstream initializations and consistently higher or competitive final RL performance.

Together, RPSFT contributes a simple projected-subspace regularizer for SFT, consistent improvements in the ID/OOD trade-off across Llama~\citep{grattafiori2024llama} and Qwen~\citep{qwen2.5} models, and empirical and theoretical evidence that preserving dominant pretrained subspaces reduces rotation, protects hidden-state representations, and improves the forgetting--adaptation trade-off.

\section{Preliminaries}

Modern post-training for reasoning language models typically proceeds in two stages: supervised fine-tuning (SFT) on curated instruction data, followed by reinforcement learning fine-tuning (RLFT) on task-level rewards~\citep{deepseekr1, openai2024openaio1card, bai2023qwentechnicalreport}.

\paragraph{SFT and PEFT.}
Supervised fine-tuning adapts pretrained parameters $\theta_0$ on labeled pairs $\mathcal{D}=\{(x,y)\}$ by minimizing the negative log-likelihood $\mathcal{L}_{\mathrm{SFT}}(\theta)=\mathbb{E}_{(x,y)\sim\mathcal{D}}[-\log \pi_\theta(y\mid x)]$, where $x$ is an input prompt, $y$ is the target response, and $\pi_\theta(y\mid x)$ is the model distribution with parameters $\theta$ initialized from $\theta_0$. In this work, the main setting is full-parameter SFT, where all model parameters are updated. We include vanilla LoRA~\citep{hu2021loralowrankadaptationlarge} as a PEFT baseline in Appendix~\ref{app:robustness-extra}.

\paragraph{RLFT and DAPO.}
After SFT, the model can be further optimized with reinforcement learning. We use DAPO~\citep{yu2025dapoopensourcellmreinforcement} that samples a group of $G$ responses from the rollout policy $\pi_{\theta_{\mathrm{old}}}(\cdot \mid x)$, computes a normalized group-relative advantage \( \hat{A}_i = \frac{r_i - \frac{1}{G}\sum_{j=1}^{G} r_j}{\operatorname{std}(\{r_j\}_{j=1}^{G}) + \varepsilon} \) with \(r_i = r(x,y_i)\) and small numerical constant \(\varepsilon\), and uses the token-level PPO ratio~\citep{schulman2017proximalpolicyoptimizationalgorithms} \( \rho_{i,t}(\theta) = \frac{\pi_{\theta}(y_{i,t} \mid x, y_{i,<t})}{\pi_{\theta_{\mathrm{old}}}(y_{i,t} \mid x, y_{i,<t})} \). The resulting surrogate objective is
\begin{equation}
{\footnotesize
\mathcal{J}_{\mathrm{DAPO}}(\theta)
= \mathbb{E}_{x,\{y_i\}\sim \pi_{\theta_{\mathrm{old}}}(\cdot \mid x)} \!\left[
\frac{1}{\sum_{i=1}^{G} |y_i|}
\sum_{i,t}
\min\!\Big(
\rho_{i,t}\hat{A}_i,\,
\operatorname{clip}\!\big(\rho_{i,t}, 1-\epsilon_{\mathrm{low}}, 1+\epsilon_{\mathrm{high}}\big)\hat{A}_i
\Big)
\right].
}
\end{equation}
Here $i$ indexes the $G$ sampled responses for each prompt, $t$ ranges over the valid tokens of response $y_i$, $|y_i|$ is the response length, and $\epsilon_{\mathrm{low}},\epsilon_{\mathrm{high}}$ are the PPO clipping thresholds. Compared with GRPO, DAPO averages the clipped policy objective over valid tokens rather than over whole responses.


\section{Why Preserving Singular Vectors  Mitigates Forgetting}
\label{sec:motivation}
Post-training must balance \emph{rapid adaptation}, which efficiently improves upon the target task, and \emph{mitigating forgetting}, which preserves capabilities already encoded in the pretrained model \citep{lyle2023understandingplasticityneuralnetworks,lu2025rethinkingstabilityplasticitytradeoffcontinual}. Prior work has linked dominant subspaces to important functional properties of neural networks \citep{haink2023hessianeigenvectorsprincipalcomponent}, and excessive rotation of leading singular subspaces in large language models has been associated with out-of-domain degradation \citep{jin2025rlfinetuninghealsood}. In this section, we complement this view by connecting this trade-off to forgetting. We analyze a local second-order expansion around the pretrained weights, following the standard curvature view used in second-order optimization \citep{martens2010deeplearninghessianfree}. Let a vectorized base-model weight be $w \in \mathbb R^d$, the post-trained weight be $w' = w + \Delta w$, and $l$ denote the loss on a task. The local expansion is
\begin{equation}
l(w') \approx l(w) + \underbrace{\langle \nabla l(w), \Delta w \rangle}_{ \text{adaptation}} + \underbrace{\frac{1}{2} \Delta w^\top \nabla^2 l(w) \Delta w}_{\text{forgetting}} .
\end{equation}
The first-order term captures task adaptation: if the update aligns with the target gradient, the loss decreases. The second-order term captures the need to mitigate forgetting: moving along high-curvature directions can sharply increase the loss of capabilities that the pretrained model already fits well. For OOD retention, forgetting is often dominated by this curvature term once the base model already achieves local-optimal OOD performance where the first-order is close to zero.

\begin{wrapfigure}[24]{r}{0.60\columnwidth}
\vspace{-0.80\baselineskip}
\centering
\includegraphics[width=\linewidth]{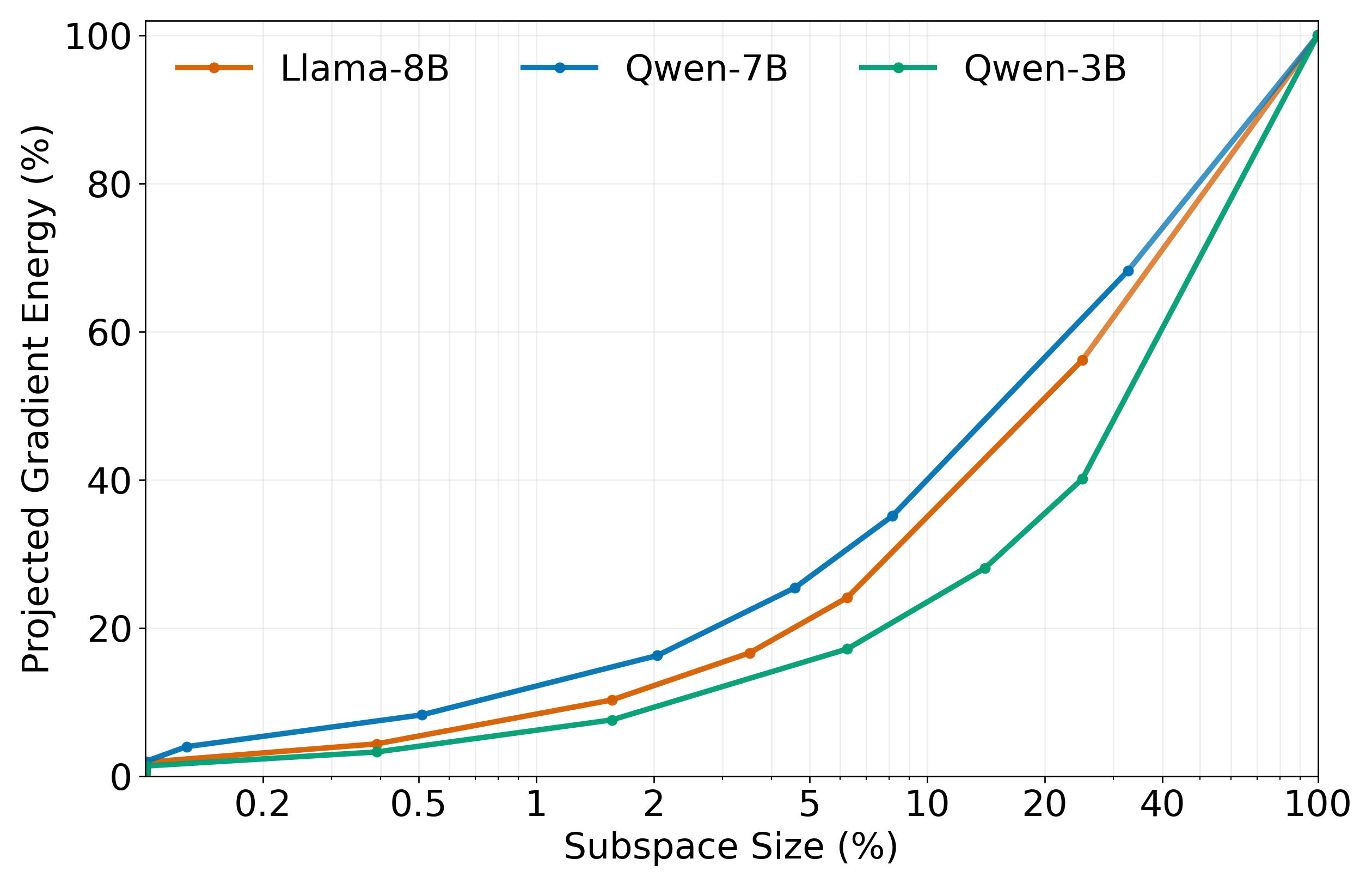}
{\caption{Fisher-projected total sum of squared gradient norms (gradient energy) captured by top-$r\times r$ SVD blocks. The x-axis reports $x(r)= \frac{r^2}{R^2}$, the strict top-$r\times r$ SVD-block size relative to the full $R\times R$ block, and the y-axis reports $y(r)=\frac{\sum_t\|\mathbf{U}_r^\top\mathbf{G}_t\mathbf{V}_r\|_F^2}{\sum_t\|\mathbf{G}_t\|_F^2}$, the Fisher-projected gradient-energy percentage, where $\mathbf{G}_t \in \mathbb{R}^{m \times n}$ denotes the gradient matrix of the $t$-th sample.}
\label{fig:svd-fisher-motivation}}
\vspace{-0.80\baselineskip}
\end{wrapfigure}

The local second-order view suggests that the Hessian is the most direct object for identifying loss-sensitive directions: updates along high-curvature directions can cause large increases in retained capability loss. However, explicitly estimating, storing, or regularizing with the Hessian is infeasible at LLM scale. For the SFT negative log-likelihood objective, the Fisher information provides a tractable positive-semidefinite curvature proxy; near a local optimum, its expected form coincides with the Hessian of the expected SFT loss \citep{martens2020newinsightsperspectivesnatural}, and continual-learning methods such as EWC~\citep{Kirkpatrick_2017} therefore use Fisher information to identify directions that should remain stable. At the same time, the empirical Fisher should not be treated as identical to the Hessian away from the local optimum~\citep{kunstner2019limitationsempiricalfisher}, and even estimating and storing a full Fisher matrix remains too expensive for large language models. This motivates a more practical question: can we find a cheap structural proxy that aligns with Fisher-sensitive directions without explicitly constructing the Fisher matrix?

We answer this question with a Fisher-projected energy diagnostic (Figure~\ref{fig:svd-fisher-motivation}). In the first attention layer, the cumulative projected energy rises quickly for Llama-8B, Qwen-7B, and Qwen-3B: a relatively small strict top-5\% singular block already captures about 20\% of the full Fisher-projected gradient energy. We therefore use rank 768 as the default protected rank in our experiments. Although the pretrained singular basis is not identical to the Fisher eigenspace, its strong overlap with Fisher-projected gradient energy indicates that dominant singular directions provide a useful low-rank structural proxy for loss-sensitive curvature directions. The exact diagnostic is given in Appendix~\ref{app:svd-energy-metric}.

Taken together, the rapid-adaptation versus forgetting-mitigation view, the second-order curvature argument, and the Fisher-overlap evidence motivate our method: preserve the dominant pretrained singular directions while still allowing useful task adaptation.

\section{Method}
\label{sec:method}

Based on the prior analysis in Section ~\ref{sec:motivation}, we propose \textbf{Rotation-Preserving SFT (RPSFT)}. Intuitively, RPSFT treats dominant singular subspaces as carriers of general-purpose transformations and discourages unnecessary rotation during fine-tuning that can harm generalization. This simple regularizer reduces out-of-domain (OOD) capability loss during supervised fine-tuning by discouraging drift within the dominant singular subspace of each weight matrix. All reported RPSFT results use full-parameter fine-tuning (FPFT). We penalize each selected layer equally.

\paragraph{Setup and Notation.}
Let $\theta_0$ be the pretrained model and let $\mathcal{M}$ denote the selected 2D weight matrices. For each pretrained matrix $\mathbf{W}^0\in\mathbb{R}^{m\times n}$, we compute its truncated SVD $\mathbf{W}^0=\mathbf{U}^0\mathbf{\Sigma}^0(\mathbf{V}^0)^\top$ and cache the top-$k$ bases $\mathbf{U}^{(k)}_0\in\mathbb{R}^{m\times k}$ and $\mathbf{V}^{(k)}_0\in\mathbb{R}^{n\times k}$, where $k$ is the protected rank.




\paragraph{Rotation-Preserving Regularization.}
For a fine-tuned weight $\mathbf{W}$, define the projected block $S(\mathbf{W})=(\mathbf{U}^{(k)}_0)^\top\mathbf{W}\mathbf{V}^{(k)}_0$ and the pretrained reference block $S^{\mathrm{ref}}=S(\mathbf{W}^0)$. Since the basis is computed from the pretrained SVD, $S^{\mathrm{ref}}$ is the leading $k\times k$ diagonal block of $\mathbf{\Sigma}^0$. RPSFT anchors how the current weight acts between the pretrained top-$k$ left and right singular subspaces, while leaving complementary blocks free to adapt.

This viewpoint also clarifies the boundary cases. When $k=0$, the penalty vanishes and RPSFT reduces to standard SFT. When $k=\min(m_\ell,n_\ell)$, the penalty becomes full weight-space $\ell_2$ anchoring around the pretrained model~\citep{kumar2024maintainingplasticitycontinuallearning}. We formalize both cases in Appendix~\ref{app:theory}. The practical intuition is that the regularizer suppresses drift along dominant pretrained directions.

In the full-parameter setting, the total loss is:
\begin{equation}
\label{eq:fpft_obj}
\mathcal{L}(\theta)
=
\mathcal{L}_{\text{task}}(\theta)
+
\lambda \sum_{\ell \in \mathcal{M}'} 
\left\|
S_\ell(\mathbf{W}) - S^{\mathrm{ref}}_\ell
\right\|_F^2,
\end{equation}
where $\lambda$ controls the regularization strength. This penalty is zero at initialization and only grows when fine-tuning moves the weight inside the pretrained top-$k$ singular subspace. Appendix Algorithm~\ref{alg:fpft} summarizes the FPFT training procedure. Appendix~\ref{app:lora-objective} gives an optional LoRA formulation.

\paragraph{Computation costs.} 
\label{subsec:cost}
For each regularized matrix $\mathbf{W}^0_\ell \in \mathbb{R}^{m_\ell \times n_\ell}$, we store $\mathbf{U}^{(k)}_{0,\ell}$, $\mathbf{V}^{(k)}_{0,\ell}$ and $S^{\mathrm{ref}}_\ell$:
\begin{equation}
\text{extra memory per layer} \;=\; \mathcal{O}\!\left((m_\ell+n_\ell)k + k^2\right).
\end{equation}
In practice, $k$ is small relative to $m_\ell,n_\ell$, so this overhead is linear in the protected rank and remains manageable under the default bf16 precision.

For each regularized matrix $\mathbf{W}\in\mathbb{R}^{m\times n}$, computing the projected block
$(\mathbf{U}^{(k)}_0)^\top \mathbf{W}\mathbf{V}^{(k)}_0$ adds $O(mnk)$ FLOPs per evaluation, with the same order for backpropagation. If applied every $s$ optimization steps, the amortized cost becomes $O(mnk/s)$ per step, corresponding to a heuristic relative overhead of roughly $\frac{k}{s\min(m,n)}$.


\section{Experiments}

\textbf{Benchmarks and evaluation.} The supervised stage is trained on \href{https://huggingface.co/datasets/wh-zhu/train_openr1_4k}{OpenR1-Math}~\citep{openr1}, and downstream DAPO is initialized from those checkpoints using the \href{https://huggingface.co/datasets/wh-zhu/dapo}{DAPO-Math-17k} RL set. In-domain results cover \href{https://huggingface.co/datasets/math-ai/aime24}{AIME24},
\href{https://huggingface.co/datasets/math-ai/aime25}{AIME25},
\href{https://huggingface.co/datasets/math-ai/amc23}{AMC23},
\href{https://huggingface.co/datasets/HuggingFaceH4/MATH-500}{MATH-500}~\citep{hendrycks2021measuringmathematicalproblemsolving},
\href{https://huggingface.co/datasets/math-ai/minervamath}{Minerva Math}~\citep{lewkowycz2022solvingquantitativereasoningproblems}, and
\href{https://huggingface.co/datasets/math-ai/olympiadbench}{OlympiadBench}~\citep{he2024olympiadbenchchallengingbenchmarkpromoting}, while out-of-domain results cover \href{https://huggingface.co/datasets/Idavidrein/gpqa}{GPQA}~\citep{rein2023gpqagraduatelevelgoogleproofqa},
\href{https://huggingface.co/datasets/google/IFEval}{IFEval-loose}~\citep{zhou2023instructionfollowingevaluationlargelanguage},
\href{https://huggingface.co/datasets/TIGER-Lab/MMLU-Pro}{MMLU-Pro}~\citep{wang2024mmluprorobustchallengingmultitask},
\href{https://huggingface.co/datasets/Maxwell-Jia/SuperGPQA-Astro}{SuperGPQA}~\citep{pteam2025supergpqascalingllmevaluation},
\href{https://huggingface.co/datasets/ThWu/safety_benchmark}{Safety Benchmark}, and
\href{https://huggingface.co/datasets/truthfulqa/truthful_qa}{TruthfulQA}~\citep{lin2022truthfulqameasuringmodelsmimic}. 
The full post-training pipeline and evaluation split are summarized in Figure~\ref{fig:result-overview}.

\textbf{Base models and baselines.}
We evaluate three instruction-tuned base checkpoints spanning two model families and scales: Llama-3.1-8B-Instruct, Qwen2.5-7B-Instruct, and Qwen2.5-3B-Instruct. We report the untuned base checkpoint as a reference and compare RPSFT with standard full-parameter SFT, Dynamic Fine-Tuning (DFT)~\citep{wu2025generalizationsftreinforcementlearning}, and importance-weighted SFT (IW)~\citep{qin2025supervisedfinetuningcurated}. DFT and IW rescale the supervised training gradient by either the token probability or the importance sampling ratio. 


\begin{figure}[htbp]
\centering
\includegraphics[width=1.0\textwidth]{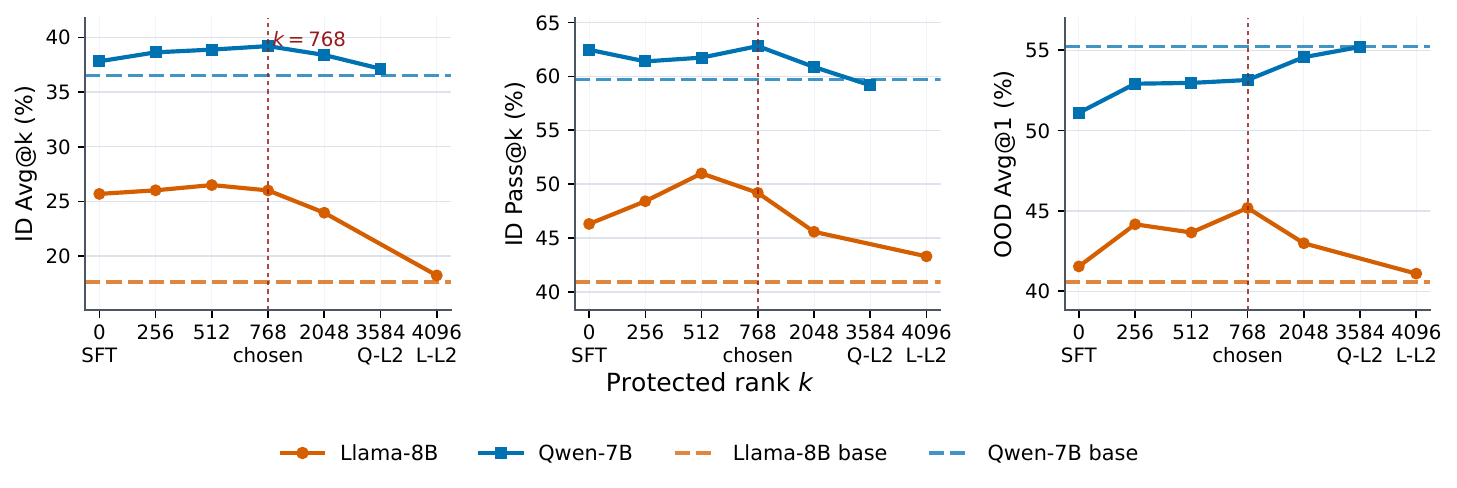}
\caption{\small Rank-selection sensitivity. The x-axis is the protected rank $k$ and the y-axes report ID Avg@k, ID Pass@k, and OOD Avg@1. $0$ is vanilla SFT, $768$ is our default, Q-L2/L-L2 are full-rank protection for Qwen/Llama (equivalent to L2 Init~\citep{kumar2024maintainingplasticitycontinuallearning}), and dashed lines show base-model scores. Scores average six ID math benchmarks and six OOD benchmarks.}

\label{fig:rank-selection-robustness}
\end{figure}


\textbf{Hyperparameters.}
RPSFT has two hyperparameters: the protected rank $k$ (this $k$ should not be confused with the $k$ in avg@k and pass@k) and regularization strength $\lambda$. As shown in Figure~\ref{fig:rank-selection-robustness}, performance is robust in intermediate ranks, while overly large ranks limit adaptation. Unless stated otherwise, we use $k=768$ and $\lambda=1$. Appendix~\ref{app:rank_selection} gives the rank-selection rule and full sweeps.

\textbf{Empirical Computation Costs.} Table~\ref{tab:empirical-cost} summarizes the empirical training costs. Although RPSFT requires more peak memory than SFT and DFT because of the precomputed protected bases and projection buffers, it remains below IW and L2 Init in peak allocated memory. RPSFT uses only $6.0\%$ more GPU-hours than SFT while remaining cheaper than IW and L2 Init.

\begin{figure}[!b]
\centering
\includegraphics[width=\linewidth]{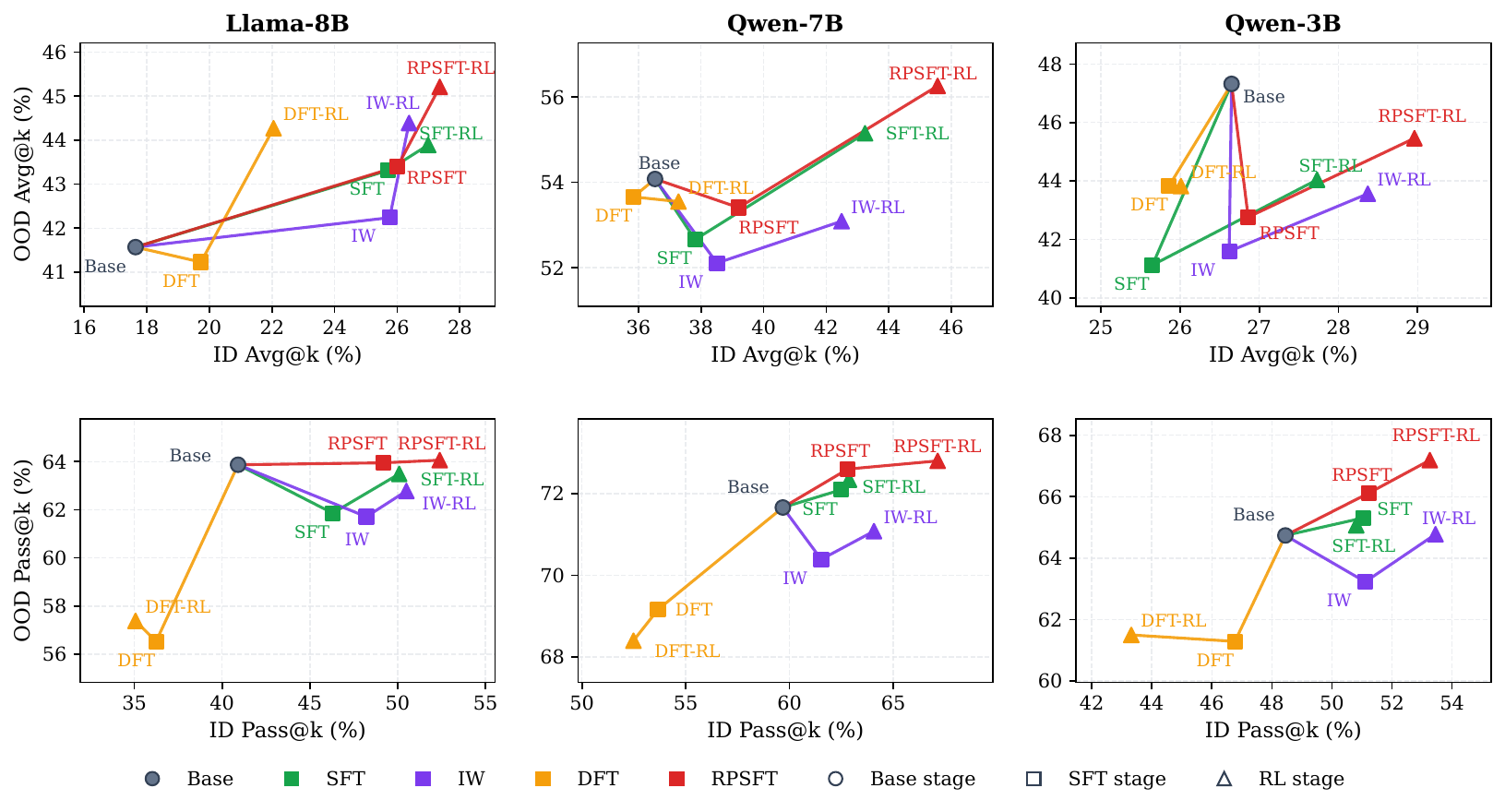}
\caption{ID/OOD trade-off across SFT and RL stages for RPSFT and baselines. The top row compares ID Avg@k with OOD Avg@k, and the bottom row compares ID Pass@k with OOD Pass@k. Scores follow the settings in Figure~\ref{fig:rank-selection-robustness}.}
\label{fig:result-overview}
\end{figure}

\begin{table}[!htbp]
\centering
\caption{\small Empirical computation and memory cost for RPSFT and baselines. Samples/s and Steps/s report training throughput; Peak alloc. and Peak reserved report maximum allocated and CUDA-reserved GPU memory.}
\label{tab:empirical-cost}
\scriptsize
\setlength{\tabcolsep}{4pt}
\renewcommand{\arraystretch}{0.88}
\begin{tabular*}{0.98\textwidth}{@{\extracolsep{\fill}}lrrrrr@{}}
\toprule
Metric & SFT & DFT & IW & L2 Init & RPSFT \\
\midrule
GPU-hours & 156.3 & 159.0 & 181.8 & 174.1 & 165.7 \\
Samples/s & 4.368 & 4.294 & 3.755 & 3.922 & 4.119 \\
Steps/s & 0.273 & 0.268 & 0.235 & 0.245 & 0.257 \\
Peak alloc. & 32.4 & 40.0 & 59.3 & 60.9 & 52.7 \\
Peak reserved & 125.4 & 128.1 & 137.0 & 136.6 & 135.2 \\
\bottomrule
\end{tabular*}
\vspace{-0.8em}
\end{table}

\subsection{Results on Supervised Fine-tuning}
\textbf{RPSFT improves generalization when possible and reduces forgetting otherwise.} Table~\ref{tab:sft-merged} shows that RPSFT gives the best tuned ID Avg@k on all three checkpoints and the best tuned ID Pass@k overall. On OOD Pass@k, RPSFT exceeds the base checkpoint for all three models, indicating improved generalization after SFT. For OOD Avg@k, RPSFT improves over the base on Llama-3.1-8B, while on the two Qwen models it reduces OOD degradation relative to vanilla SFT. Although DFT has slightly less OOD Avg@k forgetting on Qwen, it does so by underfitting the ID task: its ID averages are consistently worse than RPSFT and even below the base model in several Qwen summary rows. Thus, RPSFT provides a better ID/OOD trade-off: it preserves or improves OOD behavior without sacrificing task adaptation.
 
\textbf{First-order transfer explains when SFT generalizes or forgets.}
Figure~\ref{fig:first-order-signal} shows that Llama-8B has mostly positive first-order alignment (Section~\ref{sec:motivation}) on both ID and OOD data, explaining why SFT can improve ID performance while also transferring to OOD. In contrast, the Qwen models have weaker and sign-changing OOD first-order signals: although they still adapt to the ID task, OOD updates are less consistently descent-aligned, making them more vulnerable to curvature-driven forgetting. The metric is defined in Appendix~\ref{app:first-order-signal}.\looseness=-1

\begin{table}[t]
\centering
\scriptsize
\setlength{\tabcolsep}{3pt}
\renewcommand{\arraystretch}{0.95}
\caption{\textbf{SFT results on in-domain and out-of-domain benchmarks.} ``$\Delta$ vs.\ Base'' rows show change relative to the base model. Bold marks the strongest tuned result in the same block, and summary-average rows or columns are marked as \MeanAvg{}.
For the in-domain benchmarks, we use the same benchmark-specific $k$ for both Avg@k and Pass@k: $k=16$ for AIME24, AIME25, and AMC23, $k=4$ for MATH-500 and OlympiadBench, and $k=8$ for Minerva Math. For out-of-domain benchmarks, we use $k=4$ and report both Avg@k and Pass@k.}
\label{tab:sft-merged}
\resizebox{\textwidth}{!}{%
\begin{tabular}{cl S S S S S S S S S S S S S S S}
\toprule
& & \multicolumn{5}{c}{\textbf{Llama-3.1-8B-Instruct}} & \multicolumn{5}{c}{\textbf{Qwen2.5-7B-Instruct}} & \multicolumn{5}{c}{\textbf{Qwen2.5-3B-Instruct}} \\
\cmidrule(lr){3-7}\cmidrule(lr){8-12}\cmidrule(lr){13-17}
Domain & Benchmark & {Base} & {SFT} & {IW} & {DFT} & {RPSFT} & {Base} & {SFT} & {IW} & {DFT} & {RPSFT} & {Base} & {SFT} & {IW} & {DFT} & {RPSFT} \\
\midrule
\multirow{8}{*}{\rotatebox[origin=c]{90}{\textbf{ID Avg@k}}} & AIME24 & 3.13 & \bfseries 5.00 & 3.13 & 4.17 & 4.38 & 12.08 & \bfseries 15.21 & 14.17 & 12.29 & 14.38 & 4.79 & 3.96 & \bfseries 6.67 & 2.71 & 6.46 \\
& AIME25 & 0.83 & 5.00 & \bfseries 6.88 & 1.88 & 6.67 & 7.70 & \bfseries 19.38 & 18.54 & 9.38 & 18.75 & 2.29 & 5.00 & \bfseries 6.04 & 4.79 & 5.00 \\
& AMC23 & 21.25 & 32.19 & \bfseries 33.59 & 19.53 & 31.56 & 52.03 & 50.47 & 52.50 & 48.91 & \bfseries 55.31 & 37.03 & 35.00 & \bfseries 36.25 & 34.06 & 34.84 \\
& MATH-500 & 43.65 & 59.45 & 58.50 & 46.40 & \bfseries 60.45 & 72.85 & 72.85 & 72.60 & 73.35 & \bfseries 74.95 & 63.35 & 60.95 & 61.05 & 62.05 & \bfseries 63.35 \\
& Minerva & 22.33 & 26.88 & 26.75 & 20.17 & \bfseries 27.48 & 37.09 & 31.52 & \bfseries 34.05 & 33.73 & 33.82 & 26.24 & 22.75 & 23.71 & \bfseries 26.56 & 23.85 \\
& Olympiad & 14.63 & 25.59 & 25.71 & \bfseries 26.22 & 25.52 & 37.41 & 37.41 & \bfseries 39.18 & 37.41 & 38.00 & 26.22 & 26.25 & 26.00 & 25.00 & \bfseries 27.67 \\
& \MeanAvg{} & 17.64 & 25.70 & 25.76 & 19.73 & \bfseries 26.01 & 36.53 & 37.81 & 38.51 & 35.84 & \bfseries 39.20 & 26.65 & 25.65 & 26.62 & 25.86 & \bfseries 26.86 \\
& {$\Delta$ vs.\ Base} & \multicolumn{1}{c}{\NA} & \multicolumn{1}{c}{$\uparrow$8.06} & \multicolumn{1}{c}{$\uparrow$8.12} & \multicolumn{1}{c}{$\uparrow$2.09} & \multicolumn{1}{c}{\textbf{$\uparrow$8.37}} & \multicolumn{1}{c}{\NA} & \multicolumn{1}{c}{$\uparrow$1.28} & \multicolumn{1}{c}{$\uparrow$1.98} & \multicolumn{1}{c}{$\downarrow$0.69} & \multicolumn{1}{c}{\textbf{$\uparrow$2.67}} & \multicolumn{1}{c}{\NA} & \multicolumn{1}{c}{$\downarrow$1.00} & \multicolumn{1}{c}{$\downarrow$0.03} & \multicolumn{1}{c}{$\downarrow$0.79} & \multicolumn{1}{c}{\textbf{$\uparrow$0.21}} \\
\midrule
\multirow{8}{*}{\rotatebox[origin=c]{90}{\textbf{ID Pass@k}}} & AIME24 & 20.00 & \bfseries 20.00 & 16.67 & 13.33 & \bfseries 20.00 & 36.67 & \bfseries 50.00 & 46.67 & 23.33 & 43.33 & 20.00 & 30.00 & 26.66 & 23.33 & \bfseries 33.33 \\
& AIME25 & 13.33 & 13.33 & \bfseries 26.67 & 16.67 & 20.00 & 36.67 & 36.67 & 36.67 & 33.33 & \bfseries 43.33 & 20.00 & 23.33 & \bfseries 30.00 & 23.33 & 23.33 \\
& AMC23 & 72.50 & 77.50 & 80.00 & 62.50 & \bfseries 85.00 & 90.00 & 87.50 & 85.00 & 82.50 & \bfseries 90.00 & 80.00 & \bfseries 82.50 & \bfseries 82.50 & 77.50 & 80.00 \\
& MATH-500 & 65.00 & \bfseries 78.60 & 76.20 & 61.20 & 77.20 & 84.60 & 87.60 & \bfseries 89.00 & 84.20 & 87.40 & 79.20 & \bfseries 80.60 & 79.60 & 75.80 & 79.80 \\
& Minerva & 44.85 & 47.06 & 48.16 & 37.50 & \bfseries 50.74 & 55.17 & \bfseries 58.09 & 55.51 & 48.53 & 56.62 & 50.00 & \bfseries 47.42 & 45.22 & 42.28 & 47.06 \\
& Olympiad & 29.78 & 41.33 & 41.63 & 26.22 & \bfseries 42.22 & 54.96 & 54.96 & \bfseries 56.30 & 50.07 & 56.15 & 41.48 & 42.37 & 42.67 & 38.37 & \bfseries 43.85 \\
& \MeanAvg{} & 40.91 & 46.30 & 48.22 & 36.24 & \bfseries 49.19 & 59.68 & 62.47 & 61.53 & 53.66 & \bfseries 62.81 & 48.45 & 51.04 & 51.11 & 46.77 & \bfseries 51.23 \\
& {$\Delta$ vs.\ Base} & \multicolumn{1}{c}{\NA} & \multicolumn{1}{c}{$\uparrow$5.39} & \multicolumn{1}{c}{$\uparrow$7.31} & \multicolumn{1}{c}{$\downarrow$4.67} & \multicolumn{1}{c}{\textbf{$\uparrow$8.28}} & \multicolumn{1}{c}{\NA} & \multicolumn{1}{c}{$\uparrow$2.79} & \multicolumn{1}{c}{$\uparrow$1.85} & \multicolumn{1}{c}{$\downarrow$6.02} & \multicolumn{1}{c}{\textbf{$\uparrow$3.13}} & \multicolumn{1}{c}{\NA} & \multicolumn{1}{c}{$\uparrow$2.59} & \multicolumn{1}{c}{$\uparrow$2.66} & \multicolumn{1}{c}{$\downarrow$1.68} & \multicolumn{1}{c}{\textbf{$\uparrow$2.78}} \\
\midrule
\multirow{8}{*}{\rotatebox[origin=c]{90}{\textbf{OOD Avg@k}}} & GPQA & 27.62 & 32.70 & \bfseries 34.32 & 26.79 & 33.48 & 32.76 & \bfseries 35.66 & 34.82 & 30.47 & 34.21 & 27.18 & 27.73 & \bfseries 29.41 & 26.45 & 28.13 \\
& IFEval & 32.30 & \bfseries 32.58 & 31.70 & 31.84 & 31.93 & 64.23 & 54.48 & 55.91 & \bfseries 61.18 & 58.32 & 58.09 & 47.97 & 47.92 & \bfseries 52.59 & 49.72 \\
& MMLU-Pro & 50.71 & \bfseries 56.07 & 52.86 & 50.00 & 52.86 & 62.50 & 64.64 & 64.64 & 66.07 & \bfseries 68.21 & 50.36 & 40.71 & 42.14 & 40.71 & \bfseries 43.93 \\
& SuperGPQA & 19.88 & \bfseries 24.38 & 22.28 & 17.72 & 23.77 & 25.62 & \bfseries 28.46 & 25.99 & 25.25 & 27.84 & 21.36 & 19.81 & 19.38 & 19.26 & \bfseries 21.67 \\
& Safety & 59.43 & 57.83 & \bfseries 62.08 & 60.45 & 60.55 & 73.00 & 70.30 & 69.80 & \bfseries 72.35 & 70.45 & 65.90 & 59.98 & 60.00 & \bfseries 66.03 & 60.83 \\
& TruthfulQA & 59.50 & 56.29 & 50.18 & \bfseries 60.56 & 57.82 & 66.37 & 62.43 & 61.48 & \bfseries 66.63 & 61.40 & 61.04 & 50.55 & 50.69 & \bfseries 57.93 & 52.30 \\
& \MeanAvg{} & 41.57 & 43.31 & 42.24 & 41.23 & \bfseries 43.40 & 54.08 & 52.66 & 52.11 & \bfseries 53.66 & 53.41 & 47.32 & 41.12 & 41.59 & \bfseries 43.83 & 42.76 \\
& {$\Delta$ vs.\ Base} & \multicolumn{1}{c}{\NA} & \multicolumn{1}{c}{$\uparrow$1.73} & \multicolumn{1}{c}{$\uparrow$0.66} & \multicolumn{1}{c}{$\downarrow$0.35} & \multicolumn{1}{c}{\textbf{$\uparrow$1.83}} & \multicolumn{1}{c}{\NA} & \multicolumn{1}{c}{$\downarrow$1.42} & \multicolumn{1}{c}{$\downarrow$1.97} & \multicolumn{1}{c}{\textbf{$\downarrow$0.42}} & \multicolumn{1}{c}{$\downarrow$0.67} & \multicolumn{1}{c}{\NA} & \multicolumn{1}{c}{$\downarrow$6.20} & \multicolumn{1}{c}{$\downarrow$5.73} & \multicolumn{1}{c}{\textbf{$\downarrow$3.49}} & \multicolumn{1}{c}{$\downarrow$4.56} \\
\midrule
\multirow{8}{*}{\rotatebox[origin=c]{90}{\textbf{OOD Pass@k}}} & GPQA & 61.38 & 65.40 & \bfseries 69.20 & 57.14 & 68.08 & 59.82 & \bfseries 65.63 & 64.29 & 58.04 & 65.18 & 54.46 & 61.61 & \bfseries 65.40 & 55.13 & 62.05 \\
& IFEval & 35.86 & 35.49 & 35.30 & 34.57 & \bfseries 35.67 & 76.71 & 69.50 & 69.50 & 73.01 & \bfseries 73.38 & 70.61 & 63.03 & 62.85 & \bfseries 65.62 & 65.25 \\
& MMLU-Pro & 81.43 & 75.71 & 77.14 & 75.71 & \bfseries 81.43 & 81.43 & 80.00 & 78.57 & 80.00 & \bfseries 84.29 & 72.86 & 70.00 & 68.57 & 60.00 & \bfseries 72.86 \\
& SuperGPQA & 48.15 & \bfseries 54.81 & 49.14 & 38.02 & 53.58 & 49.38 & 53.83 & 51.11 & 45.93 & \bfseries 54.57 & 49.63 & 45.68 & 43.70 & 44.94 & \bfseries 48.89 \\
& Safety & 76.60 & 67.80 & \bfseries 73.10 & 64.60 & 68.60 & 84.10 & \bfseries 83.30 & 80.80 & 78.30 & 79.80 & 69.10 & \bfseries 80.20 & 69.70 & 68.20 & 76.20 \\
& TruthfulQA & 79.82 & 71.93 & 66.37 & 69.15 & \bfseries 76.32 & 78.51 & \bfseries 80.26 & 78.07 & 79.68 & 78.36 & 71.78 & 71.35 & 69.15 & \bfseries 73.83 & 71.49 \\
& \MeanAvg{} & 63.87 & 61.86 & 61.71 & 56.53 & \bfseries 63.95 & 71.66 & 72.09 & 70.39 & 69.16 & \bfseries 72.60 & 64.74 & 65.31 & 63.23 & 61.29 & \bfseries 66.12 \\
& {$\Delta$ vs.\ Base} & \multicolumn{1}{c}{\NA} & \multicolumn{1}{c}{$\downarrow$2.02} & \multicolumn{1}{c}{$\downarrow$2.17} & \multicolumn{1}{c}{$\downarrow$7.34} & \multicolumn{1}{c}{\textbf{$\uparrow$0.07}} & \multicolumn{1}{c}{\NA} & \multicolumn{1}{c}{$\uparrow$0.43} & \multicolumn{1}{c}{$\downarrow$1.27} & \multicolumn{1}{c}{$\downarrow$2.50} & \multicolumn{1}{c}{\textbf{$\uparrow$0.94}} & \multicolumn{1}{c}{\NA} & \multicolumn{1}{c}{$\uparrow$0.57} & \multicolumn{1}{c}{$\downarrow$1.51} & \multicolumn{1}{c}{$\downarrow$3.45} & \multicolumn{1}{c}{\textbf{$\uparrow$1.38}} \\
\bottomrule
\end{tabular}}
\renewcommand{\arraystretch}{1}
\end{table}

\FloatBarrier

\begin{figure}[!t]
\centering
\includegraphics[width=1.0\textwidth]{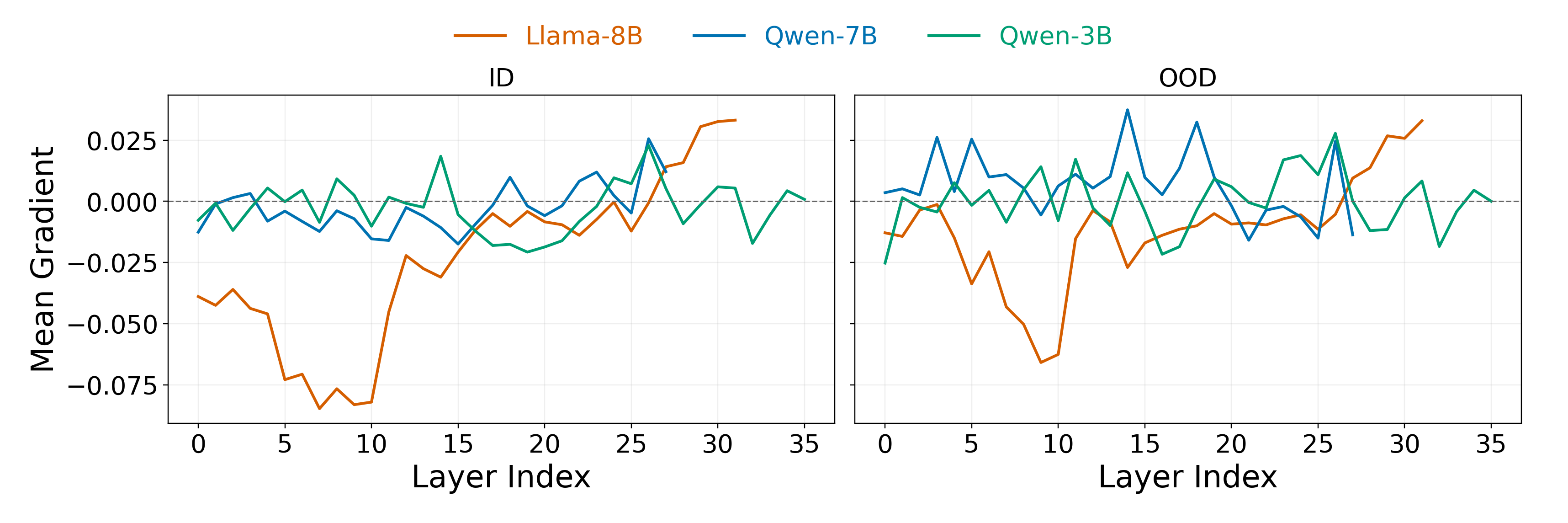}
\caption{Layerwise first-order signal on the in-domain (left) and out-of-domain (right) distributions across Llama-8B, Qwen-7B, and Qwen-3B.}
\label{fig:first-order-signal}
\end{figure}

\subsection{Results on RL Fine-Tuning}

\begin{table}[t]
\centering
\scriptsize
\setlength{\tabcolsep}{3.7pt}
\caption{\textbf{Downstream RL (DAPO) results on the six ID math benchmarks}, following Table~\ref{tab:sft-merged}'s notation. Avg@k and Pass@k blocks use init$\rightarrow$DAPO entries; bold marks the highest final value among methods.}
\label{tab:rl-id-avg}
\begin{tabular}{ll c c c c c c c}
\toprule
Metric & Method & {AIME24} & {AIME25} & {AMC23} & {MATH-500} & {Minerva} & {Olympiad} & {\MeanAvg} \\
\midrule
\multicolumn{9}{l}{\textbf{Llama-3.1-8B-Instruct}}\\
\midrule
\multirow{4}{*}{\textbf{Avg@k}} & SFT & \stagebest{5.00}{6.46} & \stagepair{5.00}{5.42} & \stagebest{32.19}{34.38} & \stagebest{59.45}{61.05} & \stagebest{26.88}{29.09} & \stagepair{25.59}{25.56} & \stagepair{25.70}{26.99} \\
& IW & \stagepair{3.13}{5.21} & \stagepair{6.88}{6.46} & \stagepair{33.59}{32.81} & \stagepair{58.50}{59.60} & \stagepair{26.75}{28.77} & \stagepair{25.71}{25.44} & \stagepair{25.76}{26.38} \\
& DFT & \stagepair{4.17}{3.13} & \stagepair{1.88}{0.83} & \stagepair{19.53}{24.53} & \stagepair{46.40}{48.50} & \stagepair{20.17}{22.98} & \stagepair{26.22}{17.41} & \stagepair{19.73}{22.05} \\
& RPSFT & \stagepair{4.38}{6.25} & \stagebest{6.67}{7.92} & \stagepair{31.56}{33.59} & \stagepair{60.45}{60.80} & \stagepair{27.48}{28.54} & \stagebest{25.52}{27.03} & \stagebest{26.01}{27.36} \\
\midrule
\multirow{4}{*}{\textbf{Pass@k}} & SFT & \stagepair{20.00}{20.00} & \stagepair{13.33}{26.67} & \stagepair{77.50}{82.50} & \stagepair{78.60}{77.80} & \stagepair{47.06}{51.47} & \stagepair{41.33}{42.07} & \stagepair{46.30}{50.09} \\
& IW & \stagepair{16.67}{23.33} & \stagebest{26.67}{30.00} & \stagepair{80.00}{77.50} & \stagepair{76.20}{78.40} & \stagepair{48.16}{51.84} & \stagepair{41.63}{41.93} & \stagepair{48.22}{50.50} \\
& DFT & \stagepair{13.33}{10.00} & \stagepair{16.67}{10.00} & \stagepair{62.50}{60.00} & \stagepair{61.20}{63.40} & \stagepair{37.50}{38.60} & \stagepair{26.22}{28.44} & \stagepair{36.24}{35.07} \\
& RPSFT & \stagebest{20.00}{30.00} & \stagepair{20.00}{23.33} & \stagebest{85.00}{85.00} & \stagebest{77.20}{79.40} & \stagebest{50.74}{52.94} & \stagebest{42.22}{43.70} & \stagebest{49.19}{52.40} \\
\midrule
\multicolumn{9}{l}{\textbf{Qwen2.5-7B-Instruct}}\\
\midrule
\multirow{4}{*}{\textbf{Avg@k}} & SFT & \stagepair{15.21}{18.75} & \stagepair{19.38}{20.83} & \stagepair{50.47}{58.12} & \stagepair{72.85}{79.35} & \stagepair{31.52}{38.97} & \stagepair{37.41}{43.44} & \stagepair{37.81}{43.24} \\
& IW & \stagepair{14.17}{17.29} & \stagepair{18.54}{20.00} & \stagepair{52.50}{60.00} & \stagepair{72.60}{77.95} & \stagepair{34.05}{36.76} & \stagepair{39.18}{42.96} & \stagepair{38.51}{42.49} \\
& DFT & \stagepair{12.29}{11.04} & \stagepair{9.38}{12.29} & \stagepair{48.91}{53.13} & \stagepair{73.35}{74.50} & \stagepair{33.73}{35.94} & \stagepair{37.41}{36.81} & \stagepair{35.84}{37.28} \\
& RPSFT & \stagebest{14.38}{21.25} & \stagebest{18.75}{22.71} & \stagebest{55.31}{62.19} & \stagebest{74.95}{80.50} & \stagebest{33.82}{39.94} & \stagebest{38.00}{46.78} & \stagebest{39.20}{45.56} \\
\midrule
\multirow{4}{*}{\textbf{Pass@k}} & SFT & \stagepair{50.00}{46.67} & \stagepair{36.67}{33.33} & \stagepair{87.50}{90.00} & \stagepair{87.60}{89.60} & \stagepair{58.09}{58.09} & \stagepair{54.96}{59.55} & \stagepair{62.47}{62.87} \\
& IW & \stagepair{46.67}{46.66} & \stagebest{36.67}{43.33} & \stagepair{85.00}{92.50} & \stagepair{89.00}{88.00} & \stagepair{55.51}{55.14} & \stagepair{56.30}{58.81} & \stagepair{61.53}{64.07} \\
& DFT & \stagepair{23.33}{23.33} & \stagepair{33.33}{23.33} & \stagepair{82.50}{82.50} & \stagepair{84.20}{84.20} & \stagepair{48.53}{52.94} & \stagepair{50.07}{48.59} & \stagepair{53.66}{52.48} \\
& RPSFT & \stagebest{43.33}{53.33} & \stagebest{43.33}{43.33} & \stagebest{90.00}{95.00} & \stagebest{87.40}{90.20} & \stagebest{56.62}{59.19} & \stagebest{56.15}{61.78} & \stagebest{62.81}{67.14} \\
\midrule
\multicolumn{9}{l}{\textbf{Qwen2.5-3B-Instruct}}\\
\midrule
\multirow{4}{*}{\textbf{Avg@k}} & SFT & \stagepair{3.96}{4.79} & \stagepair{5.00}{5.42} & \stagepair{35.00}{37.50} & \stagepair{60.95}{65.90} & \stagepair{22.75}{25.32} & \stagepair{26.25}{27.44} & \stagepair{25.65}{27.73} \\
& IW & \stagepair{6.67}{6.25} & \stagebest{6.04}{6.25} & \stagebest{36.25}{40.94} & \stagepair{61.05}{64.65} & \stagepair{23.71}{24.54} & \stagepair{26.00}{27.59} & \stagepair{26.62}{28.37} \\
& DFT & \stagepair{2.71}{2.50} & \stagepair{4.79}{4.38} & \stagepair{34.06}{39.53} & \stagepair{62.05}{60.70} & \stagepair{26.56}{23.67} & \stagepair{25.00}{25.30} & \stagepair{25.86}{26.01} \\
& RPSFT & \stagebest{6.46}{6.46} & \stagepair{5.00}{6.04} & \stagepair{34.84}{40.00} & \stagebest{63.35}{65.75} & \stagebest{23.85}{26.47} & \stagebest{27.67}{29.04} & \stagebest{26.86}{28.96} \\
\midrule
\multirow{4}{*}{\textbf{Pass@k}} & SFT & \stagepair{30.00}{23.33} & \stagebest{23.33}{33.33} & \stagepair{82.50}{75.00} & \stagebest{80.60}{81.00} & \stagepair{47.42}{47.79} & \stagepair{42.37}{44.44} & \stagepair{51.04}{50.81} \\
& IW & \stagebest{26.66}{33.33} & \stagebest{30.00}{33.33} & \stagebest{82.50}{85.00} & \stagepair{79.60}{80.40} & \stagepair{45.22}{45.96} & \stagepair{42.67}{42.66} & \stagebest{51.11}{53.45} \\
& DFT & \stagepair{23.33}{13.33} & \stagepair{23.33}{26.67} & \stagepair{77.50}{72.50} & \stagepair{75.80}{70.80} & \stagepair{42.28}{39.71} & \stagepair{38.37}{36.89} & \stagepair{46.77}{43.32} \\
& RPSFT & \stagepair{33.33}{30.00} & \stagepair{23.33}{30.00} & \stagepair{80.00}{82.50} & \stagepair{79.80}{80.00} & \stagebest{47.06}{50.37} & \stagebest{43.85}{46.67} & \stagepair{51.23}{53.26} \\
\bottomrule
\end{tabular}
\end{table}


\textbf{Stronger SFT initializers transfer better to downstream RL.} Table~\ref{tab:rl-id-avg} shows that RPSFT gives the best final ID Avg@k and Pass@k on Llama-3.1-8B and Qwen2.5-7B. On Qwen2.5-3B, RPSFT still gives the best final Avg@k, while IW is marginally higher on Pass@k (53.45 vs.\ 53.26). Thus, the ID/OOD trade-off improved by RPSFT at the SFT stage generally carries into RLFT rather than being washed out by DAPO.
\begin{table}[!htbp]
\centering
\scriptsize
\setlength{\tabcolsep}{3.2pt}
\renewcommand{\arraystretch}{0.95}
\caption{\textbf{Downstream RL (DAPO) results on the six OOD benchmarks}, following Table~\ref{tab:sft-merged}'s notation. Avg@k and Pass@k blocks use init$\rightarrow$DAPO entries with $k=4$; final checkpoints are selected by ID performance.}

\label{tab:rl-ood}
\begin{tabular}{ll c c c c c c c}
\toprule
Metric & Method & {GPQA} & {IFEval} & {MMLU-Pro} & {SuperGPQA} & {Safety} & {TruthfulQA} & {\MeanAvg} \\
\midrule
\multicolumn{9}{l}{\textbf{Llama-3.1-8B-Instruct}}\\
\midrule
\multirow{4}{*}{\textbf{Avg@k}} & SFT & \stagepair{32.70}{35.10} & \stagebest{32.58}{31.79} & \stagepair{56.07}{51.07} & \stagepair{24.38}{26.17} & \stagepair{57.83}{58.85} & \stagepair{56.29}{60.34} & \stagepair{43.31}{43.89} \\
& IW & \stagebest{34.32}{36.72} & \stagepair{31.70}{31.15} & \stagepair{52.86}{55.00} & \stagepair{22.28}{24.94} & \stagebest{62.08}{65.13} & \stagepair{50.18}{53.40} & \stagepair{42.24}{44.39} \\
& DFT & \stagepair{26.79}{28.96} & \stagepair{31.84}{31.01} & \stagebest{50.00}{57.86} & \stagepair{17.72}{18.77} & \stagepair{60.45}{64.50} & \stagebest{60.56}{64.55} & \stagepair{41.23}{44.27} \\
& RPSFT & \stagepair{33.48}{34.54} & \stagepair{31.93}{31.56} & \stagepair{52.86}{55.71} & \stagebest{23.77}{27.41} & \stagepair{60.55}{59.25} & \stagepair{57.82}{62.79} & \stagebest{43.40}{45.21} \\
\midrule
\multirow{4}{*}{\textbf{Pass@k}} & SFT & \stagepair{65.40}{68.08} & \stagebest{35.49}{35.12} & \stagebest{75.71}{80.00} & \stagebest{54.81}{56.79} & \stagepair{67.80}{66.40} & \stagepair{71.93}{74.56} & \stagepair{61.86}{63.49} \\
& IW & \stagebest{69.20}{68.75} & \stagebest{35.30}{35.12} & \stagepair{77.14}{77.14} & \stagepair{49.14}{51.85} & \stagebest{73.10}{74.40} & \stagepair{66.37}{69.44} & \stagepair{61.71}{62.78} \\
& DFT & \stagepair{57.14}{55.36} & \stagepair{34.57}{34.01} & \stagepair{75.71}{72.86} & \stagepair{38.02}{39.75} & \stagepair{64.60}{69.20} & \stagepair{69.15}{73.10} & \stagepair{56.53}{57.38} \\
& RPSFT & \stagepair{68.08}{67.86} & \stagepair{35.67}{34.75} & \stagepair{81.43}{78.57} & \stagepair{53.58}{54.57} & \stagepair{68.60}{71.40} & \stagebest{76.32}{77.19} & \stagebest{63.95}{64.06} \\
\midrule
\multicolumn{9}{l}{\textbf{Qwen2.5-7B-Instruct}}\\
\midrule
\multirow{4}{*}{\textbf{Avg@k}} & SFT & \stagepair{35.66}{37.39} & \stagepair{54.48}{52.82} & \stagepair{64.64}{71.07} & \stagepair{28.46}{32.28} & \stagepair{70.30}{71.90} & \stagepair{62.43}{65.42} & \stagepair{52.66}{55.15} \\
& IW & \stagepair{34.82}{34.26} & \stagepair{55.91}{53.37} & \stagepair{64.64}{67.50} & \stagepair{25.99}{28.51} & \stagepair{69.80}{72.05} & \stagepair{61.48}{62.83} & \stagepair{52.11}{53.09} \\
& DFT & \stagepair{30.47}{33.54} & \stagebest{61.18}{59.10} & \stagepair{66.07}{64.29} & \stagepair{25.25}{24.88} & \stagebest{72.35}{73.18} & \stagepair{66.63}{66.34} & \stagepair{53.66}{53.55} \\
& RPSFT & \stagebest{34.21}{38.00} & \stagepair{58.32}{54.07} & \stagebest{68.21}{72.86} & \stagebest{27.84}{33.33} & \stagepair{70.45}{71.93} & \stagebest{61.40}{67.36} & \stagebest{53.41}{56.26} \\
\midrule
\multirow{4}{*}{\textbf{Pass@k}} & SFT & \stagepair{65.63}{64.73} & \stagepair{69.50}{70.06} & \stagepair{80.00}{84.29} & \stagepair{53.83}{56.05} & \stagepair{83.30}{80.40} & \stagepair{80.26}{78.51} & \stagepair{72.09}{72.34} \\
& IW & \stagepair{64.29}{60.49} & \stagepair{69.50}{68.39} & \stagebest{78.57}{85.71} & \stagepair{51.11}{54.32} & \stagebest{80.80}{81.10} & \stagepair{78.07}{76.46} & \stagepair{70.39}{71.08} \\
& DFT & \stagepair{58.04}{61.61} & \stagebest{73.01}{71.35} & \stagepair{80.00}{78.57} & \stagepair{45.93}{45.43} & \stagepair{78.30}{77.10} & \stagepair{79.68}{76.32} & \stagepair{69.16}{68.40} \\
& RPSFT & \stagebest{65.18}{66.52} & \stagepair{73.38}{70.43} & \stagepair{84.29}{82.86} & \stagebest{54.57}{58.02} & \stagepair{79.80}{79.70} & \stagebest{78.36}{79.24} & \stagebest{72.60}{72.80} \\
\midrule
\multicolumn{9}{l}{\textbf{Qwen2.5-3B-Instruct}}\\
\midrule
\multirow{4}{*}{\textbf{Avg@k}} & SFT & \stagepair{27.73}{27.12} & \stagepair{47.97}{46.77} & \stagepair{40.71}{50.71} & \stagebest{19.81}{22.84} & \stagepair{59.98}{63.63} & \stagepair{50.55}{53.18} & \stagepair{41.12}{44.04} \\
& IW & \stagepair{29.41}{29.13} & \stagepair{47.92}{47.92} & \stagepair{42.14}{48.93} & \stagepair{19.38}{20.93} & \stagepair{60.00}{61.25} & \stagepair{50.69}{53.22} & \stagepair{41.59}{43.56} \\
& DFT & \stagepair{26.45}{26.90} & \stagebest{52.59}{52.73} & \stagepair{40.71}{41.07} & \stagepair{19.26}{19.94} & \stagebest{66.03}{63.95} & \stagebest{57.93}{58.33} & \stagepair{43.83}{43.82} \\
& RPSFT & \stagebest{28.13}{31.64} & \stagepair{49.72}{49.40} & \stagebest{43.93}{51.43} & \stagepair{21.67}{22.41} & \stagepair{60.83}{63.15} & \stagepair{52.30}{54.75} & \stagebest{42.76}{45.46} \\
\midrule
\multirow{4}{*}{\textbf{Pass@k}} & SFT & \stagepair{61.61}{57.37} & \stagepair{63.03}{63.59} & \stagepair{70.00}{71.43} & \stagebest{45.68}{48.64} & \stagebest{80.20}{78.90} & \stagepair{71.35}{70.47} & \stagepair{65.31}{65.07} \\
& IW & \stagepair{65.40}{59.60} & \stagepair{62.85}{63.22} & \stagebest{68.57}{75.71} & \stagepair{43.70}{47.41} & \stagepair{69.70}{73.00} & \stagepair{69.15}{69.73} & \stagepair{63.23}{64.78} \\
& DFT & \stagepair{55.13}{52.45} & \stagebest{65.62}{67.84} & \stagepair{60.00}{64.29} & \stagepair{44.94}{42.47} & \stagepair{68.20}{66.80} & \stagebest{73.83}{75.15} & \stagepair{61.29}{61.50} \\
& RPSFT & \stagebest{62.05}{65.18} & \stagepair{65.25}{63.40} & \stagebest{72.86}{75.71} & \stagebest{48.89}{48.64} & \stagepair{76.20}{77.10} & \stagepair{71.49}{73.10} & \stagebest{66.12}{67.19} \\
\bottomrule
\end{tabular}
\renewcommand{\arraystretch}{1}
\end{table}

\looseness=-1
\textbf{RPSFT gives the strongest final OOD RL performance.} Table~\ref{tab:rl-ood} shows that RPSFT achieves the best final OOD Avg@k and Pass@k summary averages across three models. IW often has a larger init$\rightarrow$DAPO gain, suggesting that RL can partially recover OOD behavior from a weaker SFT initializer. That said, RPSFT usually starts from a stronger OOD checkpoint and remains best after RL, indicating that preserving pretrained singular structure provides a better OOD initialization.

\FloatBarrier
\subsection{Geometric and Representation Drift}
RPSFT is designed to preserve dominant pretrained structure, so we evaluate drift at both the weight and representation levels. First, we measure how much fine-tuning rotates the dominant left-singular subspaces of pretrained weights across architectures; smaller values indicate that the tuned model remains closer to the pretrained model. We then test whether this weight-space stability carries over to hidden representations by comparing each tuned model's hidden states against the base model. The rotation and hidden-state drift metrics are defined in Appendices~\ref{app:rotation-metrics} and~\ref{app:hidden-drift}.

\begin{figure}[!htbp]
\centering
\includegraphics[width=0.98\textwidth]{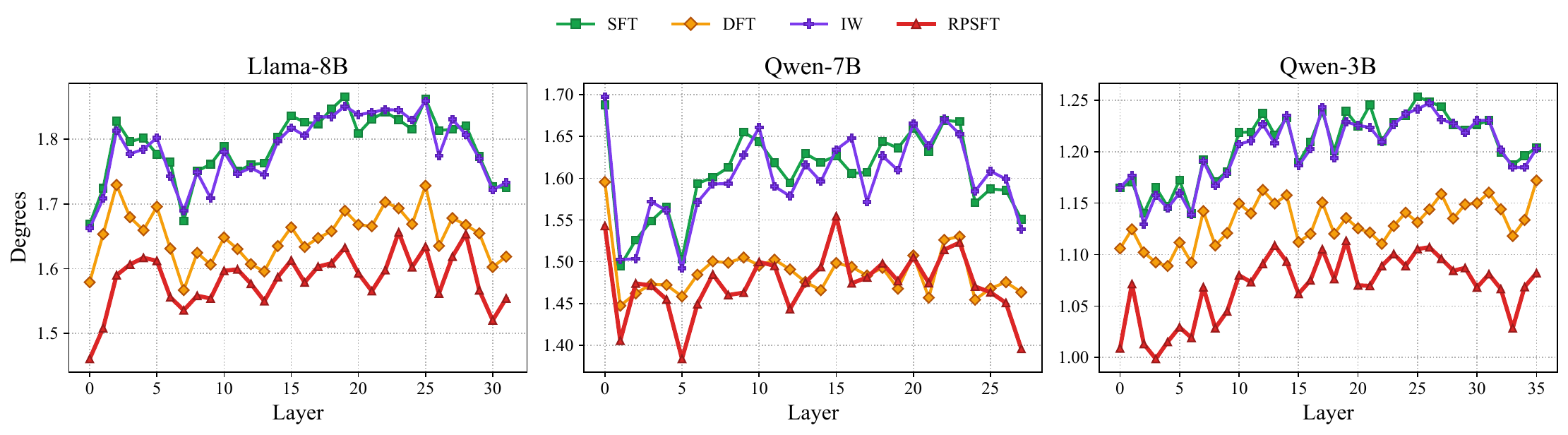}
\vspace{-0.6em}
\caption{Layerwise U-space rotation across three models. Each panel averages the principal-angle rotation over \texttt{q\_proj}, \texttt{k\_proj}, \texttt{v\_proj}, \texttt{o\_proj}, \texttt{up\_proj}, \texttt{down\_proj}, and \texttt{gate\_proj}.}
\label{fig:rotation-layerwise}
\vspace{-0.6em}
\end{figure}

\begin{figure}[htbp]
\centering
\includegraphics[width=0.98\textwidth]{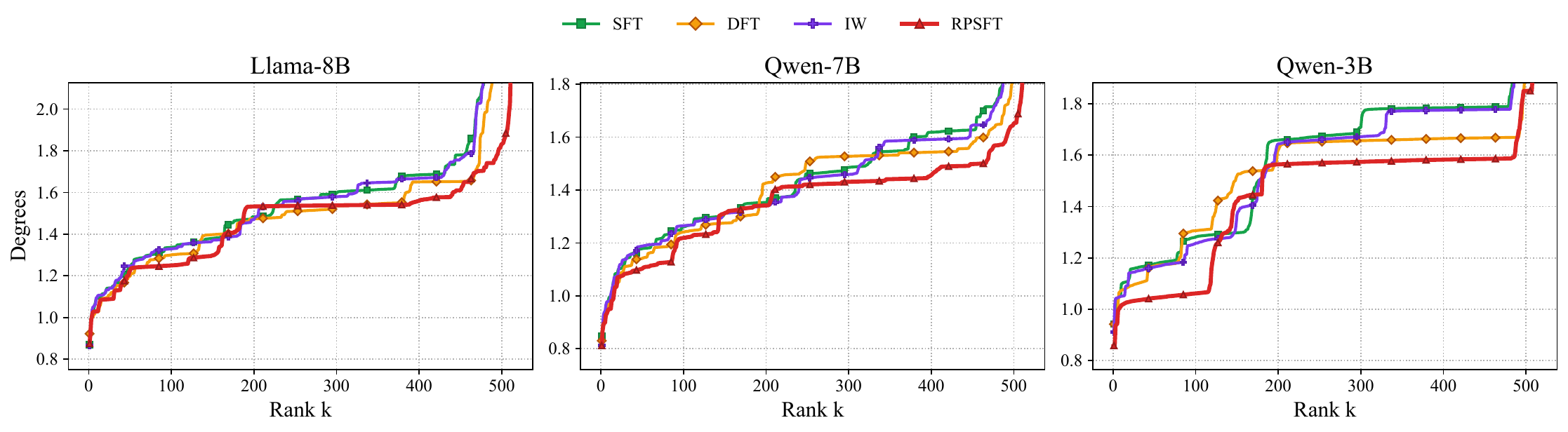}
\caption{Rankwise U-space rotation for one representative weight matrix in each model. The x-axis increases the prefix rank from 0 to 512, and the y-axis reports the mean principal-angle rotation of the first $k$ left singular vectors.}
\label{fig:rotation-rankwise}
\end{figure}
\textbf{RPSFT consistently reduces dominant-subspace rotation.}
Figures~\ref{fig:rotation-layerwise} and~\ref{fig:rotation-rankwise} show that RPSFT keeps tuned weights closer to the pretrained singular subspaces across three models. The reduction appears through most layers and across the leading-to-middle rank range, indicating that RPSFT stabilizes the dominant directions targeted by the regularizer.

\begin{figure}[htbp]
\centering
\includegraphics[width=1.00\textwidth]{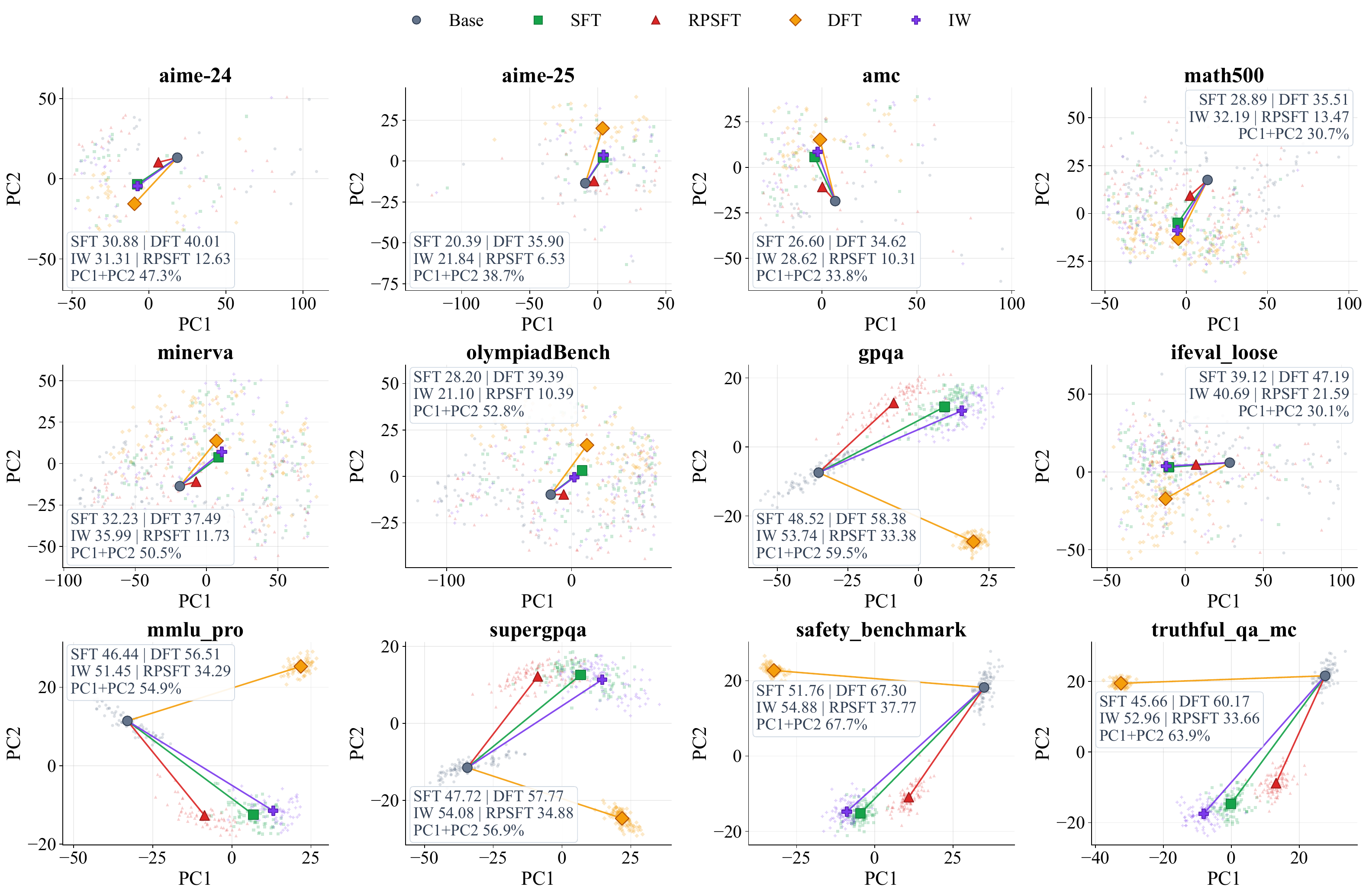}
\caption{Hidden-state drift on Qwen2.5-3B-Instruct. Each panel shows centroid shifts away from the base model in a 2D PCA view. RPSFT stays among the closest-tuned centroids to the base model, while SFT, IW, and especially DFT often shift farther. Entropy analyses are provided in Appendix~\ref{app:entropy-family}.}
\label{fig:base-shift-qwen3b}
\end{figure}

\textbf{Weight-space stability translates into representation-space stability.}
Figure~\ref{fig:base-shift-qwen3b} shows that RPSFT keeps tuned hidden-state centroids closest to the base model across benchmark panels, while SFT, IW, and especially DFT drift farther away. This suggests that constraining dominant-subspace rotation helps preserve pretrained representations rather than only reducing a weight-space metric. The same pattern holds for Llama-8B and Qwen2.5-7B in Appendix Figures~\ref{fig:base-shift-llama} and~\ref{fig:base-shift-qwen7b}.

\FloatBarrier



\section{Conclusion}
We presented RPSFT, a supervised fine-tuning regularizer that anchors the projected top-$k$ singular block of pretrained weight matrices. Across Llama and Qwen models, RPSFT improves the ID/OOD trade-off over standard SFT and strong baselines, and provides stronger initializations for downstream RLFT. Our analyses show a consistent mechanism: RPSFT reduces dominant-subspace rotation in weight space and better preserves hidden-state geometry, suggesting that controlling drift in pretrained singular subspaces is a practical way to mitigate forgetting while retaining task adaptation.

\textbf{Limitations and future work.} 
\label{con:limit_future}
RPSFT applies the same protected-rank rule broadly across weight matrices. A promising direction is to regularize layers more selectively by focusing on the early self-attention layers or other layers with the strongest Fisher--SVD overlap. This could further reduce overhead while preserving the main benefit, and is supported by recent evidence that fine-tuning effects are often layer-localized rather than uniformly distributed across the whole model~\citep{shi2025understandinglayersignificancellm,zhao2026layerwiseanalysissupervisedfinetuning}.
In terms of experiments, we focus on math training datasets with 3B--8B models; future work can extend to non-math training data and much larger model scales.

\bibliography{main}

\begin{thebibliography}{50}
\providecommand{\natexlab}[1]{#1}
\providecommand{\url}[1]{\texttt{#1}}
\expandafter\ifx\csname urlstyle\endcsname\relax
  \providecommand{\doi}[1]{doi: #1}\else
  \providecommand{\doi}{doi: \begingroup \urlstyle{rm}\Url}\fi

\bibitem[Bai et~al.(2023)Bai, Bai, Chu, Cui, Dang, Deng, Fan, Ge, Han, Huang, Hui, Ji, Li, Lin, Lin, Liu, Liu, Lu, Lu, Ma, Men, Ren, Ren, Tan, Tan, Tu, Wang, Wang, Wang, Wu, Xu, Xu, Yang, Yang, Yang, Yang, Yao, Yu, Yuan, Yuan, Zhang, Zhang, Zhang, Zhang, Zhou, Zhou, Zhou, and Zhu]{bai2023qwentechnicalreport}
Jinze Bai, Shuai Bai, Yunfei Chu, Zeyu Cui, Kai Dang, Xiaodong Deng, Yang Fan, Wenbin Ge, Yu~Han, Fei Huang, Binyuan Hui, Luo Ji, Mei Li, Junyang Lin, Runji Lin, Dayiheng Liu, Gao Liu, Chengqiang Lu, Keming Lu, Jianxin Ma, Rui Men, Xingzhang Ren, Xuancheng Ren, Chuanqi Tan, Sinan Tan, Jianhong Tu, Peng Wang, Shijie Wang, Wei Wang, Shengguang Wu, Benfeng Xu, Jin Xu, An~Yang, Hao Yang, Jian Yang, Shusheng Yang, Yang Yao, Bowen Yu, Hongyi Yuan, Zheng Yuan, Jianwei Zhang, Xingxuan Zhang, Yichang Zhang, Zhenru Zhang, Chang Zhou, Jingren Zhou, Xiaohuan Zhou, and Tianhang Zhu.
\newblock Qwen technical report, 2023.
\newblock URL \url{https://arxiv.org/abs/2309.16609}.

\bibitem[Chung et~al.(2024)Chung, Cherif, Meger, and Precup]{chung2024parsevalregularizationcontinual}
Wesley Chung, Lynn Cherif, David Meger, and Doina Precup.
\newblock Parseval regularization for continual reinforcement learning, 2024.
\newblock URL \url{https://arxiv.org/abs/2412.07224}.

\bibitem[DeepSeek-AI(2025)]{deepseekr1}
DeepSeek-AI.
\newblock Deepseek-r1: Incentivizing reasoning capability in large language models via reinforcement learning.
\newblock \emph{arXiv preprint arXiv:2501.12948}, 2025.

\bibitem[Elsayed et~al.(2024)Elsayed, Lan, Lyle, and Mahmood]{elsayed2024weightclippingdeepcontinual}
Mohamed Elsayed, Qingfeng Lan, Clare Lyle, and A.~Rupam Mahmood.
\newblock Weight clipping for deep continual and reinforcement learning, 2024.
\newblock URL \url{https://arxiv.org/abs/2407.01704}.

\bibitem[Franke et~al.(2024)Franke, Hefenbrock, and Hutter]{franke2024preserving}
J{\"o}rg~K.H. Franke, Michael Hefenbrock, and Frank Hutter.
\newblock Preserving principal subspaces to reduce catastrophic forgetting in fine-tuning.
\newblock In \emph{ICLR 2024 Workshop on Mathematical and Empirical Understanding of Foundation Models}, 2024.
\newblock URL \url{https://openreview.net/forum?id=XoWtroECJU}.

\bibitem[Grattafiori et~al.(2024)Grattafiori, Dubey, Jauhri, Pandey, Kadian, Al-Dahle, Letman, Mathur, Schelten, Vaughan, et~al.]{grattafiori2024llama}
Aaron Grattafiori, Abhimanyu Dubey, Abhinav Jauhri, Abhinav Pandey, Abhishek Kadian, Ahmad Al-Dahle, Aiesha Letman, Akhil Mathur, Alan Schelten, Alex Vaughan, et~al.
\newblock The llama 3 herd of models.
\newblock \emph{arXiv preprint arXiv:2407.21783}, 2024.

\bibitem[Haink(2023)]{haink2023hessianeigenvectorsprincipalcomponent}
David Haink.
\newblock Hessian eigenvectors and principal component analysis of neural network weight matrices, 2023.
\newblock URL \url{https://arxiv.org/abs/2311.00452}.

\bibitem[He et~al.(2024)He, Luo, Bai, Hu, Thai, Shen, Hu, Han, Huang, Zhang, Liu, Qi, Liu, and Sun]{he2024olympiadbenchchallengingbenchmarkpromoting}
Chaoqun He, Renjie Luo, Yuzhuo Bai, Shengding Hu, Zhen~Leng Thai, Junhao Shen, Jinyi Hu, Xu~Han, Yujie Huang, Yuxiang Zhang, Jie Liu, Lei Qi, Zhiyuan Liu, and Maosong Sun.
\newblock Olympiadbench: A challenging benchmark for promoting agi with olympiad-level bilingual multimodal scientific problems, 2024.
\newblock URL \url{https://arxiv.org/abs/2402.14008}.

\bibitem[Hendrycks et~al.(2021)Hendrycks, Burns, Kadavath, Arora, Basart, Tang, Song, and Steinhardt]{hendrycks2021measuringmathematicalproblemsolving}
Dan Hendrycks, Collin Burns, Saurav Kadavath, Akul Arora, Steven Basart, Eric Tang, Dawn Song, and Jacob Steinhardt.
\newblock Measuring mathematical problem solving with the math dataset, 2021.
\newblock URL \url{https://arxiv.org/abs/2103.03874}.

\bibitem[Hu et~al.(2021)Hu, Shen, Wallis, Allen-Zhu, Li, Wang, Wang, and Chen]{hu2021loralowrankadaptationlarge}
Edward~J. Hu, Yelong Shen, Phillip Wallis, Zeyuan Allen-Zhu, Yuanzhi Li, Shean Wang, Lu~Wang, and Weizhu Chen.
\newblock Lora: Low-rank adaptation of large language models, 2021.
\newblock URL \url{https://arxiv.org/abs/2106.09685}.

\bibitem[Huan et~al.(2025)Huan, Li, Zheng, Xu, Kim, Du, Poovendran, Neubig, and Yue]{huan2025mathreasoninggeneral}
Maggie Huan, Yuetai Li, Tuney Zheng, Xiaoyu Xu, Seungone Kim, Minxin Du, Radha Poovendran, Graham Neubig, and Xiang Yue.
\newblock Does math reasoning improve general llm capabilities? understanding transferability of llm reasoning, 2025.
\newblock URL \url{https://arxiv.org/abs/2507.00432}.

\bibitem[{Hugging Face}(2025)]{openr1}
{Hugging Face}.
\newblock Open r1: A fully open reproduction of deepseek-r1, January 2025.
\newblock URL \url{https://github.com/huggingface/open-r1}.

\bibitem[Jin et~al.(2025)Jin, Luan, Lyu, Rabusseau, Rabbany, Precup, and Hamdaqa]{jin2025rlfinetuninghealsood}
Hangzhan Jin, Sitao Luan, Sicheng Lyu, Guillaume Rabusseau, Reihaneh Rabbany, Doina Precup, and Mohammad Hamdaqa.
\newblock Rl fine-tuning heals ood forgetting in sft, 2025.
\newblock URL \url{https://arxiv.org/abs/2509.12235}.

\bibitem[Kirkpatrick et~al.(2017)Kirkpatrick, Pascanu, Rabinowitz, Veness, Desjardins, Rusu, Milan, Quan, Ramalho, Grabska-Barwinska, Hassabis, Clopath, Kumaran, and Hadsell]{Kirkpatrick_2017}
James Kirkpatrick, Razvan Pascanu, Neil Rabinowitz, Joel Veness, Guillaume Desjardins, Andrei~A. Rusu, Kieran Milan, John Quan, Tiago Ramalho, Agnieszka Grabska-Barwinska, Demis Hassabis, Claudia Clopath, Dharshan Kumaran, and Raia Hadsell.
\newblock Overcoming catastrophic forgetting in neural networks.
\newblock \emph{Proceedings of the National Academy of Sciences}, 114\penalty0 (13):\penalty0 3521–3526, March 2017.
\newblock ISSN 1091-6490.
\newblock \doi{10.1073/pnas.1611835114}.
\newblock URL \url{http://dx.doi.org/10.1073/pnas.1611835114}.

\bibitem[Kumar et~al.(2024)Kumar, Marklund, and Roy]{kumar2024maintainingplasticitycontinuallearning}
Saurabh Kumar, Henrik Marklund, and Benjamin~Van Roy.
\newblock Maintaining plasticity in continual learning via regenerative regularization, 2024.
\newblock URL \url{https://arxiv.org/abs/2308.11958}.

\bibitem[Kunstner et~al.(2019)Kunstner, Hennig, and Balles]{kunstner2019limitationsempiricalfisher}
Frederik Kunstner, Philipp Hennig, and Lukas Balles.
\newblock Limitations of the empirical fisher approximation for natural gradient descent.
\newblock In \emph{Advances in Neural Information Processing Systems 32}, pages 4158--4169, 2019.

\bibitem[Lai et~al.(2025)Lai, Zhao, Feng, Ma, Liu, Zhao, Lin, Yi, Xie, Zhang, Liu, Meng, and Zhu]{lai2025reinforcementfinetuningnaturallymitigates}
Song Lai, Haohan Zhao, Rong Feng, Changyi Ma, Wenzhuo Liu, Hongbo Zhao, Xi~Lin, Dong Yi, Min Xie, Qingfu Zhang, Hongbin Liu, Gaofeng Meng, and Fei Zhu.
\newblock Reinforcement fine-tuning naturally mitigates forgetting in continual post-training, 2025.
\newblock URL \url{https://arxiv.org/abs/2507.05386}.

\bibitem[Lewandowski et~al.(2024)Lewandowski, Bortkiewicz, Kumar, Gyorgy, Schuurmans, Ostaszewski, and Machado]{lewandowski2024learningcontinuallyspectral}
Alex Lewandowski, Michal Bortkiewicz, Saurabh Kumar, Andras Gyorgy, Dale Schuurmans, Mateusz Ostaszewski, and Marlos~C. Machado.
\newblock Learning continually by spectral regularization, 2024.
\newblock URL \url{https://arxiv.org/abs/2406.06811}.

\bibitem[Lewkowycz et~al.(2022)Lewkowycz, Andreassen, Dohan, Dyer, Michalewski, Ramasesh, Slone, Anil, Schlag, Gutman-Solo, Wu, Neyshabur, Gur-Ari, and Misra]{lewkowycz2022solvingquantitativereasoningproblems}
Aitor Lewkowycz, Anders Andreassen, David Dohan, Ethan Dyer, Henryk Michalewski, Vinay Ramasesh, Ambrose Slone, Cem Anil, Imanol Schlag, Theo Gutman-Solo, Yuhuai Wu, Behnam Neyshabur, Guy Gur-Ari, and Vedant Misra.
\newblock Solving quantitative reasoning problems with language models, 2022.
\newblock URL \url{https://arxiv.org/abs/2206.14858}.

\bibitem[Lin et~al.(2025)Lin, Wang, Qian, Wang, Srinivasan, Zeng, Jiao, Zhou, Gesi, Wang, Guo, Zhong, Zhang, Sanghavi, Chen, Yun, and Li]{lin2025sftdoesntalwayshurt}
Jiacheng Lin, Zhongruo Wang, Kun Qian, Tian Wang, Arvind Srinivasan, Hansi Zeng, Ruochen Jiao, Xie Zhou, Jiri Gesi, Dakuo Wang, Yufan Guo, Kai Zhong, Weiqi Zhang, Sujay Sanghavi, Changyou Chen, Hyokun Yun, and Lihong Li.
\newblock Sft doesn't always hurt general capabilities: Revisiting domain-specific fine-tuning in llms, 2025.
\newblock URL \url{https://arxiv.org/abs/2509.20758}.

\bibitem[Lin et~al.(2022)Lin, Hilton, and Evans]{lin2022truthfulqameasuringmodelsmimic}
Stephanie Lin, Jacob Hilton, and Owain Evans.
\newblock Truthfulqa: Measuring how models mimic human falsehoods, 2022.
\newblock URL \url{https://arxiv.org/abs/2109.07958}.

\bibitem[Lu et~al.(2025)Lu, Yuan, Feng, and Sun]{lu2025rethinkingstabilityplasticitytradeoffcontinual}
Aojun Lu, Hangjie Yuan, Tao Feng, and Yanan Sun.
\newblock Rethinking the stability-plasticity trade-off in continual learning from an architectural perspective, 2025.
\newblock URL \url{https://arxiv.org/abs/2506.03951}.

\bibitem[Lyle et~al.(2023)Lyle, Zheng, Nikishin, Pires, Pascanu, and Dabney]{lyle2023understandingplasticityneuralnetworks}
Clare Lyle, Zeyu Zheng, Evgenii Nikishin, Bernardo~Avila Pires, Razvan Pascanu, and Will Dabney.
\newblock Understanding plasticity in neural networks, 2023.
\newblock URL \url{https://arxiv.org/abs/2303.01486}.

\bibitem[Martens(2010)]{martens2010deeplearninghessianfree}
James Martens.
\newblock Deep learning via hessian-free optimization.
\newblock In \emph{Proceedings of the 27th International Conference on Machine Learning}, pages 735--742, 2010.

\bibitem[Martens(2020)]{martens2020newinsightsperspectivesnatural}
James Martens.
\newblock New insights and perspectives on the natural gradient method, 2020.
\newblock URL \url{https://arxiv.org/abs/1412.1193}.

\bibitem[Meng et~al.(2025)Meng, Wang, and Zhang]{meng2025pissaprincipalsingularvalues}
Fanxu Meng, Zhaohui Wang, and Muhan Zhang.
\newblock Pissa: Principal singular values and singular vectors adaptation of large language models, 2025.
\newblock URL \url{https://arxiv.org/abs/2404.02948}.

\bibitem[Mukherjee et~al.(2025)Mukherjee, Yuan, Hakkani-Tur, and Peng]{mukherjee2025reinforcementlearningfinetunes}
Sagnik Mukherjee, Lifan Yuan, Dilek Hakkani-Tur, and Hao Peng.
\newblock Reinforcement learning finetunes small subnetworks in large language models, 2025.
\newblock URL \url{https://arxiv.org/abs/2505.11711}.

\bibitem[Nayak et~al.(2025)Nayak, Killamsetty, Han, Bhandwaldar, Chanda, Xu, Wang, Pareja, Silkin, Eyceoz, and Srivastava]{nayak2025sculptingsubspacesconstrainedfinetuning}
Nikhil~Shivakumar Nayak, Krishnateja Killamsetty, Ligong Han, Abhishek Bhandwaldar, Prateek Chanda, Kai Xu, Hao Wang, Aldo Pareja, Oleg Silkin, Mustafa Eyceoz, and Akash Srivastava.
\newblock Sculpting subspaces: Constrained full fine-tuning in llms for continual learning, 2025.
\newblock URL \url{https://arxiv.org/abs/2504.07097}.

\bibitem[Ni et~al.(2025)Ni, Nie, Chaudhary, Liu, Rangwala, and Fakoor]{ni2025offlinelearningforgettingreasoning}
Tianwei Ni, Allen Nie, Sapana Chaudhary, Yao Liu, Huzefa Rangwala, and Rasool Fakoor.
\newblock Offline learning and forgetting for reasoning with large language models, 2025.
\newblock URL \url{https://arxiv.org/abs/2504.11364}.

\bibitem[OpenAI et~al.(2024)OpenAI, :, Jaech, Kalai, Lerer, Richardson, El-Kishky, Low, Helyar, Madry, Beutel, Carney, Iftimie, Karpenko, Passos, Neitz, Prokofiev, Wei, Tam, Bennett, Kumar, Saraiva, Vallone, Duberstein, Kondrich, Mishchenko, Applebaum, Jiang, Nair, Zoph, Ghorbani, Rossen, Sokolowsky, Barak, McGrew, Minaiev, Hao, Baker, Houghton, McKinzie, Eastman, Lugaresi, Bassin, Hudson, Li, de~Bourcy, Voss, Shen, Zhang, Koch, Orsinger, Hesse, Fischer, Chan, Roberts, Kappler, Levy, Selsam, Dohan, Farhi, Mely, Robinson, Tsipras, Li, Oprica, Freeman, Zhang, Wong, Proehl, Cheung, Mitchell, Wallace, Ritter, Mays, Wang, Such, Raso, Leoni, Tsimpourlas, Song, von Lohmann, Sulit, Salmon, Parascandolo, Chabot, Zhao, Brockman, Leclerc, Salman, Bao, Sheng, Andrin, Bagherinezhad, Ren, Lightman, Chung, Kivlichan, O'Connell, Osband, Gilaberte, Akkaya, Kostrikov, Sutskever, Kofman, Pachocki, Lennon, Wei, Harb, Twore, Feng, Yu, Weng, Tang, Yu, Candela, Palermo, Parish, Heidecke, Hallman, Rizzo, Gordon, Uesato, Ward,
  Huizinga, Wang, Chen, Xiao, Singhal, Nguyen, Cobbe, Shi, Wood, Rimbach, Gu-Lemberg, Liu, Lu, Stone, Yu, Ahmad, Yang, Liu, Maksin, Ho, Fedus, Weng, Li, McCallum, Held, Kuhn, Kondraciuk, Kaiser, Metz, Boyd, Trebacz, Joglekar, Chen, Tintor, Meyer, Jones, Kaufer, Schwarzer, Shah, Yatbaz, Guan, Xu, Yan, Glaese, Chen, Lampe, Malek, Wang, Fradin, McClay, Pavlov, Wang, Wang, Murati, Bavarian, Rohaninejad, McAleese, Chowdhury, Chowdhury, Ryder, Tezak, Brown, Nachum, Boiko, Murk, Watkins, Chao, Ashbourne, Izmailov, Zhokhov, Dias, Arora, Lin, Lopes, Gaon, Miyara, Leike, Hwang, Garg, Brown, James, Shu, Cheu, Greene, Jain, Altman, Toizer, Toyer, Miserendino, Agarwal, Hernandez, Baker, McKinney, Yan, Zhao, Hu, Santurkar, Chaudhuri, Zhang, Fu, Papay, Lin, Balaji, Sanjeev, Sidor, Broda, Clark, Wang, Gordon, Sanders, Patwardhan, Sottiaux, Degry, Dimson, Zheng, Garipov, Stasi, Bansal, Creech, Peterson, Eloundou, Qi, Kosaraju, Monaco, Pong, Fomenko, Zheng, Zhou, McCabe, Zaremba, Dubois, Lu, Chen, Cha, Bai, He, Zhang, Wang,
  Shao, and Li]{openai2024openaio1card}
OpenAI, :, Aaron Jaech, Adam Kalai, Adam Lerer, Adam Richardson, Ahmed El-Kishky, Aiden Low, Alec Helyar, Aleksander Madry, Alex Beutel, Alex Carney, Alex Iftimie, Alex Karpenko, Alex~Tachard Passos, Alexander Neitz, Alexander Prokofiev, Alexander Wei, Allison Tam, Ally Bennett, Ananya Kumar, Andre Saraiva, Andrea Vallone, Andrew Duberstein, Andrew Kondrich, Andrey Mishchenko, Andy Applebaum, Angela Jiang, Ashvin Nair, Barret Zoph, Behrooz Ghorbani, Ben Rossen, Benjamin Sokolowsky, Boaz Barak, Bob McGrew, Borys Minaiev, Botao Hao, Bowen Baker, Brandon Houghton, Brandon McKinzie, Brydon Eastman, Camillo Lugaresi, Cary Bassin, Cary Hudson, Chak~Ming Li, Charles de~Bourcy, Chelsea Voss, Chen Shen, Chong Zhang, Chris Koch, Chris Orsinger, Christopher Hesse, Claudia Fischer, Clive Chan, Dan Roberts, Daniel Kappler, Daniel Levy, Daniel Selsam, David Dohan, David Farhi, David Mely, David Robinson, Dimitris Tsipras, Doug Li, Dragos Oprica, Eben Freeman, Eddie Zhang, Edmund Wong, Elizabeth Proehl, Enoch Cheung, Eric
  Mitchell, Eric Wallace, Erik Ritter, Evan Mays, Fan Wang, Felipe~Petroski Such, Filippo Raso, Florencia Leoni, Foivos Tsimpourlas, Francis Song, Fred von Lohmann, Freddie Sulit, Geoff Salmon, Giambattista Parascandolo, Gildas Chabot, Grace Zhao, Greg Brockman, Guillaume Leclerc, Hadi Salman, Haiming Bao, Hao Sheng, Hart Andrin, Hessam Bagherinezhad, Hongyu Ren, Hunter Lightman, Hyung~Won Chung, Ian Kivlichan, Ian O'Connell, Ian Osband, Ignasi~Clavera Gilaberte, Ilge Akkaya, Ilya Kostrikov, Ilya Sutskever, Irina Kofman, Jakub Pachocki, James Lennon, Jason Wei, Jean Harb, Jerry Twore, Jiacheng Feng, Jiahui Yu, Jiayi Weng, Jie Tang, Jieqi Yu, Joaquin~Quiñonero Candela, Joe Palermo, Joel Parish, Johannes Heidecke, John Hallman, John Rizzo, Jonathan Gordon, Jonathan Uesato, Jonathan Ward, Joost Huizinga, Julie Wang, Kai Chen, Kai Xiao, Karan Singhal, Karina Nguyen, Karl Cobbe, Katy Shi, Kayla Wood, Kendra Rimbach, Keren Gu-Lemberg, Kevin Liu, Kevin Lu, Kevin Stone, Kevin Yu, Lama Ahmad, Lauren Yang, Leo Liu,
  Leon Maksin, Leyton Ho, Liam Fedus, Lilian Weng, Linden Li, Lindsay McCallum, Lindsey Held, Lorenz Kuhn, Lukas Kondraciuk, Lukasz Kaiser, Luke Metz, Madelaine Boyd, Maja Trebacz, Manas Joglekar, Mark Chen, Marko Tintor, Mason Meyer, Matt Jones, Matt Kaufer, Max Schwarzer, Meghan Shah, Mehmet Yatbaz, Melody~Y. Guan, Mengyuan Xu, Mengyuan Yan, Mia Glaese, Mianna Chen, Michael Lampe, Michael Malek, Michele Wang, Michelle Fradin, Mike McClay, Mikhail Pavlov, Miles Wang, Mingxuan Wang, Mira Murati, Mo~Bavarian, Mostafa Rohaninejad, Nat McAleese, Neil Chowdhury, Neil Chowdhury, Nick Ryder, Nikolas Tezak, Noam Brown, Ofir Nachum, Oleg Boiko, Oleg Murk, Olivia Watkins, Patrick Chao, Paul Ashbourne, Pavel Izmailov, Peter Zhokhov, Rachel Dias, Rahul Arora, Randall Lin, Rapha~Gontijo Lopes, Raz Gaon, Reah Miyara, Reimar Leike, Renny Hwang, Rhythm Garg, Robin Brown, Roshan James, Rui Shu, Ryan Cheu, Ryan Greene, Saachi Jain, Sam Altman, Sam Toizer, Sam Toyer, Samuel Miserendino, Sandhini Agarwal, Santiago Hernandez,
  Sasha Baker, Scott McKinney, Scottie Yan, Shengjia Zhao, Shengli Hu, Shibani Santurkar, Shraman~Ray Chaudhuri, Shuyuan Zhang, Siyuan Fu, Spencer Papay, Steph Lin, Suchir Balaji, Suvansh Sanjeev, Szymon Sidor, Tal Broda, Aidan Clark, Tao Wang, Taylor Gordon, Ted Sanders, Tejal Patwardhan, Thibault Sottiaux, Thomas Degry, Thomas Dimson, Tianhao Zheng, Timur Garipov, Tom Stasi, Trapit Bansal, Trevor Creech, Troy Peterson, Tyna Eloundou, Valerie Qi, Vineet Kosaraju, Vinnie Monaco, Vitchyr Pong, Vlad Fomenko, Weiyi Zheng, Wenda Zhou, Wes McCabe, Wojciech Zaremba, Yann Dubois, Yinghai Lu, Yining Chen, Young Cha, Yu~Bai, Yuchen He, Yuchen Zhang, Yunyun Wang, Zheng Shao, and Zhuohan Li.
\newblock Openai o1 system card, 2024.
\newblock URL \url{https://arxiv.org/abs/2412.16720}.

\bibitem[Qin and Springenberg(2025)]{qin2025supervisedfinetuningcurated}
Chongli Qin and Jost~Tobias Springenberg.
\newblock Supervised fine tuning on curated data is reinforcement learning (and can be improved), 2025.
\newblock URL \url{https://arxiv.org/abs/2507.12856}.

\bibitem[Rein et~al.(2023)Rein, Hou, Stickland, Petty, Pang, Dirani, Michael, and Bowman]{rein2023gpqagraduatelevelgoogleproofqa}
David Rein, Betty~Li Hou, Asa~Cooper Stickland, Jackson Petty, Richard~Yuanzhe Pang, Julien Dirani, Julian Michael, and Samuel~R. Bowman.
\newblock Gpqa: A graduate-level google-proof q\&a benchmark, 2023.
\newblock URL \url{https://arxiv.org/abs/2311.12022}.

\bibitem[Schulman et~al.(2017)Schulman, Wolski, Dhariwal, Radford, and Klimov]{schulman2017proximalpolicyoptimizationalgorithms}
John Schulman, Filip Wolski, Prafulla Dhariwal, Alec Radford, and Oleg Klimov.
\newblock Proximal policy optimization algorithms, 2017.
\newblock URL \url{https://arxiv.org/abs/1707.06347}.

\bibitem[Shao et~al.(2024)Shao, Wang, Zhu, Xu, Song, Bi, Zhang, Zhang, Li, Wu, and Guo]{shao2024deepseekmathpushinglimitsmathematical}
Zhihong Shao, Peiyi Wang, Qihao Zhu, Runxin Xu, Junxiao Song, Xiao Bi, Haowei Zhang, Mingchuan Zhang, Y.~K. Li, Y.~Wu, and Daya Guo.
\newblock Deepseekmath: Pushing the limits of mathematical reasoning in open language models, 2024.
\newblock URL \url{https://arxiv.org/abs/2402.03300}.

\bibitem[Shenfeld et~al.(2025)Shenfeld, Pari, and Agrawal]{shenfeld2025rlsrazoronlinereinforcement}
Idan Shenfeld, Jyothish Pari, and Pulkit Agrawal.
\newblock Rl's razor: Why online reinforcement learning forgets less, 2025.
\newblock URL \url{https://arxiv.org/abs/2509.04259}.

\bibitem[Shi et~al.(2025)Shi, Lu, Dong, Zhang, Zhang, Feng, and Wu]{shi2025understandinglayersignificancellm}
Guangyuan Shi, Zexin Lu, Xiaoyu Dong, Wenlong Zhang, Xuanyu Zhang, Yujie Feng, and Xiao-Ming Wu.
\newblock Understanding layer significance in llm alignment, 2025.
\newblock URL \url{https://arxiv.org/abs/2410.17875}.

\bibitem[Team et~al.(2025)Team, Du, Yao, Ma, Wang, Zheng, Zhu, Liu, Liang, Jin, Wei, Zheng, Deng, Gavin, Jia, Jiang, Liao, Li, Li, Li, Li, Li, Ma, Ni, Que, Wang, Wen, Wu, Hsing, Xu, Yang, Wang, Zhou, Bai, Bu, Cai, Chen, Chen, Cheng, Cheng, Ding, Huang, Huang, Li, Li, Li, Liang, Lin, Lin, Ma, Pang, Peng, Peng, Qi, Qiu, Qu, Quan, Tan, Wang, Wang, Wang, Wang, Wang, Xu, Yang, Yuan, Yue, Zhan, Zhang, Zhang, Zhang, Zhang, Zhang, Zhao, Zheng, Zhong, Gao, Li, Liu, Liu, Liu, Ni, Peng, Qin, Su, Wang, Wang, Yang, Yang, Cao, Yue, Zhang, Zhou, Liu, Lin, Huang, and Zhang]{pteam2025supergpqascalingllmevaluation}
P~Team, Xinrun Du, Yifan Yao, Kaijing Ma, Bingli Wang, Tianyu Zheng, King Zhu, Minghao Liu, Yiming Liang, Xiaolong Jin, Zhenlin Wei, Chujie Zheng, Kaixin Deng, Shawn Gavin, Shian Jia, Sichao Jiang, Yiyan Liao, Rui Li, Qinrui Li, Sirun Li, Yizhi Li, Yunwen Li, David Ma, Yuansheng Ni, Haoran Que, Qiyao Wang, Zhoufutu Wen, Siwei Wu, Tyshawn Hsing, Ming Xu, Zhenzhu Yang, Zekun~Moore Wang, Junting Zhou, Yuelin Bai, Xingyuan Bu, Chenglin Cai, Liang Chen, Yifan Chen, Chengtuo Cheng, Tianhao Cheng, Keyi Ding, Siming Huang, Yun Huang, Yaoru Li, Yizhe Li, Zhaoqun Li, Tianhao Liang, Chengdong Lin, Hongquan Lin, Yinghao Ma, Tianyang Pang, Zhongyuan Peng, Zifan Peng, Qige Qi, Shi Qiu, Xingwei Qu, Shanghaoran Quan, Yizhou Tan, Zili Wang, Chenqing Wang, Hao Wang, Yiya Wang, Yubo Wang, Jiajun Xu, Kexin Yang, Ruibin Yuan, Yuanhao Yue, Tianyang Zhan, Chun Zhang, Jinyang Zhang, Xiyue Zhang, Xingjian Zhang, Yue Zhang, Yongchi Zhao, Xiangyu Zheng, Chenghua Zhong, Yang Gao, Zhoujun Li, Dayiheng Liu, Qian Liu, Tianyu Liu, Shiwen
  Ni, Junran Peng, Yujia Qin, Wenbo Su, Guoyin Wang, Shi Wang, Jian Yang, Min Yang, Meng Cao, Xiang Yue, Zhaoxiang Zhang, Wangchunshu Zhou, Jiaheng Liu, Qunshu Lin, Wenhao Huang, and Ge~Zhang.
\newblock Supergpqa: Scaling llm evaluation across 285 graduate disciplines, 2025.
\newblock URL \url{https://arxiv.org/abs/2502.14739}.

\bibitem[Team(2024)]{qwen2.5}
Qwen Team.
\newblock Qwen2.5: A party of foundation models, September 2024.
\newblock URL \url{https://qwenlm.github.io/blog/qwen2.5/}.

\bibitem[Wang et~al.(2025{\natexlab{a}})Wang, Lyu, Sun, and Jing]{wang2025continualgradientlowrankprojection}
Chenxu Wang, Yilin Lyu, Zicheng Sun, and Liping Jing.
\newblock Continual gradient low-rank projection fine-tuning for llms, 2025{\natexlab{a}}.
\newblock URL \url{https://arxiv.org/abs/2507.02503}.

\bibitem[Wang et~al.(2025{\natexlab{b}})Wang, Li, Wang, Chen, and Chen]{wang2025miloraharnessingminorsingular}
Hanqing Wang, Yixia Li, Shuo Wang, Guanhua Chen, and Yun Chen.
\newblock Milora: Harnessing minor singular components for parameter-efficient llm finetuning, 2025{\natexlab{b}}.
\newblock URL \url{https://arxiv.org/abs/2406.09044}.

\bibitem[Wang et~al.(2024{\natexlab{a}})Wang, Zhang, Su, and Zhu]{wang2024comprehensivesurveycontinuallearning}
Liyuan Wang, Xingxing Zhang, Hang Su, and Jun Zhu.
\newblock A comprehensive survey of continual learning: Theory, method and application, 2024{\natexlab{a}}.
\newblock URL \url{https://arxiv.org/abs/2302.00487}.

\bibitem[Wang et~al.(2023)Wang, Chen, Ge, Xia, Bao, Zheng, Zhang, Gui, and Huang]{wang2023orthogonalsubspacelearninglanguage}
Xiao Wang, Tianze Chen, Qiming Ge, Han Xia, Rong Bao, Rui Zheng, Qi~Zhang, Tao Gui, and Xuanjing Huang.
\newblock Orthogonal subspace learning for language model continual learning, 2023.
\newblock URL \url{https://arxiv.org/abs/2310.14152}.

\bibitem[Wang et~al.(2024{\natexlab{b}})Wang, Ma, Zhang, Ni, Chandra, Guo, Ren, Arulraj, He, Jiang, Li, Ku, Wang, Zhuang, Fan, Yue, and Chen]{wang2024mmluprorobustchallengingmultitask}
Yubo Wang, Xueguang Ma, Ge~Zhang, Yuansheng Ni, Abhranil Chandra, Shiguang Guo, Weiming Ren, Aaran Arulraj, Xuan He, Ziyan Jiang, Tianle Li, Max Ku, Kai Wang, Alex Zhuang, Rongqi Fan, Xiang Yue, and Wenhu Chen.
\newblock Mmlu-pro: A more robust and challenging multi-task language understanding benchmark, 2024{\natexlab{b}}.
\newblock URL \url{https://arxiv.org/abs/2406.01574}.

\bibitem[Wu et~al.(2025)Wu, Zhou, Ziheng, Peng, Ye, Hu, Zhu, Qi, Yang, and Yang]{wu2025generalizationsftreinforcementlearning}
Yongliang Wu, Yizhou Zhou, Zhou Ziheng, Yingzhe Peng, Xinyu Ye, Xinting Hu, Wenbo Zhu, Lu~Qi, Ming-Hsuan Yang, and Xu~Yang.
\newblock On the generalization of sft: A reinforcement learning perspective with reward rectification, 2025.
\newblock URL \url{https://arxiv.org/abs/2508.05629}.

\bibitem[Yu et~al.(2025)Yu, Zhang, Zhu, Yuan, Zuo, Yue, Dai, Fan, Liu, Liu, Liu, Lin, Lin, Ma, Sheng, Tong, Zhang, Zhang, Zhang, Zhu, Zhu, Chen, Chen, Wang, Yu, Song, Wei, Zhou, Liu, Ma, Zhang, Yan, Qiao, Wu, and Wang]{yu2025dapoopensourcellmreinforcement}
Qiying Yu, Zheng Zhang, Ruofei Zhu, Yufeng Yuan, Xiaochen Zuo, Yu~Yue, Weinan Dai, Tiantian Fan, Gaohong Liu, Lingjun Liu, Xin Liu, Haibin Lin, Zhiqi Lin, Bole Ma, Guangming Sheng, Yuxuan Tong, Chi Zhang, Mofan Zhang, Wang Zhang, Hang Zhu, Jinhua Zhu, Jiaze Chen, Jiangjie Chen, Chengyi Wang, Hongli Yu, Yuxuan Song, Xiangpeng Wei, Hao Zhou, Jingjing Liu, Wei-Ying Ma, Ya-Qin Zhang, Lin Yan, Mu~Qiao, Yonghui Wu, and Mingxuan Wang.
\newblock Dapo: An open-source llm reinforcement learning system at scale, 2025.
\newblock URL \url{https://arxiv.org/abs/2503.14476}.

\bibitem[Zhao et~al.(2026)Zhao, Gong, Chen, Kang, and Li]{zhao2026layerwiseanalysissupervisedfinetuning}
Qinghua Zhao, Xueling Gong, Xinyu Chen, Zhongfeng Kang, and Xinlu Li.
\newblock A layer-wise analysis of supervised fine-tuning, 2026.
\newblock URL \url{https://arxiv.org/abs/2604.11838}.

\bibitem[Zhou et~al.(2023)Zhou, Lu, Mishra, Brahma, Basu, Luan, Zhou, and Hou]{zhou2023instructionfollowingevaluationlargelanguage}
Jeffrey Zhou, Tianjian Lu, Swaroop Mishra, Siddhartha Brahma, Sujoy Basu, Yi~Luan, Denny Zhou, and Le~Hou.
\newblock Instruction-following evaluation for large language models, 2023.
\newblock URL \url{https://arxiv.org/abs/2311.07911}.

\bibitem[Zhu et~al.(2025{\natexlab{a}})Zhu, Zhang, Huang, Su, Liu, Zhao, Fedorov, Pirsiavash, Lee, Pan, Wang, Tian, and Tai]{zhu2025whyrlupdateslooksparse}
Hanqing Zhu, Zhenyu Zhang, Hanxian Huang, DiJia Su, Zechun Liu, Jiawei Zhao, Igor Fedorov, Hamed Pirsiavash, Jinwon Lee, David~Z. Pan, Zhangyang Wang, Yuandong Tian, and Kai~Sheng Tai.
\newblock Why rl updates look sparse: An implicit compass drives optimization bias, 2025{\natexlab{a}}.
\newblock URL \url{https://openreview.net/forum?id=Q4mF4tLGbf}.
\newblock Submitted to ICLR 2026.

\bibitem[Zhu et~al.(2025{\natexlab{b}})Zhu, Zhang, Huang, Su, Liu, Zhao, Fedorov, Pirsiavash, Sha, Lee, Pan, Wang, Tian, and Tai]{zhu2025pathtakenrlvrprovably}
Hanqing Zhu, Zhenyu Zhang, Hanxian Huang, DiJia Su, Zechun Liu, Jiawei Zhao, Igor Fedorov, Hamed Pirsiavash, Zhizhou Sha, Jinwon Lee, David~Z. Pan, Zhangyang Wang, Yuandong Tian, and Kai~Sheng Tai.
\newblock The path not taken: Rlvr provably learns off the principals, 2025{\natexlab{b}}.
\newblock URL \url{https://arxiv.org/abs/2511.08567}.

\bibitem[Zhu et~al.(2025{\natexlab{c}})Zhu, Xie, Wang, Sun, Wang, and Liu]{zhu2025proximalsupervisedfinetuning}
Wenhong Zhu, Ruobing Xie, Rui Wang, Xingwu Sun, Di~Wang, and Pengfei Liu.
\newblock Proximal supervised fine-tuning, 2025{\natexlab{c}}.
\newblock URL \url{https://arxiv.org/abs/2508.17784}.

\end{thebibliography}
\bibliographystyle{plainnat}

\clearpage
\onecolumn
\appendix

\startcontents[appendix]
\section*{Appendix Contents}
\printcontents[appendix]{}{1}{\setcounter{tocdepth}{2}}
\clearpage

\section{Algorithms}

\begin{algorithm}[H]
\caption{Rotation-Preserving SFT (RPSFT)}
\label{alg:fpft}
\begin{algorithmic}[1]
\REQUIRE Pretrained weights $\theta_0 = \{\mathbf{W}^0_\ell\}$, SFT dataset $\mathcal{D}$, regularized layer set $\mathcal{M}_{\mathrm{reg}}$, rank $k$, regularization coefficient $\lambda$, update period $s$
\STATE \textbf{Precompute:} for each $\ell\in\mathcal{M}_{\mathrm{reg}}$, compute $(\mathbf{U}^{(k)}_{0,\ell}, \mathbf{V}^{(k)}_{0,\ell})$ from the SVD of $\mathbf{W}^0_\ell$; set $S^{\mathrm{ref}}_\ell \leftarrow (\mathbf{U}^{(k)}_{0,\ell})^\top \mathbf{W}^0_\ell \mathbf{V}^{(k)}_{0,\ell}$
\STATE Initialize $\theta \leftarrow \theta_0$
\FOR{supervised training step $t=1,2,\dots$}
  \STATE Sample minibatch $(x,y)\sim\mathcal{D}$
  \STATE Compute the standard SFT loss $\mathcal{L}_{\mathrm{SFT}}$
  \STATE $\mathcal{L}_{\mathrm{reg}} \leftarrow 0$
  \FOR{each $\ell \in \mathcal{M}_{\mathrm{reg}}$}
    \STATE $S \leftarrow (\mathbf{U}^{(k)}_{0,\ell})^\top \mathbf{W}_\ell \mathbf{V}^{(k)}_{0,\ell}$
    \STATE $\mathcal{L}_{\mathrm{reg}} \leftarrow \mathcal{L}_{\mathrm{reg}} + \|S - S^{\mathrm{ref}}_\ell\|_F^2$
  \ENDFOR
  \STATE $\mathcal{L} \leftarrow \mathcal{L}_{\mathrm{SFT}} + \lambda \mathcal{L}_{\mathrm{reg}}$
  \STATE Update $\theta \leftarrow \theta - \eta \nabla_\theta \mathcal{L}$
\ENDFOR
\STATE \textbf{return} $\theta$
\end{algorithmic}
\end{algorithm}

\noindent\textbf{Illustrative PyTorch implementation.}
The core implementation only caches the pretrained singular bases once and adds a projected-block penalty to the usual SFT loss:
\begin{lstlisting}[language=Python]

# Precompute once for each selected pretrained weight W0
svd_cache = []
with torch.no_grad():
    for W0 in pretrained_layers:
        U, _, Vh = torch.linalg.svd(W0.float(), full_matrices=False)
        Uk = U[:, :k].to(device=W0.device, dtype=W0.dtype)
        Vk = Vh[:k].T.to(device=W0.device, dtype=W0.dtype)
        S0 = Uk.T @ W0 @ Vk
        svd_cache.append((Uk, Vk, S0))

# During SFT, add the projected-block penalty every s steps
reg_loss = task_loss.new_zeros(())
if step % s == 0:
    for W, (Uk, Vk, S0) in zip(model_layers, svd_cache):
        S = Uk.T @ W @ Vk
        reg_loss = reg_loss + (S - S0).pow(2).sum()

loss = task_loss + lam * reg_loss
loss.backward()

\end{lstlisting}
\subsection{SVD Notation}
\label{app:svd-background}

For a weight matrix $\mathbf{W} \in \mathbb{R}^{d_{\text{out}} \times d_{\text{in}}}$, singular value decomposition writes $\mathbf{W} = \mathbf{U}\mathbf{\Sigma}\mathbf{V}^\top$, where $d_{\text{out}}$ and $d_{\text{in}}$ are the output and input dimensions, $\mathbf{U}\in\mathbb{R}^{d_{\text{out}}\times r}$ and $\mathbf{V}\in\mathbb{R}^{d_{\text{in}}\times r}$ have orthonormal columns, $\mathbf{\Sigma}\in\mathbb{R}^{r\times r}$ is the diagonal matrix of singular values in descending order, and $r=\operatorname{rank}(\mathbf{W})$. The leading singular directions capture the dominant transformations implemented by the model and define the pretrained subspaces that RPSFT protects.

\subsection{LoRA Objective}
\label{app:lora-objective}

For completeness, RPSFT can be applied to LoRA by adding the same projected-block drift penalty to the effective weight. This optional variant is not used in the reported experiments; the LoRA rows in Appendix~\ref{app:robustness-extra} are vanilla LoRA baselines with adapter rank $r=32$.
\begin{equation}
\label{eq:lora_obj}
\mathcal{L}(A,B)
=
\mathcal{L}_{\text{task}}(A,B)
+
\lambda \sum_{\ell \in \mathcal{M}'}
\left\|
S_\ell(\mathbf{W}^0_\ell + \Delta \mathbf{W}_\ell) - S^{\mathrm{ref}}_\ell
\right\|_F^2,
\end{equation}
where $\Delta \mathbf{W}_\ell = B_\ell A_\ell$. This keeps the regularizer aligned with
FPFT: it penalizes drift inside the pretrained dominant subspace even when only
low-rank parameters are trained.


\section{Related Work}
\label{app:related-work}

\subsection{Traditional Continual-Learning Regularization}

Continual-learning surveys commonly separate replay, architecture expansion, and regularization-based approaches \citep{wang2024comprehensivesurveycontinuallearning}. Our work is closest to the regularization family: rather than storing old data or changing the model architecture, these methods modify the optimization objective or update rule to control parameter drift. A classical example is Elastic Weight Consolidation (EWC) \citep{Kirkpatrick_2017}, which uses a Fisher-weighted quadratic penalty to discourage movement of important parameters. A simpler related baseline is anchoring or regenerative regularization, which penalizes uniform Euclidean movement from a reference model, typically through a term proportional to $\|\theta-\theta_0\|_2^2$ \citep{kumar2024maintainingplasticitycontinuallearning}. Other recent continual-learning regularizers control the weight geometry more directly: spectral regularization keeps the maximum singular value of each layer near one \citep{lewandowski2024learningcontinuallyspectral}, weight clipping bounds weight magnitudes after each update \citep{elsayed2024weightclippingdeepcontinual}, and Parseval regularization encourages near-orthogonal weight matrices \citep{chung2024parsevalregularizationcontinual}. These methods motivate regularizing the geometry of the update, but they do not target the pretrained singular directions that are most tied to OOD forgetting in our analysis.

\subsection{LLM Post-Training and Reasoning Generalization}

Recent LLM studies show that reasoning post-training can either generalize or over-specialize depending on the learning signal and update geometry. Other work shows that math-only post-training often transfers imperfectly: SFT can induce representation and output drift, whereas RL tends to preserve broader capabilities \citep{huan2025mathreasoninggeneral}. This is consistent with findings that RL fine-tuning mitigates forgetting through reward-variance scaling \citep{lai2025reinforcementfinetuningnaturallymitigates}, that online RL is biased toward KL-minimal solutions \citep{shenfeld2025rlsrazoronlinereinforcement}, and that RL updates may be sparse, spectrum-preserving, or off-principal in parameter space \citep{mukherjee2025reinforcementlearningfinetunes,zhu2025whyrlupdateslooksparse,zhu2025pathtakenrlvrprovably}. Offline reasoning studies also show that naive SFT can harm search capability and that smaller learning rates can reduce degradation \citep{ni2025offlinelearningforgettingreasoning,lin2025sftdoesntalwayshurt}. These results motivate our focus on the SFT stage: RPSFT explicitly controls SFT-induced movement in dominant pretrained singular subspaces, aiming to obtain rapid task adaptation without causing the representation drift that later RL analyses suggest is avoidable.

Several LLM fine-tuning methods are close to this goal. Proximal SFT constrains policy drift with PPO-style updates \citep{zhu2025proximalsupervisedfinetuning}; Dynamic SFT and importance-weighted SFT rescale token or example losses \citep{wu2025generalizationsftreinforcementlearning,qin2025supervisedfinetuningcurated}; and subspace-constrained methods project gradients away from protected directions, including principal-subspace preservation and gradient low-rank projection \citep{franke2024preserving,nayak2025sculptingsubspacesconstrainedfinetuning,wang2025continualgradientlowrankprojection}. Orthogonality-based PEFT methods such as O-LoRA similarly reduce interference by constraining the adapter subspace \citep{wang2023orthogonalsubspacelearninglanguage}. Spectral PEFT methods such as PiSSA and MiLoRA use singular directions for efficient adaptation \citep{meng2025pissaprincipalsingularvalues,wang2025miloraharnessingminorsingular}, but their goal is parameter efficiency rather than explicitly mitigating OOD forgetting. In contrast, RPSFT keeps the full fine-tuning parameterization and adds a soft projected-block penalty in the pretrained singular-vector basis.

\subsection{Comparison with Regularization Methods}
\label{app:regularizer-comparison}

Let $\theta_0$ denote the pretrained parameters, and let $\mathbf{W}^0_\ell=\mathbf{U}^0_\ell\mathbf{\Sigma}^0_\ell(\mathbf{V}^0_\ell)^\top$ be the SVD of layer $\ell$. Table~\ref{tab:method-comparison} compares methods whose update rules or objectives are explicit enough to make a direct theoretical comparison. We omit papers that are primarily diagnostic or empirical motivation when the paper does not define a comparable regularizer.

\begin{table}[t]
\centering
\scriptsize
\caption{Comparison with related regularization and subspace-preservation methods.}
\label{tab:method-comparison}
\resizebox{\textwidth}{!}{%
\begin{tabular}{p{0.15\textwidth}p{0.36\textwidth}p{0.20\textwidth}p{0.16\textwidth}p{0.08\textwidth}}
\toprule
Method & Objective or update & Extra compute & Extra memory & Ref. \\
\midrule
EWC &
$\mathcal{L}_{\rm task}
+\frac{\lambda}{2}\sum_i F_i(\theta_i-\theta_{0,i})^2$ &
$O(TP)$ Fisher estimation; $O(P)$ per step &
$O(P)$ &
\citep{Kirkpatrick_2017} \\
\midrule
L2 anchor &
$\mathcal{L}_{\rm task}+\lambda\|\theta-\theta_0\|_2^2$ &
$O(P)$ per step &
$O(P)$ &
\citep{kumar2024maintainingplasticitycontinuallearning} \\
\midrule
Weight clipping &
$\theta_i\leftarrow\operatorname{clip}(\theta_i,-\epsilon,\epsilon)$ &
$O(P)$ per step &
$O(1)$ &
\citep{elsayed2024weightclippingdeepcontinual} \\
\midrule
Hard subspace projection &
$g_\ell\leftarrow(\mathbf{I}-\mathbf{P}_\ell)g_\ell$ &
$O\!\left(\sum_\ell m_\ell n_\ell k_\ell\right)$ per step &
$O\!\left(\sum_\ell (m_\ell+n_\ell)k_\ell\right)$ &
\citep{franke2024preserving,nayak2025sculptingsubspacesconstrainedfinetuning} \\
\midrule
RPSFT (ours) &
$\mathcal{L}_{\rm task}
+\lambda\sum_\ell\|S_\ell-S_\ell^0\|_F^2$,

$S_\ell=(\mathbf{U}^{(k)}_{0,\ell})^\top
\mathbf{W}_\ell\mathbf{V}^{(k)}_{0,\ell}$ &
$O\!\left(\sum_\ell m_\ell n_\ell k_\ell / s\right)$ &
$O\!\left(\sum_\ell ((m_\ell+n_\ell)k_\ell+k_\ell^2)\right)$ &
-- \\
\bottomrule
\end{tabular}%
}
\end{table}

\section{Theory: Properties of RPSFT Regularization}
\label{app:theory}

\subsection{Setup and Notation}
Consider a weight matrix $\mathbf{W}\in\mathbb{R}^{m\times n}$ with pretrained initialization $\mathbf{W}_0$.
Let the (thin) SVD of $\mathbf{W}_0$ be
\begin{equation}
\mathbf{W}_0 = \mathbf{U}_0 \mathbf{\Sigma}_0 \mathbf{V}_0^\top,
\end{equation}
where $\mathbf{U}_0\in\mathbb{R}^{m\times r}$, $\mathbf{V}_0\in\mathbb{R}^{n\times r}$ have orthonormal columns, $\mathbf{\Sigma}_0\in\mathbb{R}^{r\times r}$ is diagonal, and $r=\mathrm{rank}(\mathbf{W}_0)$.
Let $\mathbf{U}_k\in\mathbb{R}^{m\times k}$ and $\mathbf{V}_k\in\mathbb{R}^{n\times k}$ denote the top-$k$ left and right singular vectors of $\mathbf{W}_0$.
Define the linear operator
\begin{equation}
\mathcal{P}_k(\mathbf{W}) \triangleq \mathbf{U}_k^\top \mathbf{W} \mathbf{V}_k \in \mathbb{R}^{k\times k}.
\end{equation}
RPSFT adds the penalty
\begin{equation}
\mathcal{R}_k(\mathbf{W}) \triangleq \left\| \mathcal{P}_k(\mathbf{W}) - \mathcal{P}_k(\mathbf{W}_0)\right\|_F^2
= \left\| \mathbf{U}_k^\top (\mathbf{W}-\mathbf{W}_0) \mathbf{V}_k \right\|_F^2.
\end{equation}
Given a task loss $f(\mathbf{W})$ (e.g., the SFT objective), the RPSFT objective is
\begin{equation}
F_\lambda(\mathbf{W}) \triangleq f(\mathbf{W}) + \lambda \, \mathcal{R}_k(\mathbf{W}), \qquad \lambda \ge 0.
\end{equation}

\subsection{Boundary Cases and Limiting Behavior}
We first record basic limiting cases that clarify how $\lambda$ and $k$ interpolate between standard SFT and constrained training.

\paragraph{Case $\lambda \to 0$.}
As $\lambda\to 0$, $F_\lambda(\mathbf{W}) \to f(\mathbf{W})$ and RPSFT reduces to the original SFT objective.

\paragraph{Case $\lambda \to \infty$.}
For $\lambda\to\infty$, minimizers of $F_\lambda$ converge to solutions of the
constrained problem:
\begin{equation}
\min_{\mathbf{W}} \ f(\mathbf{W}) \quad \text{s.t.} \quad \mathcal{P}_k(\mathbf{W})=\mathcal{P}_k(\mathbf{W}_0)
\ \Longleftrightarrow\ 
\mathbf{U}_k^\top (\mathbf{W}-\mathbf{W}_0) \mathbf{V}_k = 0.
\label{eq:constrained_limit}
\end{equation}
which fixes the component of $\mathbf{W}$ acting between the pretrained top-$k$ left and
right singular subspaces while leaving all other components unconstrained.


\paragraph{Case $k=0$.}
$\mathcal{R}_0(\mathbf{W})\equiv 0$, hence RPSFT is exactly standard SFT for any $\lambda$.

\paragraph{Case $k=\min(m,n)$ (full rank).}
Let $k=\min(m,n)$ and let $\mathbf{U}\in\mathbb{R}^{m\times m}$, $\mathbf{V}\in\mathbb{R}^{n\times n}$ be any orthogonal completions of the singular vector bases of $\mathbf{W}_0$.
Then, by orthogonal invariance of the Frobenius norm,
\begin{equation}
\mathcal{R}_k(\mathbf{W}) = \|\mathbf{U}^\top (\mathbf{W}-\mathbf{W}_0) \mathbf{V}\|_F^2 = \|\mathbf{W}-\mathbf{W}_0\|_F^2,
\end{equation}
so RPSFT becomes standard $\ell_2$ anchoring around the pretrained weights.

\subsection{Block Decomposition Interpretation}
Let $\mathbf{U}=[\mathbf{U}_k\ \ \mathbf{U}_\perp]\in\mathbb{R}^{m\times m}$ and $\mathbf{V}=[\mathbf{V}_k\ \ \mathbf{V}_\perp]\in\mathbb{R}^{n\times n}$ be orthogonal matrices extending $\mathbf{U}_k,\mathbf{V}_k$.
Define the rotated coordinates
\begin{equation}
\widetilde{\mathbf{W}} \triangleq \mathbf{U}^\top \mathbf{W} \mathbf{V}
=
\begin{pmatrix}
A & B\\
C & D
\end{pmatrix},
\end{equation}
where $A\in\mathbb{R}^{k\times k}$.
Similarly $\widetilde{\mathbf{W}}_0 = \mathbf{U}^\top \mathbf{W}_0 \mathbf{V}$ has top-left block $A_0$.
Then
\begin{equation}
\mathcal{R}_k(\mathbf{W})=\|A-A_0\|_F^2.
\label{eq:block_penalty}
\end{equation}
Thus, RPSFT penalizes drift of the top-left $k\times k$ block in the pretrained singular-vector basis, while allowing updates through the other blocks $B,C,D$.

\subsection{Optimality Condition and \texorpdfstring{$O(1/\lambda)$}{O(1/lambda)} Drift Control}
We give a simple bound showing how $\lambda$ controls the protected-block deviation.

\begin{proposition}[Stationary Condition]
\label{prop:stationary}
Assume $f$ is differentiable. Any stationary point $\mathbf{W}_\lambda$ of $F_\lambda$ satisfies
\begin{equation}
\nabla f(\mathbf{W}_\lambda) + 2\lambda \, \mathbf{U}_k \big(\mathbf{U}_k^\top (\mathbf{W}_\lambda - \mathbf{W}_0) \mathbf{V}_k\big) \mathbf{V}_k^\top = 0.
\label{eq:stationary_eq}
\end{equation}
\end{proposition}

\begin{proof}
Using $\mathcal{R}_k(\mathbf{W})=\|\mathbf{U}_k^\top (\mathbf{W}-\mathbf{W}_0)\mathbf{V}_k\|_F^2$ and standard matrix calculus,
$\nabla_{\mathbf{W}} \mathcal{R}_k(\mathbf{W}) = 2 \mathbf{U}_k (\mathbf{U}_k^\top (\mathbf{W}-\mathbf{W}_0)\mathbf{V}_k) \mathbf{V}_k^\top$.
Setting $\nabla F_\lambda(\mathbf{W})=0$ yields Eq.~\eqref{eq:stationary_eq}.
\end{proof}

\begin{proposition}[$1/\lambda$ Control of Protected Drift]
\label{prop:lambda_control}
Let $\mathbf{W}_\lambda$ be any stationary point of $F_\lambda$. Then
\begin{equation}
\left\|\mathbf{U}_k^\top (\mathbf{W}_\lambda-\mathbf{W}_0) \mathbf{V}_k \right\|_F
\le
\frac{\|\nabla f(\mathbf{W}_\lambda)\|_F}{2\lambda}.
\label{eq:lambda_bound}
\end{equation}
\end{proposition}

\begin{proof}
Left-multiply Eq.~\eqref{eq:stationary_eq} by $\mathbf{U}_k^\top$ and right-multiply by $\mathbf{V}_k$:
\begin{equation}
\mathbf{U}_k^\top \nabla f(\mathbf{W}_\lambda) \mathbf{V}_k + 2\lambda \, \mathbf{U}_k^\top (\mathbf{W}_\lambda-\mathbf{W}_0) \mathbf{V}_k = 0.
\end{equation}
Taking Frobenius norms and using $\|\mathbf{U}_k^\top X \mathbf{V}_k\|_F \le \|X\|_F$ gives Eq.~\eqref{eq:lambda_bound}.
\end{proof}
\paragraph{Implication.}
Eq.~\eqref{eq:lambda_bound} shows that 1) at any stationary point, the protected-block deviation is bounded by the task-gradient magnitude scaled by $O(1/\lambda)$
; 2) as $\lambda$ increases, the deviation of the
protected block $\mathbf{U}_k^\top(\mathbf{W}_\lambda-\mathbf{W}_0)\mathbf{V}_k$ shrinks at rate $O(1/\lambda)$, formalizing a smooth interpolation between
unconstrained SFT ($\lambda=0$) and the constrained regime obtained as
$\lambda\to\infty$.

\subsection{Rank Selection: Boundary and Threshold Rule}
\label{app:rank_selection}

We now formalize why an excessively large protected rank can hurt rapid adaptation.
We work in the forgetting-dominated regime discussed above, where the pretrained
model is locally close to an OOD stationary point, so the first-order OOD term
is negligible.

Let $\Delta \coloneqq \mathbf{W}-\mathbf{W}_0$ and
\begin{equation}
\widetilde{\Delta} \coloneqq \mathbf{U}^\top \Delta \mathbf{V},
\end{equation}
and index the entries of \(\widetilde{\Delta}\) by \(s=(i,j)\), with coordinate
value \(\delta_s\). For each rank \(k\), define the protected coordinate set
\begin{equation}
S_k \coloneqq \{(i,j): 1\le i\le k,\; 1\le j\le k\}.
\end{equation}
Thus \(s\in S_k\) means that \(\delta_s\) lies in the protected top-left
\(k\times k\) block of the pretrained singular-vector basis.

We use the local quadratic model
\begin{equation}
f_{\mathrm{id}}(\mathbf{W}_0+\Delta)-f_{\mathrm{id}}(\mathbf{W}_0)
\approx
-\sum_s g_s \delta_s + \frac12 \sum_s h_s \delta_s^2,
\qquad h_s>0,
\end{equation}
where \(g_s\) is the in-domain learning drive along coordinate \(s\), and \(h_s\)
is the corresponding local curvature. Under RPSFT, the local objective becomes
\begin{equation}
-\sum_s g_s \delta_s
+
\frac12 \sum_s h_s \delta_s^2
+
\lambda \sum_{s\in S_k} \delta_s^2.
\end{equation}

For OOD forgetting, we use the quadratic proxy
\begin{equation}
f_{\mathrm{ood}}(\mathbf{W}_0+\Delta)-f_{\mathrm{ood}}(\mathbf{W}_0)
\approx
\frac12 \sum_s c_s \delta_s^2,
\qquad c_s\ge 0,
\end{equation}
where \(c_s\) measures how sensitive OOD performance is to movement along
coordinate \(s\): larger \(c_s\) means larger OOD degradation.

\begin{proposition}[Optimal local step]
\label{prop:optimal_local_step}
For each coordinate \(s\), the minimizer of the local regularized objective is
\begin{equation}
\delta_s^\star(k)=
\begin{cases}
\dfrac{g_s}{h_s+2\lambda}, & s\in S_k,\\[1.2ex]
\dfrac{g_s}{h_s}, & s\notin S_k.
\end{cases}
\end{equation}
The resulting OOD increase is
\begin{equation}
F_{\mathrm{ood}}(k)
=
\frac12 \sum_{s\in S_k} \frac{c_s g_s^2}{(h_s+2\lambda)^2}
+
\frac12 \sum_{s\notin S_k} \frac{c_s g_s^2}{h_s^2}.
\end{equation}
The resulting ID gain, measured on the unregularized local ID proxy after taking
the regularized step, is
\begin{equation}
G_{\mathrm{id}}(k)
=
\sum_{s\in S_k}
\left(
\frac{g_s^2}{h_s+2\lambda}
-
\frac12 h_s \frac{g_s^2}{(h_s+2\lambda)^2}
\right)
+
\frac12 \sum_{s\notin S_k} \frac{g_s^2}{h_s}.
\end{equation}
\end{proposition}

\begin{proof}
The objective is separable across coordinates, so the optimizer is obtained
coordinatewise. Substituting the resulting \(\delta_s^\star(k)\) into the ID and
OOD quadratic proxies gives the stated formulas.
\end{proof}

\begin{proposition}[Existence of an upper rank boundary]
\label{prop:rank_boundary}
Assume there exists \(q\) such that
\begin{equation}
c_s=0
\qquad \text{for all } s\notin S_q.
\end{equation}
That is, all OOD-sensitive coordinates are already contained in the protected
top-\(q\) block. Define the scalarized utility
\begin{equation}
\Phi(k)\coloneqq G_{\mathrm{id}}(k)-\beta F_{\mathrm{ood}}(k),
\qquad \beta>0,
\end{equation}
where \(\beta\) controls how strongly OOD forgetting is penalized relative to
ID gain. Then every maximizer \(k^\star\) of \(\Phi\) satisfies
\begin{equation}
k^\star \le q.
\end{equation}
\end{proposition}

\begin{proof}
For \(k\ge q\), enlarging the protected set no longer changes \(F_{\mathrm{ood}}(k)\),
because every coordinate with \(c_s>0\) is already protected at rank \(q\). Hence
\begin{equation}
F_{\mathrm{ood}}(k)=F_{\mathrm{ood}}(q), \qquad k\ge q.
\end{equation}
On the other hand, protecting any additional coordinate weakly decreases its
contribution to \(G_{\mathrm{id}}(k)\), with strict decrease whenever
\(\lambda>0\) and \(g_s\neq 0\). Therefore
\begin{equation}
G_{\mathrm{id}}(k)\le G_{\mathrm{id}}(q), \qquad k\ge q,
\end{equation}
and thus
\begin{equation}
\Phi(k)\le \Phi(q), \qquad k\ge q.
\end{equation}
So no maximizer can lie above \(q\).
\end{proof}

\begin{corollary}[Threshold rule]
\label{cor:threshold_rule}
Consider a coordinate \(s\) that is currently unprotected. Define the ID cost of
protecting this coordinate by
\begin{equation}
\Delta_{\mathrm{ID},s}
\coloneqq
\frac12 \frac{g_s^2}{h_s}
-
\left(
\frac{g_s^2}{h_s+2\lambda}
-
\frac12 h_s \frac{g_s^2}{(h_s+2\lambda)^2}
\right)
=
\frac{2\lambda^2 g_s^2}{h_s (h_s+2\lambda)^2},
\end{equation}
and define the OOD gain of protecting this coordinate by
\begin{equation}
\Delta_{\mathrm{OOD},s}
\coloneqq
\frac12 c_s \frac{g_s^2}{h_s^2}
-
\frac12 c_s \frac{g_s^2}{(h_s+2\lambda)^2}
=
\frac{2 c_s \lambda (h_s+\lambda) g_s^2}
{h_s^2 (h_s+2\lambda)^2}.
\end{equation}
Then protecting coordinate \(s\) improves \(\Phi\) iff
\begin{equation}
\beta\,\Delta_{\mathrm{OOD},s}>\Delta_{\mathrm{ID},s},
\end{equation}
which is equivalent to
\begin{equation}
c_s>\frac{\lambda h_s}{\beta (h_s+\lambda)}.
\end{equation}
\end{corollary}

\begin{proof}
This follows by comparing the protected and unprotected one-coordinate
contributions in Proposition~\ref{prop:optimal_local_step} and simplifying.
\end{proof}

\paragraph{Interpretation.}
The upper-boundary proposition shows that once the protected rank exceeds the
support of OOD-sensitive coordinates, increasing \(k\) can no longer improve
robustness, but can still worsen the trade-off between mitigating forgetting and rapid adaptation by
over-suppressing directions that are useful for in-domain adaptation. The
threshold rule gives the per-coordinate version of the same idea: a direction
should be protected only if its OOD sensitivity \(c_s\) is large enough to
outweigh the in-domain gain lost by shrinking that direction. In this view,
\(c_s\) can be interpreted as the OOD sensitivity of coordinate \(s\) in the
pretrained singular basis, where larger \(c_s\) means that updating this
direction causes greater OOD degradation. If OOD sensitivity decays with
singular index, this naturally induces a finite rank boundary, so the optimal
rank \(k\) should be large enough to cover the high-\(c_s\) directions, but not
so large that it also protects many low-\(c_s\) directions. This provides a
simple explanation for the rank sweep in Appendix~\ref{app:robustness-extra}:
moving from \(k=0\) to a moderate rank improves robustness because it protects
the most OOD-sensitive directions, whereas overly large ranks behave more like
global anchoring and degrade the trade-off between mitigating forgetting and rapid adaptation
by over-constraining directions that are not strongly OOD-sensitive.

\paragraph{Practical guidance.}
Before SFT, we choose \(k\) using the base model and the same Fisher-projected gradient-energy diagnostic shown in Figure~\ref{fig:svd-fisher-motivation}. Concretely, we compute per-sample gradients on a small batch from the OOD data, project them into the pretrained SVD basis of an early attention matrix such as layer-1 \texttt{q\_proj}, and sweep candidate ranks \(r\). We then inspect the curve \(x(r)=100r^2/R^2\) versus \(y(r)=\mathrm{tr}(\mathbf{P}_{\mathrm{svd},r}\mathbf{F})/\mathrm{tr}(\mathbf{F})\). A practical default is the smallest \(r\) whose strict top-\(r\times r\) block captures about \(20\%\) of the gradient energy in the early attention layer: this protects a meaningful fraction of loss-sensitive directions while keeping the protected block small enough for adaptation. In our experiments, \(r=768\) corresponds to a roughly \(5\%\) strict block and already captures about \(20\%\) of the gradient energy, so we use \(k=768\) as the default protected rank.

\subsection{Gradient-Flow View: Exponentially Damped Task-Induced Drift}
\label{app:grad_flow}

To gain insight into the optimization dynamics induced by RPSFT, we consider the
continuous-time gradient-flow limit of the regularized objective
\[
F_\lambda(\mathbf{W})=f(\mathbf{W})+\lambda\|\mathbf{U}_k^\top(\mathbf{W}-\mathbf{W}_0)\mathbf{V}_k\|_F^2.
\]
Assume that $f$ is differentiable with locally Lipschitz gradient, so that
the gradient flow is well-defined. The resulting dynamics are
\begin{equation}
\frac{d\mathbf{W}(t)}{dt}
=
-\nabla f(\mathbf{W}(t))
-
2\lambda\,\mathbf{U}_k\big(\mathbf{U}_k^\top(\mathbf{W}(t)-\mathbf{W}_0)\mathbf{V}_k\big)\mathbf{V}_k^\top.
\label{eq:grad_flow}
\end{equation}

\paragraph{Protected coordinates.}
Define the protected coordinates
\[
A(t)\triangleq \mathbf{U}_k^\top(\mathbf{W}(t)-\mathbf{W}_0)\mathbf{V}_k,
\]
which measure the deviation of the weights from the pretrained model within the
top-$k$ singular subspace of $\mathbf{W}_0$. Multiplying Eq.~\eqref{eq:grad_flow} on the left
by $\mathbf{U}_k^\top$ and on the right by $\mathbf{V}_k$ yields
\begin{equation}
\frac{dA(t)}{dt}
=
-\,\mathbf{U}_k^\top \nabla f(\mathbf{W}(t)) \mathbf{V}_k
-
2\lambda A(t).
\label{eq:A_dynamics} 
\end{equation}

\paragraph{Closed-form solution.}
Eq.~\eqref{eq:A_dynamics} is a linear ODE with a time-dependent forcing term.
Define
\[
G(t)\triangleq \mathbf{U}_k^\top \nabla f(\mathbf{W}(t)) \mathbf{V}_k,
\]
so that $\dot A(t) + 2\lambda A(t) = -G(t)$.
Solving this equation yields
\begin{equation}
A(t)
=
e^{-2\lambda t}A(0)
-
\int_0^t e^{-2\lambda (t-s)} G(s)\,ds,
\label{eq:A_solution}
\end{equation}
or equivalently,
\[
A(t)
=
e^{-2\lambda t}A(0)
-
\int_0^t e^{-2\lambda (t-s)}
\big(\mathbf{U}_k^\top \nabla f(\mathbf{W}(s)) \mathbf{V}_k\big)\,ds.
\]

The first term represents an exponentially decaying initial deviation from the
pretrained weights, while the second term is an exponentially weighted
accumulation of task-induced gradients in the protected subspace.

\paragraph{Interpretation.}
Eq.~\eqref{eq:A_solution} shows that RPSFT acts as an exponential damping
mechanism on task-induced drift within the pretrained top-$k$ singular subspace.
Even when fine-tuning is initialized from the pretrained model (so that $A(0)=0$),
task gradients generally introduce nonzero drift, but their influence is
continuously decayed with rate $2\lambda$. Consequently, RPSFT behaves
as a temporal low-pass filter on gradients in the protected subspace, stabilizing
dominant pretrained representations while still allowing task-relevant adaptation.


\section{Additional Results}
\label{app:robustness-extra}

We study the sensitivity of RPSFT to the protected rank $k$ on Llama-3.1-8B-Instruct as the empirical counterpart to the rank-boundary analysis in Appendix~\ref{app:rank_selection}.
Here, SFT corresponds to $k=0$, and our chosen configuration is $k=768$.
The LoRA row is a vanilla LoRA baseline with adapter rank $r=32$; no RPSFT penalty is applied to it. The full-rank setting $k=4096$ reduces RPSFT to weight-space anchoring around the pretrained model, so we label it as L2 Init in Table~\ref{tab:robust-llama-merged}. This matches the regenerative regularization view of \citet{kumar2024maintainingplasticitycontinuallearning}. The main insight is that Llama benefits from an intermediate protected rank: $k=256$--$768$ preserves adaptation, but larger $k$ progressively weakens both in-domain adaptation and OOD generalization. This matches the positive Llama OOD first-order signal in Figure~\ref{fig:first-order-signal}: when the update direction can still help OOD performance, over-protecting too many directions removes useful transfer rather than only preventing forgetting. The same pattern agrees with the rank-boundary analysis in Appendix~\ref{app:rank_selection}: once the protected block already covers the most OOD-sensitive directions, expanding it mainly constrains useful task updates.

\begin{table}[ht]
\centering
\scriptsize
\setlength{\tabcolsep}{4pt}
\renewcommand{\arraystretch}{0.95}
\caption{Rank robustness on Llama-3.1-8B-Instruct across ID Avg@k, ID Pass@k, and OOD Avg@1 (\%). The LoRA row is vanilla LoRA with adapter rank $r=32$; RPSFT rows vary protected rank $k$. Base cells are left unmarked, and bold marks the best non-base value in each domain block and column.}
\label{tab:robust-llama-merged}
\resizebox{\textwidth}{!}{%
\begin{tabular}{ll S S S S S S S}
\toprule
Domain & Setting & {AIME24} & {AIME25} & {AMC23} & {MATH-500} & {Minerva} & {Olympiad} & {\MeanAvg} \\
\midrule
\multirow{8}{*}{\rotatebox[origin=c]{90}{\textbf{ID Avg@k}}}
& Base & 3.13 & 0.83 & 21.25 & 43.65 & 22.33 & 14.63 & 17.64 \\
& SFT & 5.00 & 5.00 & 32.19 & 59.45 & 26.88 & 25.59 & 25.69 \\
& Vanilla LoRA ($r$=32) & 3.13 & 2.50 & 23.75 & 47.80 & 21.60 & 14.63 & 18.90 \\
& RPSFT-256 & \bfseries 5.42 & 6.46 & 29.84 & 59.05 & \bfseries 28.63 & 26.70 & 26.02 \\
& RPSFT-512 & 4.58 & \bfseries 7.50 & \bfseries 34.06 & 58.20 & 27.34 & \bfseries 27.34 & \bfseries 26.50 \\
& RPSFT-768 & 4.38 & 6.67 & 31.56 & \bfseries 60.45 & 27.48 & 25.52 & 26.01 \\
& RPSFT-2048 & 4.17 & 5.21 & 28.44 & 56.70 & 25.64 & 23.63 & 23.97 \\
& L2 Init ($k$=4096) & 4.38 & 1.04 & 21.88 & 44.60 & 21.51 & 16.00 & 18.24 \\
\midrule
\multirow{8}{*}{\rotatebox[origin=c]{90}{\textbf{ID Pass@k}}}
& Base & 20.00 & 13.33 & 72.50 & 65.00 & 44.85 & 29.78 & 40.91 \\
& SFT & 20.00 & 13.33 & 77.50 & \bfseries 78.60 & 47.06 & 41.33 & 46.30 \\
& Vanilla LoRA ($r$=32) & 23.33 & 20.00 & 80.00 & 67.40 & 48.53 & 29.78 & 44.84 \\
& RPSFT-256 & 20.00 & \bfseries 26.66 & 75.00 & 75.40 & \bfseries 50.74 & 42.67 & 48.41 \\
& RPSFT-512 & 20.00 & 23.33 & \bfseries 87.50 & 76.60 & 49.26 & \bfseries 49.26 & \bfseries 50.99 \\
& RPSFT-768 & 20.00 & 20.00 & 85.00 & 77.20 & \bfseries 50.74 & 42.22 & 49.19 \\
& RPSFT-2048 & 16.67 & 20.00 & 72.50 & 76.60 & 47.79 & 39.85 & 45.57 \\
& L2 Init ($k$=4096) & \bfseries 26.66 & 13.33 & 72.50 & 67.60 & 47.79 & 31.85 & 43.29 \\
\midrule
 &  & {GPQA} & {IFEval} & {MMLU-Pro} & {SuperGPQA} & {Safety} & {TruthfulQA} & {\MeanAvg} \\
\midrule
\multirow{8}{*}{\rotatebox[origin=c]{90}{\textbf{OOD Avg@1}}}
& Base & 25.89 & 32.53 & 47.14 & 18.77 & 61.10 & 58.04 & 40.58 \\
& SFT & 31.70 & 31.42 & 52.86 & 19.75 & 57.90 & 55.70 & 41.55 \\
& Vanilla LoRA ($r$=32) & 28.57 & 31.79 & 48.57 & 17.78 & 63.80 & 56.14 & 41.11 \\
& RPSFT-256 & 33.93 & 30.87 & \bfseries 58.57 & 21.98 & 63.40 & 56.29 & 44.17 \\
& RPSFT-512 & 32.81 & 30.87 & 52.86 & 20.74 & \bfseries 68.70 & 55.99 & 43.66 \\
& RPSFT-768 & \bfseries 34.60 & \bfseries 33.46 & \bfseries 58.57 & \bfseries 24.20 & 61.70 & 58.63 & \bfseries 45.19 \\
& RPSFT-2048 & 33.26 & 31.05 & 54.29 & 23.95 & 58.80 & 56.58 & 42.99 \\
& L2 Init ($k$=4096) & 26.34 & 32.53 & 47.14 & 16.79 & 62.70 & \bfseries 61.11 & 41.10 \\
\bottomrule
\end{tabular}}
\renewcommand{\arraystretch}{1}
\end{table}

\FloatBarrier

\paragraph{Qwen2.5-7B robustness sweep.}
Table~\ref{tab:robust-qwen-merged} reports the same protected-rank sweep on Qwen2.5-7B-Instruct. Here, SFT corresponds to $k=0$, the LoRA row is vanilla LoRA with adapter rank $r=32$, the chosen RPSFT configuration is again $k=768$, and $k=3584$ is equivalent to L2 Init. Qwen shows the other side of the same trade-off: increasing $k$ reduces in-domain adaptation but steadily mitigates OOD forgetting, with the full-rank L2 Init row closest to the base model on OOD average. This matches Figure~\ref{fig:first-order-signal}, where the Qwen OOD first-order signal is weaker and changes sign more often, so unconstrained task adaptation is less reliably aligned with OOD improvement. The trend is also consistent with Appendix~\ref{app:rank_selection}: protecting more OOD-sensitive directions improves retention, but beyond a moderate rank the extra protection starts to suppress task learning. Vanilla LoRA remains weaker, especially on Qwen OOD retention, indicating that parameter-efficient adaptation alone does not provide the same forgetting control as the RPSFT projected-block penalty.

\begin{table}[ht]
\centering
\scriptsize
\setlength{\tabcolsep}{4pt}
\renewcommand{\arraystretch}{0.95}
\caption{Rank robustness on Qwen2.5-7B-Instruct across ID Avg@k, ID Pass@k, and OOD Avg@1 (\%). The LoRA row is vanilla LoRA with adapter rank $r=32$; RPSFT rows vary protected rank $k$. Base cells are left unmarked, and bold marks the best non-base value in each domain block and column.}
\label{tab:robust-qwen-merged}
\resizebox{\textwidth}{!}{%
\begin{tabular}{ll S S S S S S S}
\toprule
Domain & Setting & {AIME24} & {AIME25} & {AMC23} & {MATH-500} & {Minerva} & {Olympiad} & {\MeanAvg} \\
\midrule
\multirow{8}{*}{\rotatebox[origin=c]{90}{\textbf{ID Avg@k}}}
& Base & 12.08 & 7.70 & 52.03 & 72.85 & 37.09 & 37.41 & 36.53 \\
& SFT & 15.21 & \bfseries 19.38 & 50.47 & 72.85 & 31.52 & 37.41 & 37.81 \\
& Vanilla LoRA ($r$=32) & 11.25 & 13.54 & 46.41 & 71.80 & 34.24 & 33.96 & 35.20 \\
& RPSFT-256 & \bfseries 16.25 & 17.50 & 54.38 & 73.25 & 32.81 & 37.56 & 38.63 \\
& RPSFT-512 & 15.00 & \bfseries 19.38 & 52.97 & 73.30 & 33.23 & \bfseries 39.37 & 38.88 \\
& RPSFT-768 & 14.38 & 18.75 & \bfseries 55.31 & \bfseries 74.95 & 33.82 & 38.00 & \bfseries 39.20 \\
& RPSFT-2048 & 14.17 & 17.08 & 51.72 & 74.70 & 33.50 & 39.15 & 38.39 \\
& L2 Init ($k$=3584) & 11.88 & 9.38 & 51.88 & 74.70 & \bfseries 37.55 & 37.37 & 37.13 \\
\midrule
\multirow{8}{*}{\rotatebox[origin=c]{90}{\textbf{ID Pass@k}}}
& Base & 36.67 & 36.67 & 90.00 & 84.60 & 55.17 & 54.96 & 59.68 \\
& SFT & \bfseries 50.00 & 36.67 & 87.50 & \bfseries 87.60 & \bfseries 58.09 & 54.96 & 62.47 \\
& Vanilla LoRA ($r$=32) & 40.00 & 30.00 & 77.50 & 84.60 & 55.88 & 52.15 & 56.69 \\
& RPSFT-256 & 46.67 & 36.67 & \bfseries 90.00 & 86.60 & 54.04 & 54.37 & 61.39 \\
& RPSFT-512 & 43.33 & 40.00 & 85.00 & 86.60 & 58.46 & \bfseries 57.04 & 61.74 \\
& RPSFT-768 & 43.33 & \bfseries 43.33 & \bfseries 90.00 & 87.40 & 56.62 & 56.15 & \bfseries 62.81 \\
& RPSFT-2048 & 43.33 & 36.67 & 85.00 & 86.60 & 57.35 & 56.29 & 60.87 \\
& L2 Init ($k$=3584) & 30.00 & 40.00 & \bfseries 90.00 & 84.80 & \bfseries 58.09 & 52.15 & 59.17 \\
\midrule
 &  & {GPQA} & {IFEval} & {MMLU-Pro} & {SuperGPQA} & {Safety} & {TruthfulQA} & {\MeanAvg} \\
\midrule
\multirow{8}{*}{\rotatebox[origin=c]{90}{\textbf{OOD Avg@1}}}
& Base & 33.93 & 61.37 & 67.14 & 27.65 & 74.50 & 66.96 & 55.26 \\
& SFT & 32.80 & 52.13 & 65.71 & 26.42 & 70.10 & 59.50 & 51.11 \\
& Vanilla LoRA ($r$=32) & 32.59 & 45.84 & 58.57 & 23.21 & 67.60 & 64.04 & 48.64 \\
& RPSFT-256 & 34.82 & 56.38 & 67.14 & 29.38 & 70.00 & 59.80 & 52.92 \\
& RPSFT-512 & \bfseries 36.38 & 56.01 & 65.71 & \bfseries 29.63 & 68.80 & 61.26 & 52.97 \\
& RPSFT-768 & 34.60 & 54.34 & 70.00 & 26.42 & \bfseries 71.10 & 62.43 & 53.15 \\
& RPSFT-2048 & 33.71 & 57.67 & \bfseries 74.29 & 27.41 & 70.90 & 63.45 & 54.57 \\
& L2 Init ($k$=3584) & 35.49 & \bfseries 63.03 & 68.57 & 27.41 & 70.10 & \bfseries 66.67 & \bfseries 55.21 \\
\bottomrule
\end{tabular}}
\renewcommand{\arraystretch}{1}
\end{table}

\FloatBarrier

\paragraph{OOD Pass@1 results.}
Tables~\ref{tab:sft-ood-pass1-app} and~\ref{tab:rl-ood-pass1-app} report the earlier one-sample OOD evaluation. The main results use the four-sample OOD Avg@k and Pass@k evaluation in Tables~\ref{tab:sft-merged} and~\ref{tab:rl-ood}.

\begin{table}[ht]
\centering
\scriptsize
\setlength{\tabcolsep}{3pt}
\renewcommand{\arraystretch}{0.95}
\caption{SFT-stage OOD Pass@1 results on the six OOD benchmarks. Bold marks the strongest tuned result within each model block.}
\label{tab:sft-ood-pass1-app}
\resizebox{\textwidth}{!}{%
\begin{tabular}{l S S S S S S S S S S S S S S S}
\toprule
Benchmark & \multicolumn{5}{c}{\textbf{Llama-3.1-8B-Instruct}} & \multicolumn{5}{c}{\textbf{Qwen2.5-7B-Instruct}} & \multicolumn{5}{c}{\textbf{Qwen2.5-3B-Instruct}} \\
\cmidrule(lr){2-6}\cmidrule(lr){7-11}\cmidrule(lr){12-16}
& {Base} & {SFT} & {IW} & {DFT} & {RPSFT} & {Base} & {SFT} & {IW} & {DFT} & {RPSFT} & {Base} & {SFT} & {IW} & {DFT} & {RPSFT} \\
\midrule
GPQA & 25.89 & 31.70 & \bfseries 37.72 & 27.68 & 34.60 & 33.93 & 32.80 & 32.37 & 30.36 & \bfseries 34.60 & 27.01 & 30.58 & 27.01 & 26.56 & \bfseries 31.70 \\
IFEval & 32.53 & 31.42 & 32.53 & 31.98 & \bfseries 33.46 & 61.37 & 52.13 & 54.34 & \bfseries 59.70 & 54.34 & 54.16 & 47.32 & 46.03 & 51.20 & \bfseries 52.13 \\
MMLU-Pro & 47.14 & 52.86 & 51.43 & 52.86 & \bfseries 58.57 & 67.14 & 65.71 & 58.57 & 62.86 & \bfseries 70.00 & 52.86 & \bfseries 45.71 & \bfseries 45.71 & 34.29 & 44.28 \\
SuperGPQA & 18.77 & 19.75 & 21.48 & 17.53 & \bfseries 24.20 & 27.65 & 26.42 & \bfseries 26.91 & 25.19 & 26.42 & 18.27 & 19.26 & 19.51 & 16.30 & \bfseries 21.48 \\
Safety & 61.10 & 57.90 & \bfseries 62.20 & 60.80 & 61.70 & 74.50 & 70.10 & 70.10 & \bfseries 71.90 & 71.10 & 65.90 & 59.60 & 60.00 & \bfseries 65.80 & 62.20 \\
TruthfulQA & 58.04 & 55.70 & 49.56 & \bfseries 60.82 & 58.63 & 66.96 & 59.50 & 60.53 & \bfseries 65.06 & 62.43 & 59.94 & 49.42 & 50.29 & \bfseries 57.31 & 51.61 \\
\MeanAvg{} & 40.58 & 41.55 & 42.49 & 41.95 & \bfseries 45.19 & 55.26 & 51.11 & 50.47 & 52.52 & \bfseries 53.15 & 46.36 & 41.99 & 41.43 & 41.91 & \bfseries 43.90 \\
\bottomrule
\end{tabular}}
\renewcommand{\arraystretch}{1}
\end{table}

\begin{table}[ht]
\centering
\scriptsize
\setlength{\tabcolsep}{4.5pt}
\caption{Downstream RL OOD Pass@1 results under DAPO. Entries are init$\rightarrow$DAPO, and final checkpoints are selected by ID performance.}
\label{tab:rl-ood-pass1-app}
\begin{tabular}{l c c c c c c c}
\toprule
Method & {GPQA} & {IFEval} & {MMLU-Pro} & {SuperGPQA} & {Safety} & {TruthfulQA} & {\MeanAvg} \\
\midrule
\multicolumn{8}{l}{\textbf{Llama-3.1-8B-Instruct}}\\
\midrule
SFT & \stagepair{31.70}{35.27} & \stagepair{31.42}{31.79} & \stagepair{52.86}{50.00} & \stagepair{19.75}{24.94} & \stagepair{57.90}{59.20} & \stagepair{55.70}{59.06} & \stagepair{41.55}{43.38} \\
IW & \stagebest{37.72}{35.71} & \stagepair{32.53}{31.24} & \stagepair{51.43}{58.57} & \stagepair{21.48}{24.69} & \stagepair{62.20}{60.50} & \stagepair{49.56}{54.82} & \stagepair{42.49}{44.25} \\
DFT & \stagepair{27.68}{27.23} & \stagepair{31.98}{31.05} & \stagebest{52.86}{60.00} & \stagepair{17.53}{19.75} & \stagepair{60.80}{58.60} & \stagebest{60.82}{64.18} & \stagepair{41.95}{43.47} \\
RPSFT & \stagepair{34.60}{35.27} & \stagebest{33.46}{32.53} & \stagepair{58.57}{54.29} & \stagebest{24.20}{29.14} & \stagebest{61.70}{60.80} & \stagepair{58.63}{62.28} & \stagebest{45.19}{45.72} \\
\midrule
\multicolumn{8}{l}{\textbf{Qwen2.5-7B-Instruct}}\\
\midrule
SFT & \stagepair{32.80}{33.92} & \stagepair{52.13}{52.31} & \stagepair{65.71}{71.43} & \stagepair{26.42}{32.10} & \stagepair{70.10}{72.90} & \stagepair{59.50}{65.20} & \stagepair{51.11}{54.64} \\
IW & \stagepair{32.37}{37.28} & \stagepair{54.34}{56.38} & \stagebest{58.57}{74.29} & \stagepair{26.91}{32.59} & \stagepair{70.10}{73.50} & \stagepair{60.53}{61.99} & \stagebest{50.47}{56.01} \\
DFT & \stagepair{30.36}{31.25} & \stagebest{59.70}{60.63} & \stagepair{62.86}{54.29} & \stagepair{25.19}{26.17} & \stagebest{71.90}{73.70} & \stagebest{65.06}{68.27} & \stagepair{52.52}{52.38} \\
RPSFT & \stagebest{34.60}{39.29} & \stagepair{54.34}{54.53} & \stagepair{70.00}{71.43} & \stagebest{26.42}{32.84} & \stagepair{71.10}{72.60} & \stagepair{62.43}{63.74} & \stagepair{53.15}{55.74} \\
\midrule
\multicolumn{8}{l}{\textbf{Qwen2.5-3B-Instruct}}\\
\midrule
SFT & \stagepair{30.58}{27.68} & \stagepair{47.32}{48.43} & \stagepair{45.71}{47.14} & \stagepair{19.26}{22.72} & \stagepair{59.60}{62.50} & \stagepair{49.42}{52.92} & \stagepair{41.99}{43.56} \\
IW & \stagepair{27.01}{27.90} & \stagepair{46.03}{46.58} & \stagepair{45.71}{48.57} & \stagebest{19.51}{22.96} & \stagepair{60.00}{62.30} & \stagepair{50.29}{51.32} & \stagepair{41.43}{43.27} \\
DFT & \stagepair{26.56}{25.67} & \stagebest{51.20}{51.76} & \stagepair{34.29}{44.29} & \stagepair{16.30}{17.78} & \stagebest{65.80}{64.80} & \stagebest{57.31}{56.87} & \stagepair{41.91}{43.53} \\
RPSFT & \stagebest{31.70}{28.79} & \stagepair{52.13}{48.61} & \stagebest{44.28}{55.71} & \stagepair{21.48}{20.99} & \stagepair{62.20}{63.70} & \stagepair{51.61}{53.95} & \stagebest{43.90}{45.29} \\
\bottomrule
\end{tabular}
\end{table}

\FloatBarrier

\section{Experiment Details}
\label{app:experiment-details}

\paragraph{Precomputation.}
We compute the baseline SVD blocks immediately after loading the pretrained model. For each $\ell \in \mathcal{M}'$, we compute top-$k$ singular vectors of $\mathbf{W}^0_\ell$ (e.g., via truncated SVD) and store $(\mathbf{U}^{(k)}_{0,\ell}, \mathbf{V}^{(k)}_{0,\ell}, S^{\mathrm{ref}}_\ell)$. This process costs several minutes to finish, depending on the model size; for Llama-8B, it takes around 3 minutes. 

\paragraph{Training datasets}
For supervised fine-tuning, we use OpenR1-Math \citep{openr1}, which is widely used by prior work; the processed dataset we use is around 25k samples. For downstream reinforcement learning, we use the DAPO-Math-17k set together with the DAPO objective. Every DAPO run is initialized from the corresponding supervised checkpoint, so the RL comparisons isolate the effect of the SFT initializer.

\paragraph{Training setup}
For supervised fine-tuning, we train Qwen2.5-7B and Qwen2.5-3B for 12 epochs, and we train Llama-3.1-8B for 20 epochs because its base checkpoint is weaker. We evaluate the same saved checkpoints across methods. All shared training hyperparameters are kept the same across SFT baselines, and only the method-specific parameters are changed, using the default settings for each method.

Table~\ref{tab:sft-hparams} lists the shared SFT hyperparameters used for all SFT-stage methods.
Unless stated otherwise, RPSFT uses protected rank \(k=768\) and regularization coefficient \(\lambda=1\).

\begin{table}[ht]
\centering
\small
\setlength{\tabcolsep}{8pt}
\caption{Shared supervised fine-tuning hyperparameters.}
\label{tab:sft-hparams}
\begin{tabular}{ll}
\toprule
Parameter & Value \\
\midrule
Learning rate & \num{1e-6} \\
Epochs & 12 for Qwen; 20 for Llama \\
Precision & \texttt{bf16} \\
LR scheduler & cosine \\
Warmup ratio & 0.03 \\
Weight decay & 0.0 \\
\bottomrule
\end{tabular}
\end{table}

For downstream reinforcement learning, we use DAPO and initialize every run from the corresponding supervised checkpoint. We train each run for 100 total training steps and then select the best checkpoint according to in-domain performance for the reported evaluation. Table~\ref{tab:rl-hparams} lists the shared RL-stage hyperparameters.

\begin{table}[ht]
\centering
\small
\setlength{\tabcolsep}{8pt}
\caption{Shared reinforcement-learning fine-tuning hyperparameters.}
\label{tab:rl-hparams}
\begin{tabular}{ll}
\toprule
Parameter & Value \\
\midrule
Actor learning rate & \num{1e-6} \\
Train batch size & 256 prompts \\
PPO mini-batch size & 32 prompts \\
Generation batch size & 128 prompts \\
Responses per prompt & 8 \\
Max prompt length & 2048 tokens \\
Max response length & 10240 tokens \\
Temperature & 1.0 \\
Validation top-p & 0.7 \\
PPO-style clip low / high & 0.2 / 0.28 \\
Total training steps & 100 \\
Overlong buffer length & 2048 tokens \\
Overlong penalty factor & 1.0 \\
\bottomrule
\end{tabular}
\end{table}

\paragraph{Compute and artifacts.}
All SFT and RL training runs use 8 H200 GPUs. We provide the code, dataset preparation and evaluation artifacts, training configurations, and documentation as supplemental material for reproduction, and will make the artifact public on GitHub after the anonymous review period. We use VERL as our main training framework.


\paragraph{Benchmarks}

In-domain tasks: 
\href{https://huggingface.co/datasets/math-ai/aime24}{AIME24},
\href{https://huggingface.co/datasets/math-ai/aime25}{AIME25},
\href{https://huggingface.co/datasets/math-ai/amc23}{AMC23},
\href{https://huggingface.co/datasets/HuggingFaceH4/MATH-500}{MATH-500}~\citep{hendrycks2021measuringmathematicalproblemsolving},
\href{https://huggingface.co/datasets/math-ai/minervamath}{Minerva Math}~\citep{lewkowycz2022solvingquantitativereasoningproblems}, and
\href{https://huggingface.co/datasets/math-ai/olympiadbench}{OlympiadBench}~\citep{he2024olympiadbenchchallengingbenchmarkpromoting}. 

Out-of-domain tasks: 
\href{https://huggingface.co/datasets/Idavidrein/gpqa}{GPQA}~\citep{rein2023gpqagraduatelevelgoogleproofqa},
\href{https://huggingface.co/datasets/google/IFEval}{IFEval-loose}~\citep{zhou2023instructionfollowingevaluationlargelanguage},
\href{https://huggingface.co/datasets/TIGER-Lab/MMLU-Pro}{MMLU-Pro}~\citep{wang2024mmluprorobustchallengingmultitask},
\href{https://huggingface.co/datasets/Maxwell-Jia/SuperGPQA-Astro}{SuperGPQA}~\citep{pteam2025supergpqascalingllmevaluation},
\href{https://huggingface.co/datasets/ThWu/safety_benchmark}{Safety Benchmark}, and
\href{https://huggingface.co/datasets/truthfulqa/truthful_qa}{TruthfulQA}~\citep{lin2022truthfulqameasuringmodelsmimic}.

\section{Metrics for Analysis Figures}

\subsection{Strict SVD Subspace Energy}
\label{app:svd-energy-metric}

Figure~\ref{fig:svd-fisher-motivation} measures the ratio of how much of the empirical Fisher curvature is captured by the strict top singular block. Let \(\mathbf{G}_t \in \mathbb{R}^{m \times n}\) denote the gradient matrix of the \(t\)-th sample, where \(t=1,\dots,N\) indexes the \(N\) sampled gradients used in the diagnostic, and let
\[
g_t = \mathrm{vec}(\mathbf{G}_t)
\]
be its vectorized form. We define the empirical Fisher matrix as
\[
\mathbf{F} = \sum_{t=1}^{N} g_t g_t^\top .
\]

Let \(u_i \in \mathbb{R}^{m}\) and \(v_j \in \mathbb{R}^{n}\) be the \(i\)-th left and \(j\)-th right singular vectors of the pretrained weight matrix, let \(r\) be the protected singular rank, and let \(\mathbf{P}_{\mathrm{svd},r}\) denote the orthogonal projector onto the strict top-\(r\times r\) singular-vector product subspace
\[
\mathrm{span}\bigl\{\mathrm{vec}(u_i v_j^\top): 1\le i,j\le r\bigr\}.
\]

Then the fraction of gradient energy captured by the strict top-\(r\times r\) singular block is
\[
\frac{\mathrm{tr}(\mathbf{P}_{\mathrm{svd},r} \mathbf{F})}{\mathrm{tr}(\mathbf{F})}
=
\frac{\sum_{t=1}^{N} \|\mathbf{P}_{\mathrm{svd},r} g_t\|_2^2}{\sum_{t=1}^{N} \|g_t\|_2^2}.
\]

Equivalently, in matrix form, this can be written as
\[
\frac{\mathrm{tr}(\mathbf{P}_{\mathrm{svd},r} \mathbf{F})}{\mathrm{tr}(\mathbf{F})}
=
\frac{\sum_{t=1}^{N} \left\|\mathbf{U}_r^\top \mathbf{G}_t \mathbf{V}_r\right\|_F^2}
{\sum_{t=1}^{N} \|\mathbf{G}_t\|_F^2},
\]
where \(\mathbf{U}_r = [u_1,\dots,u_r]\) and \(\mathbf{V}_r = [v_1,\dots,v_r]\).

Here, the denominator \(\mathrm{tr}(\mathbf{F})=\sum_{t=1}^N \|g_t\|_2^2\) is the total gradient energy across all samples, meaning the total sum of squared gradient norms, while the numerator \(\mathrm{tr}(\mathbf{P}_{\mathrm{svd},r} \mathbf{F})\) is the portion of that energy lying in the strict top-\(r\times r\) singular block. 
In the figure, the x-axis reports \(x(r)=r^2/R^2\), where \(R=\operatorname{rank}(\mathbf{W}_0)\) is the full singular rank, and the y-axis reports \(\mathrm{tr}(\mathbf{P}_{\mathrm{svd},r}\mathbf{F})/\mathrm{tr}(\mathbf{F})\). We apply this diagnostic to the layer-1 attention
\texttt{q\_proj} weight for Llama-8B, Qwen-7B, and Qwen-3B.

\subsection{Layerwise First-Order Signal}
\label{app:first-order-signal}

To diagnose whether a checkpoint update helps or hurts a dataset at first
order, we compare the update direction against the dataset gradient. For a
checkpoint weight $\mathbf{W}_{\mathrm{ckpt}}$ and base-model weight $\mathbf{W}_{\mathrm{base}}$,
the update is
\begin{equation}
\Delta \mathbf{W} = \mathbf{W}_{\mathrm{ckpt}} - \mathbf{W}_{\mathrm{base}}.
\end{equation}
For a batch $B$ from an ID or OOD dataset, the average gradient on matrix $\mathbf{W}$
is
\begin{equation}
g_{\mathbf{W}}
=
\frac{1}{|B|}
\sum_{(x,y)\in B}
\nabla_{\mathbf{W}} \ell(f_{\mathbf{W}}(x), y).
\end{equation}
Let $\mathcal{P}_\ell$ be the set of matrices assigned to layer $\ell$ in the
plot. The layerwise first-order signal is the average inner product
\begin{equation}
m_\ell
=
\frac{1}{|\mathcal{P}_\ell|}
\sum_{\mathbf{W}\in\mathcal{P}_\ell}
\left\langle g_{\mathbf{W}}, \Delta \mathbf{W} \right\rangle.
\end{equation}
Negative values mean the checkpoint update points along the local descent
direction for that dataset, so it reduces the loss at first order. Positive
values mean the update conflicts with the dataset gradient and tends to increase
the loss at first order.



\subsection{Rotation Metrics}
\label{app:rotation-metrics}

For each layer $\ell$ and weight type
\[
t \in \{\texttt{q\_proj},\texttt{k\_proj},\texttt{v\_proj},\texttt{o\_proj},\texttt{up\_proj},\texttt{down\_proj},\texttt{gate\_proj}\},
\]
let the base-model weight and tuned-model weight be
\[
\mathbf{W}_{\mathrm{base}}^{(\ell,t)},\qquad \mathbf{W}_m^{(\ell,t)}.
\]
Here \(m\) indexes the tuned method or checkpoint being compared. Let \(d_{\mathrm{out}}^{(\ell,t)}\) and \(d_{\mathrm{in}}^{(\ell,t)}\) denote the output and input dimensions of this matrix. We compute truncated SVD with
\[
K_{\ell,t} = \min\bigl(512,d_{\mathrm{out}}^{(\ell,t)},d_{\mathrm{in}}^{(\ell,t)}\bigr),
\]
so that
\[
\mathbf{W}_{\mathrm{base}}^{(\ell,t)} = \mathbf{U}_{\mathrm{base}}^{(\ell,t)} \mathbf{\Sigma}_{\mathrm{base}}^{(\ell,t)} {\mathbf{V}_{\mathrm{base}}^{(\ell,t)}}^\top,
\qquad
\mathbf{W}_m^{(\ell,t)} = \mathbf{U}_m^{(\ell,t)} \mathbf{\Sigma}_m^{(\ell,t)} {\mathbf{V}_m^{(\ell,t)}}^\top,
\]
where
\[
\mathbf{U}_{\mathrm{base}}^{(\ell,t)},\,\mathbf{U}_m^{(\ell,t)} \in \mathbb{R}^{d_{\mathrm{out}}^{(\ell,t)}\times K_{\ell,t}},
\qquad
\mathbf{V}_{\mathrm{base}}^{(\ell,t)},\,\mathbf{V}_m^{(\ell,t)} \in \mathbb{R}^{d_{\mathrm{in}}^{(\ell,t)}\times K_{\ell,t}}.
\]
To measure U-space rotation, we compare the left-singular subspaces through
\[
\mathbf{M}_U^{(\ell,t,m)} = \left(\mathbf{U}_{\mathrm{base}}^{(\ell,t)}\right)^\top \mathbf{U}_m^{(\ell,t)}.
\]
If the singular values of $\mathbf{M}_U^{(\ell,t,m)}$ are $\sigma_1^{(\ell,t,m)},\dots,\sigma_{K_{\ell,t}}^{(\ell,t,m)}$, the principal angles are
\[
\theta_i^{(\ell,t,m)} = \arccos\!\left(\sigma_i^{(\ell,t,m)}\right).
\]
The mean U-space rotation for layer/type pair $(\ell,t)$ is
\[
r_U^{(\ell,t,m)}
=
\frac{180}{\pi}\cdot \frac{1}{K_{\ell,t}}\sum_{i=1}^{K_{\ell,t}}\theta_i^{(\ell,t,m)}.
\]
The layerwise plot in Figure~\ref{fig:rotation-layerwise} uses \texttt{types-all}, so the value at layer $\ell$ averages these per-type rotations over the available matrix types:
\[
y_m(\ell)
=
\frac{\sum\limits_{t \in \mathcal{T}_\ell} c^{(\ell,t,m)} \, r_U^{(\ell,t,m)}}
     {\sum\limits_{t \in \mathcal{T}_\ell} c^{(\ell,t,m)}},
\]
where $\mathcal{T}_\ell$ is the set of available types in layer $\ell$. In the current plotting script, $c^{(\ell,t,m)}=1$ for each available entry, so this reduces to the simple mean across types:
\[
y_m(\ell)
=
\frac{1}{|\mathcal{T}_\ell|}
\sum_{t \in \mathcal{T}_\ell}
\left(
\frac{180}{\pi}\cdot \frac{1}{K_{\ell,t}}\sum_{i=1}^{K_{\ell,t}}
\arccos\!\big(\sigma_i^{(\ell,t,m)}\big)
\right).
\]
Thus, each curve reports how much the top left-singular subspaces rotate away from the base model across layers after averaging over all selected matrix types.

For the rank-wise case-study plot in Figure~\ref{fig:rotation-rankwise}, we fix one matrix case for each model and vary the prefix rank $r\le K_{\ell,t}$. At each $r$, we replace the full truncated bases by their first $r$ columns, compute the singular values of
\[
\mathbf{M}_U^{(r,m)} = \left(\mathbf{U}_{\mathrm{base}}^{(r)}\right)^\top \mathbf{U}_m^{(r)},
\]
where $\mathbf{U}_{\mathrm{base}}^{(r)}$ and $\mathbf{U}_m^{(r)}$ are the first $r$ columns of the truncated left-singular bases for the base and tuned matrices, and report the mean principal angle
\[
y_m(r)
=
\frac{180}{\pi}\cdot \frac{1}{r}\sum_{i=1}^{r}\arccos\!\big(\sigma_i^{(r,m)}\big).
\]
This keeps the same rotation definition and only changes the swept dimension: the layerwise figure varies $\ell$ after averaging over types, while the rankwise figure varies the prefix rank within one fixed case-study matrix.

\subsection{Entropy Metric}
\label{app:entropy-metric}

For prompt $j$, method $m$ generates tokens
\[
y^{(m,j)} = \left(y^{(m,j)}_1,\dots,y^{(m,j)}_{T_{m,j}}\right),
\]
with greedy decoding and $T_{m,j}\le 128$. At generation step $t$, let
\[
p^{(m,j)}_t(v) = \operatorname{softmax}(z^{(m,j)}_t)_v.
\]
The token entropy is
\[
h^{(m,j)}_t
=
-\sum_{v\in\mathcal V} p^{(m,j)}_t(v)\log p^{(m,j)}_t(v),
\]
and the sample-level average entropy is
\[
e_{m,j}
=
\frac{1}{T_{m,j}}\sum_{t=1}^{T_{m,j}} h^{(m,j)}_t.
\]
Each figure plots a Gaussian KDE over $\{e_{m,j}\}_{j=1}^{N_m}$:
\[
\hat f_m(x)
=
\frac{1}{N_m h_m}\sum_{j=1}^{N_m}
\frac{1}{\sqrt{2\pi}}
\exp\!\left(
-\frac{1}{2}\left(\frac{x-e_{m,j}}{h_m}\right)^2
\right),
\]
using the default bandwidth
\[
h_m = 1.06\, s_m\, N_m^{-1/5},
\]
where $s_m$ is the sample standard deviation. We use the same definition for both datasets: AIME25 for the in-domain plot and GPQA for the out-of-domain plot.

\subsection{Auxiliary Entropy Plots Across Model Families}
\label{app:entropy-family}

We include these entropy plots as auxiliary analysis rather than as a primary claim. Compared with the hidden-state drift results in the main paper, the entropy differences are more model-dependent and less uniformly separated across baselines.

\begin{figure}[ht]
\centering
\includegraphics[width=\textwidth]{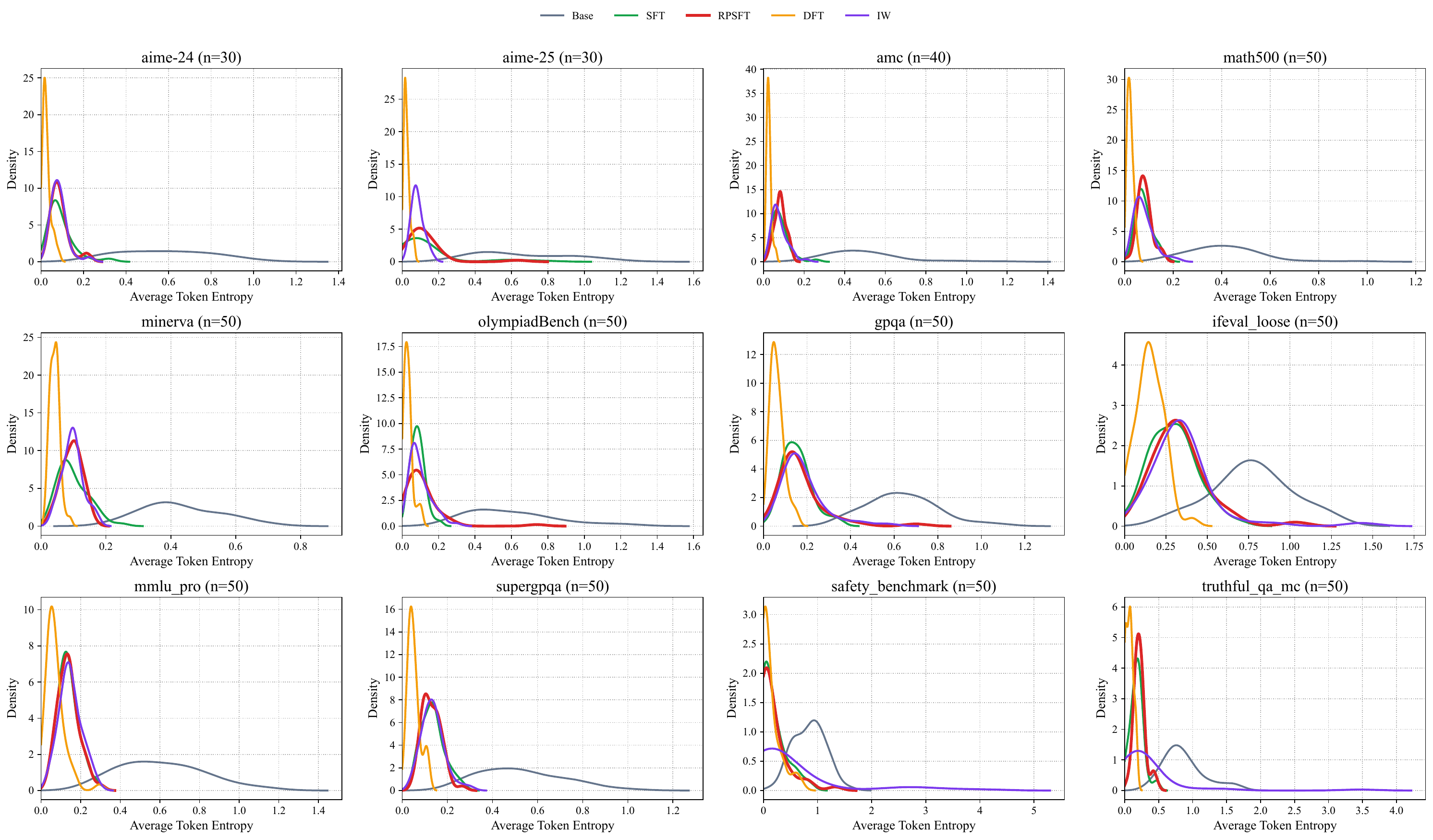}
\caption{Average-token-entropy distributions across the twelve in-domain and out-of-domain benchmarks for Llama-3.1-8B-Instruct. The differences are modest and benchmark-dependent: DFT more often shifts toward lower entropy, while RPSFT generally stays in a similar range to the stronger tuned baselines.}
\label{fig:entropy-llama-overall}
\end{figure}

\begin{figure}[ht]
\centering
\includegraphics[width=\textwidth]{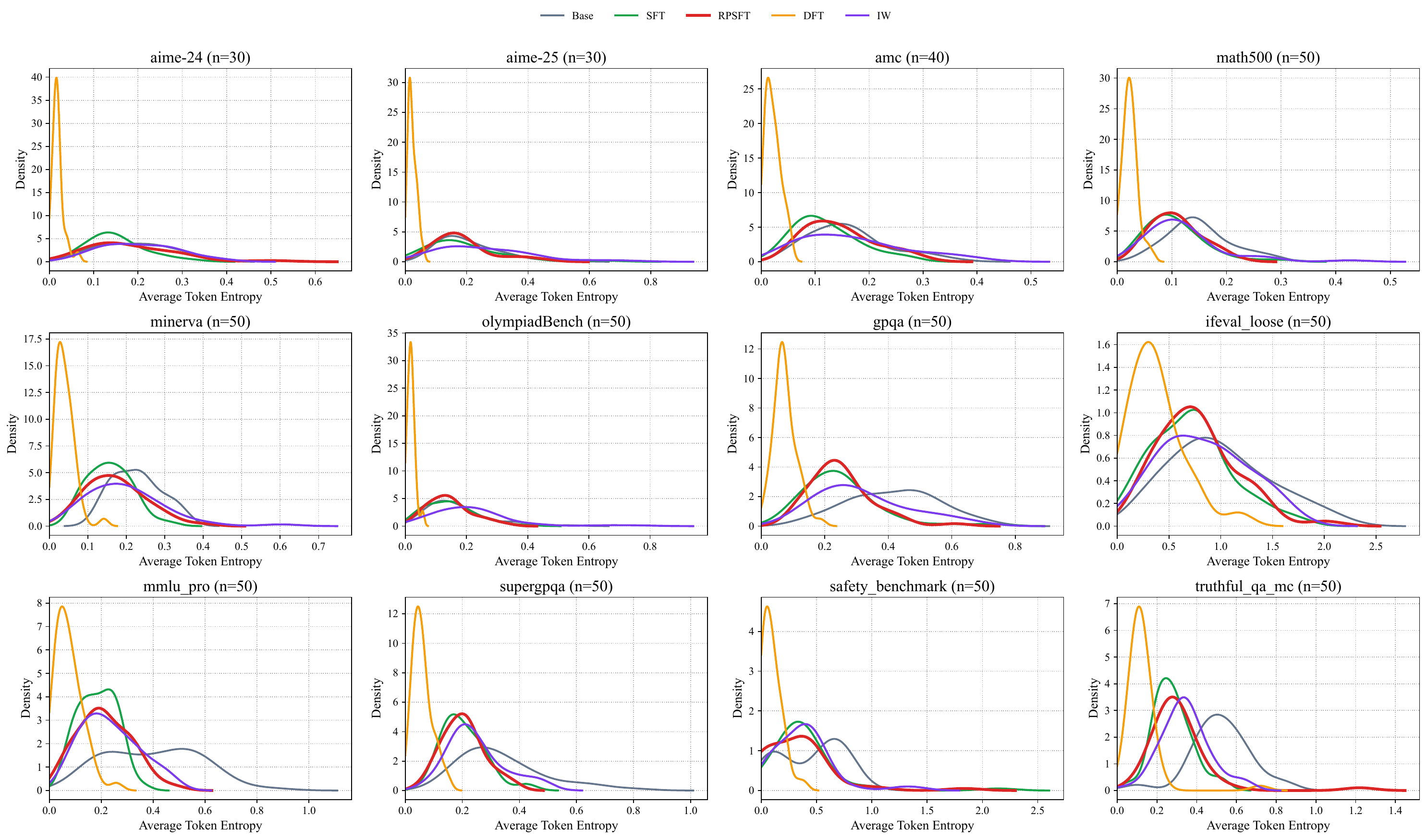}
\caption{Average-token-entropy distributions across the same twelve benchmarks for Qwen2.5-7B-Instruct. RPSFT often remains above DFT and within the broader range of the tuned baselines, but the gaps are not uniformly large across all panels.}
\label{fig:entropy-qwen7b-overall}
\end{figure}

\begin{figure}[ht]
\centering
\includegraphics[width=\textwidth]{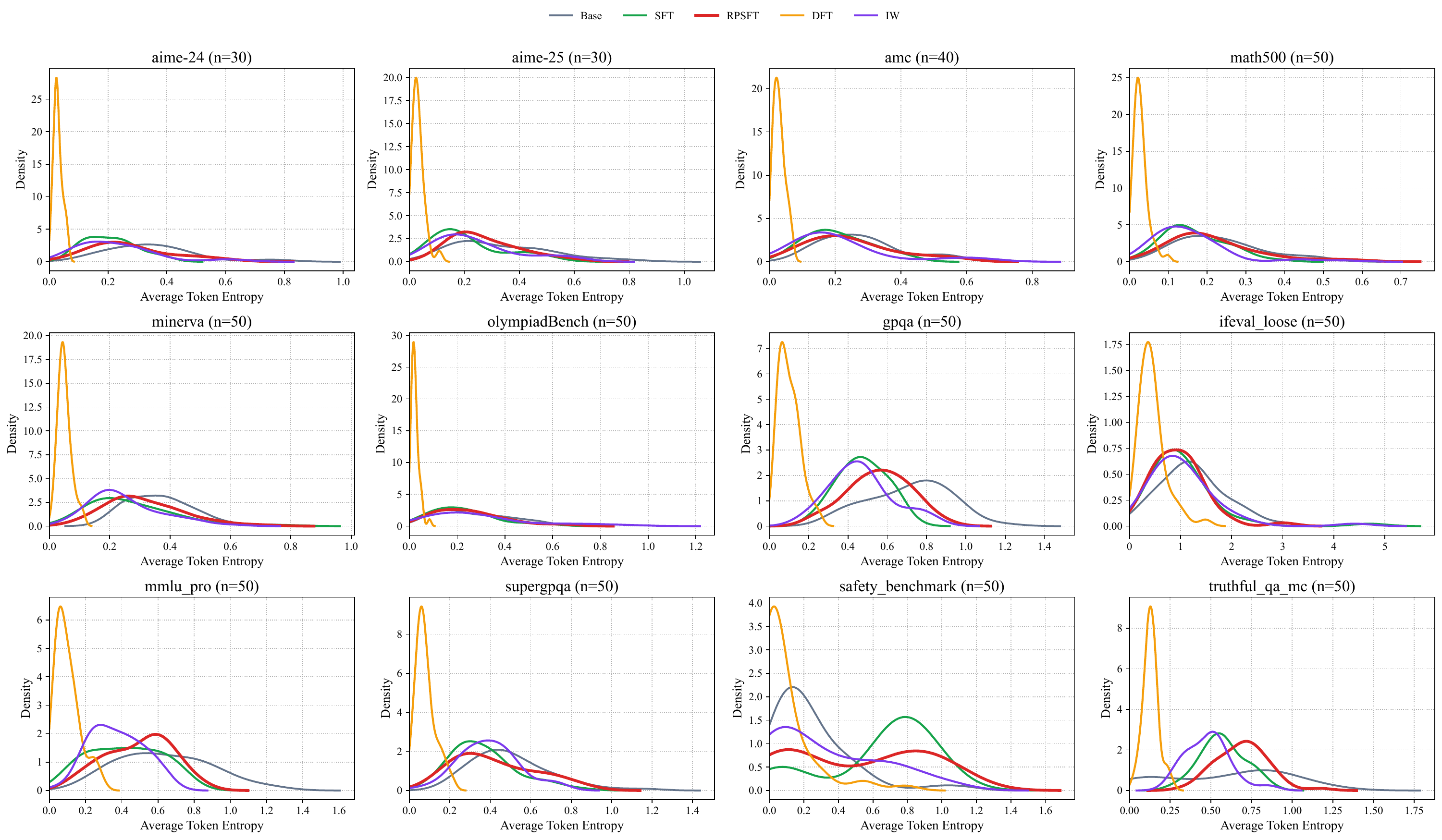}
\caption{Average-token-entropy distributions across the same twelve benchmarks for Qwen2.5-3B-Instruct. RPSFT often shifts toward a broader entropy profile than standard SFT or DFT, although the base model remains the highest-entropy reference and the gaps are still panel-dependent.}
\label{fig:entropy-qwen3b-overall}
\end{figure}

\subsection{Base-Model Drift in Hidden Representations}
\label{app:hidden-drift}

For dataset $d$ with $N_d$ prompts, let $h_i^{(m,d)}\in\mathbb{R}^p$ be the pooled hidden representation for prompt $i$ under model $m\in\{\mathrm{base},\mathrm{sft},\mathrm{RPSFT},\mathrm{dft},\mathrm{iw}\}$.
The hidden-space centroid of model $m$ on dataset $d$ is\looseness=-1
\[
\mu^{(m,d)}
=
\frac{1}{N_d}\sum_{i=1}^{N_d} h_i^{(m,d)} .
\]
We measure drift from the base model by the centroid distance
\[
D_{\mathrm{hidden}}^{(m,d)}
=
\left\|
\mu^{(m,d)}-\mu^{(\mathrm{base},d)}
\right\|_2 .
\]

For visualization, we stack all prompt representations from the five models,
\[
X^{(d)}=
\begin{bmatrix}
H^{(\mathrm{base},d)}\\
H^{(\mathrm{sft},d)}\\
H^{(\mathrm{RPSFT},d)}\\
H^{(\mathrm{dft},d)}\\
H^{(\mathrm{iw},d)}
\end{bmatrix}
\in \mathbb{R}^{M_d \times p},
\qquad
M_d = 5N_d,
\]
center the matrix as $\tilde X^{(d)} = X^{(d)} - \mathbf{1}\,\bar{x}^{(d)\top}$, and project onto the first two PCA directions:
\[
Z^{(d)}=\tilde X^{(d)}W_{1:2}\in\mathbb{R}^{M_d\times 2}.
\]
If $z_i^{(m,d)}\in\mathbb{R}^2$ is the PCA coordinate of prompt $i$, the PCA-plane centroid is
\[
c^{(m,d)}
=
\frac{1}{N_d}\sum_{i=1}^{N_d} z_i^{(m,d)} ,
\]
and the plotted centroid shift is
\[
D_{\mathrm{PCA}}^{(m,d)}
=
\left\|
c^{(m,d)}-c^{(\mathrm{base},d)}
\right\|_2 .
\]
In the plots, each point is one prompt and each arrow starts at the base centroid $c^{(\mathrm{base},d)}$ and ends at the centroid of a tuned model. We use $D_{\mathrm{hidden}}^{(m,d)}$ for the quantitative drift measurement and $D_{\mathrm{PCA}}^{(m,d)}$ only for the 2D visualization.

\begin{figure}[ht]
\centering
\includegraphics[width=1.0\textwidth]{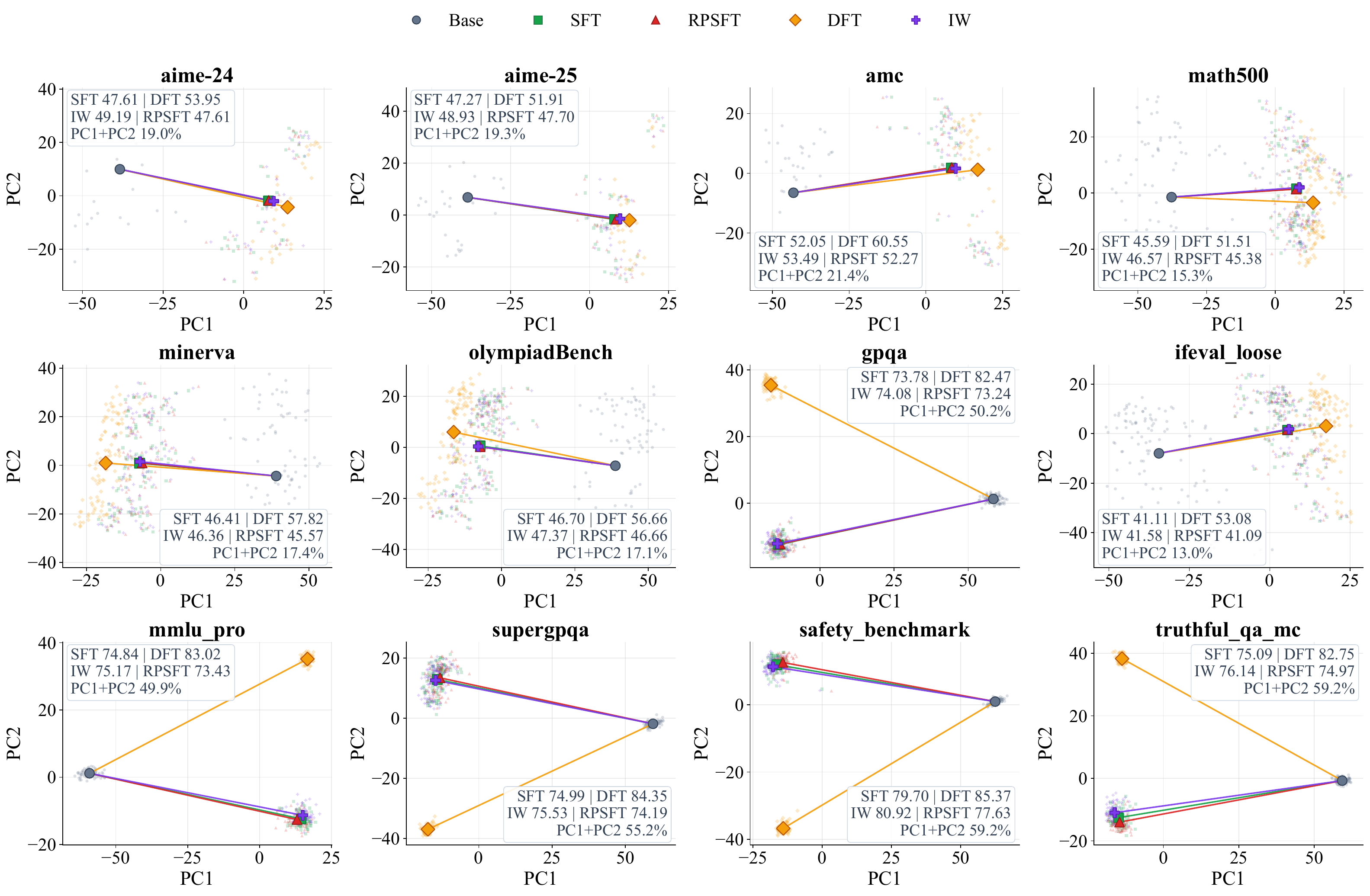}
\caption{Hidden-state drift on Llama-3.1-8B-Instruct. Across the benchmark panels, RPSFT remains close to the base-model centroid and is competitive with the strongest tuned alternatives, which is consistent with reduced representation drift.}
\label{fig:base-shift-llama}
\end{figure}

\begin{figure}[ht]
\centering
\includegraphics[width=1.0\textwidth]{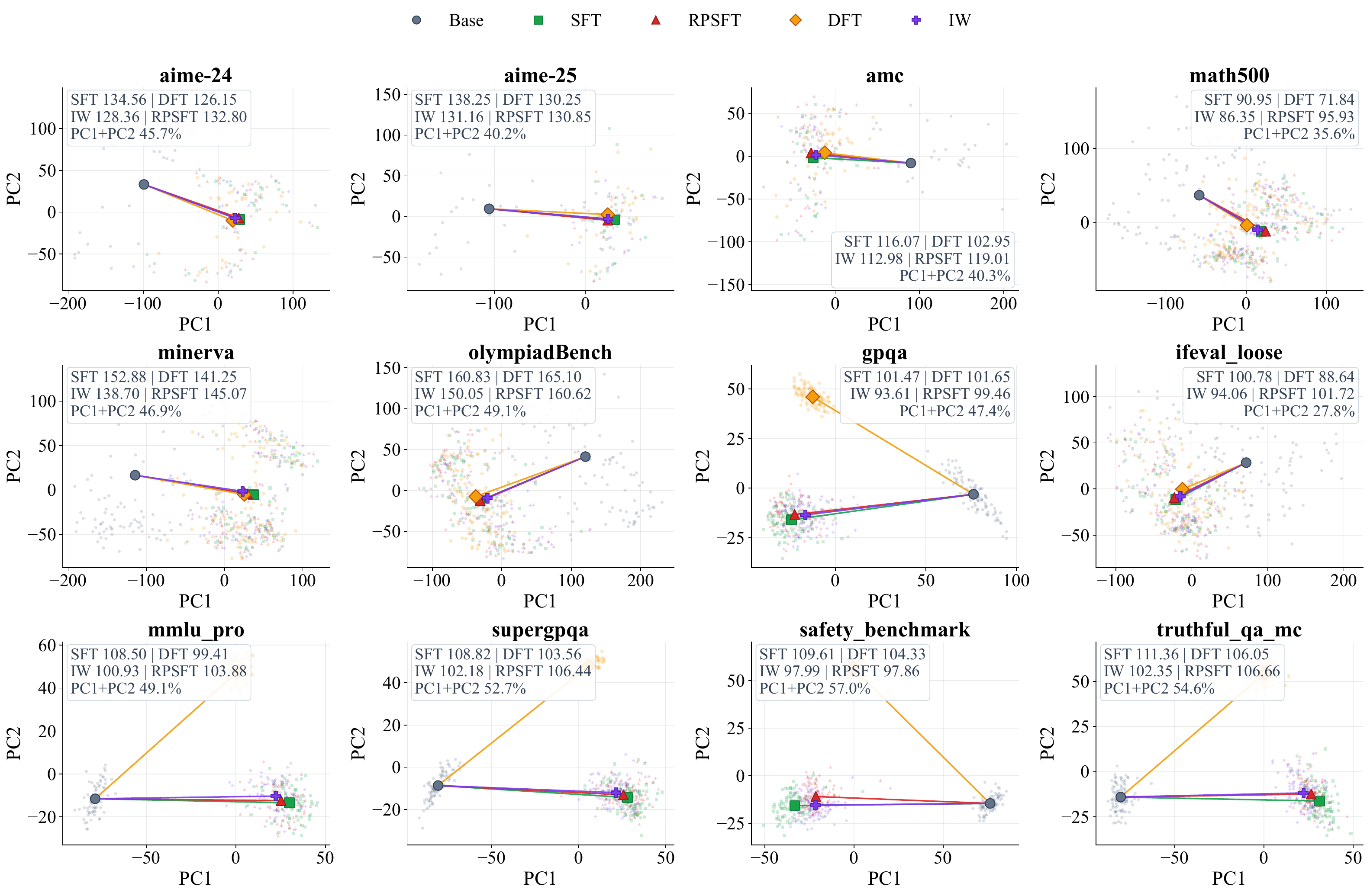}
\caption{Hidden-state drift on Qwen2.5-7B-Instruct. The PCA view compares centroid shifts away from the base model across benchmark panels. RPSFT is usually among the closest tuned centroids and remains closer than standard SFT in most panels, which is consistent with reduced representation drift after supervised fine-tuning.}
\label{fig:base-shift-qwen7b}
\end{figure}

\clearpage

\end{document}